\documentclass[pdflatex,sn-basic]{sn-jnl}

\usepackage{graphicx}%
\usepackage{multirow}%
\usepackage{amsmath,amssymb,amsfonts}%
\usepackage{amsthm}%
\usepackage{mathrsfs}%
\usepackage[title]{appendix}%
\usepackage{xcolor}%
\usepackage{textcomp}%
\usepackage{manyfoot}%
\usepackage{booktabs}%
\usepackage{algorithm}%
\usepackage{algorithmicx}%
\usepackage{algpseudocode}%
\usepackage{listings}%
\usepackage{bicaption}
\usepackage{subcaption}
\usepackage{multicol}
\usepackage{chngcntr}
\usepackage{etoolbox}

\makeatletter
\patchcmd{\@maketitle}{\vskip21pt}{\vskip0pt}{}{\PackageError{PBOS}{Title spacing patch failed}{}}
\patchcmd{\@maketitle}{\vskip20pt}{\vskip10pt}{}{\PackageError{PBOS}{Title spacing patch failed}{}}
\def\ps@pbosfirstpage{%
  \ps@titlepage
  \def\@oddfoot{\vbox{%
    \hsize\textwidth\parindent0pt\normalfont\footnotesize
    \hrule width0.3\textwidth height0.3pt\kern4pt
    Published in \textit{Machine Learning}, \textbf{114}, Article 292 (2025).\par
    \url{https://doi.org/10.1007/s10994-025-06893-z}\par
    \kern8pt\hbox to\textwidth{\hfil\normalfont\normalsize\thepage\hfil}}}%
  \let\@evenfoot\@oddfoot
}
\makeatother

\theoremstyle{plain}%
\newtheorem{theorem}{Theorem}

\newtheorem{example}{Example}%
\newtheorem{remark}{Remark}%
\newtheorem{lemma}{Lemma}
\newtheorem{corollary}{Corollary}

\newtheorem{definition}{Definition}%

\hypersetup{pdftitle={Preference-based opponent shaping in differentiable games},
  pdfauthor={Xinyu Qiao, Yudong Hu, Congying Han, Weiyan Wu, Tiande Guo},
  pdfsubject={Published in Machine Learning, 114, Article 292 (2025)},
  pdfkeywords={Opponent Shaping, Preference, Game Theory, Differentiable Game}}

\begin{document}

\title[Preference-based opponent shaping]{Preference-based opponent shaping in differentiable games}

\author[1]{\fnm{Xinyu} \sur{Qiao}}\email{qiaoxinyu22@mails.ucas.ac.cn}

\author[1]{\fnm{Yudong} \sur{Hu}}\email{huyudong201@mails.ucas.ac.cn}

\author*[1]{\fnm{Congying} \sur{Han}}\email{hancy@ucas.ac.cn}

\author[1]{\fnm{Weiyan} \sur{Wu}}\email{wuweiyan22@mails.ucas.ac.cn}

\author[1]{\fnm{Tiande} \sur{Guo}}\email{tdguo@ucas.ac.cn}

\affil*[1]{\orgdiv{School of Mathematical Sciences}, \orgname{University of Chinese Academy of Sciences}, \orgaddress{\street{Yuquan}, \city{Beijing}, \postcode{100049}, \country{China}}}

\abstract{Strategy learning in game environments with multi-agent is a challenging problem. Since each agent's reward is determined by the joint strategy, a greedy learning strategy that aims to maximize its own reward may fall into a local optimum. Recent studies have proposed the opponent modeling and shaping methods for game environments. These methods enhance the efficiency of strategy learning by modeling the strategies and updating processes of other agents. However, these methods often rely on simple predictions of opponent strategy changes. Due to the lack of modeling behavioral preferences such as cooperation and competition, they are usually applicable only to predefined scenarios and lack generalization capabilities. In this paper, we propose a novel Preference-based Opponent Shaping (PBOS) method to enhance the strategy learning process by shaping agents' preferences towards cooperation. We introduce the preference parameter, which is incorporated into the agent's loss function, thus allowing the agent to directly consider the opponent's loss function when updating the strategy. We update the preference parameters concurrently with strategy learning to ensure that agents can adapt to any cooperative or competitive game environment. Through a series of experiments, we verify the performance of PBOS algorithm in a variety of differentiable games. The experimental results show that the PBOS algorithm can guide the agent to learn the appropriate preference parameters, so as to achieve better reward distribution in multiple game environments.}


\maketitle
\thispagestyle{pbosfirstpage}
\clearpage

\section{Introduction}\label{sec1}

Multi-agent reinforcement learning (MARL), as a theoretical framework for modeling agent behavior in complex game environments, has become a significant area of research \citep{zhang2021multi, yang2020overview}. Unlike traditional game theory, MARL typically allows agents to learn strategies through repeated interactions to achieve equilibrium \citep{wen2022multi}. By relaxing the assumptions of agent rationality and independence, MARL can learn strategies efficiently with arbitrary environments and opponents \citep{foerster2017learning,letcher2018stable,hu2023modeling}. Current applications of MARL in game environments are primarily focused on zero-sum games (fully competitive) \citep{foerster2017learning,zhang2020model} and fully cooperative games \citep{gupta2017cooperative,yu2022surprising}, since the behavioral preferences of opponent agents in these environments are relatively easy to predict. Nevertheless, the environments in  practical applications, \textit{e.g.}, economic markets, robotics and distributed control, may have multiple equilibria \citep{hu2003nash,zhang2020bi}, and opponent agents may not exhibit clear preferences for different strategies, thus agents need to learn strategies in general-sum games \citep{curry2023learning,bucsoniu2010multi}. The Prisoner's dilemma \citep{robert1981cooperation,Harper2017ReinforcementLP} is a classic example of the tension between mutual cooperation leading to a win-win situation and focusing solely on self-interest leading to a lose-lose situation. Therefore, modeling and shaping the behavior of opponent agents is the main challenge for the application of MARL in these environments \citep{fung2024analysing}.
    
Recent advancements in MARL have introduced opponent modeling and shaping techniques that allow agents to learn not just their own strategies, but also to predict and influence the strategies of the opponent, such as \citep{foerster2017learning,letcher2018stable,willi2022cola}. These methods show promise in improving the efficiency of strategy learning by incorporating the behavior of other agents into the learning process. However, these methods presuppose a static notion of opponents' strategic preferences, suggesting that other agents adhere to updating their strategies in accordance with preordained protocols, an assumption that overlooks the dynamic interplay of evolving preferences in the learning process. Due to the lack of dynamic modeling of opponents' behavioral preferences, these methods may fail when confronted with specific environments and adversaries, such as Stackelberg Leader Game \citep{VANDAMME1999105} and Stag Hunt \citep{rousseau1984inequality}. 

    In general-sum games, the effectiveness of the agent's learned strategy depends on the preferences of other opponents due to the existence of multiple equilibria. When facing opponents who are willing to cooperate, a win-win strategy can yield higher rewards. Conversely, if other agents adopt greedy strategies aimed at maximizing their own rewards, cooperative behavior can significantly harm one's own rewards. For example, in some general-sum cooperative games, if each player focuses only on his own loss function, it may lead to a non-optimal distribution of rewards, such as in Stackelberg Leader Game \citep{VANDAMME1999105}, many game algorithms \citep{foerster2017learning,letcher2018stable,schafer2019competitive} converge to the Nash equilibrium (NE) point (2,1)\footnote{In this notation, a tuple $(r_1, r_2)$ indicates the reward received by each agent. For example, (2,1) means agent 1 gets reward 2 and agent 2 gets reward 1.} , without finding a better one (3,2)---Stackelberg equilibrium \citep{zhang2020bi}.

    Considering that the rewards of different strategies are determined by the opponents' behavioral preferences, we propose a novel opponent shaping method known as Preference-Based Opponent Shaping (PBOS). We observe that cooperative and competitive behaviors in games are related to the agents' loss functions. Consistent loss functions induce cooperation in games, whereas divergent loss functions lead to complete competition. Therefore, if an agent optimizes a weighted average of their loss and their opponent’s rather than just optimizing their loss, different behavioral preferences will emerge. We introduce a preference parameter into the agent's strategy learning, which allows the agent to update its strategy while considering the opponent's loss function. This approach not only enhances the agent's expected returns but also promotes the maximization of cooperation.

    The primary contribution of this study lies in the innovative introduction of the preference parameter $c$, which is employed to model an agent’s response to the goodwill expressed by its opponent. Specifically, $c>0$ indicates that the agent is more inclined toward cooperation, whereas $c<0$ suggests a tendency toward adversarial behavior. To facilitate the effective learning of an appropriate preference parameter, this paper adopts a method that shapes the opponent’s preference parameter dynamics. The opponent's preference parameter variation is modeled as a function of the agent’s own preference parameter change, expressed as $\Delta c_{-i}=g_i(\Delta c_i)+\epsilon_i,i=1,2$, where $-i$ denotes the opponent of agent $i$, $\Delta c_i$ represents the variation in the preference parameter, $g_i$  is a function, and $\epsilon_i$ denotes noise. The $g_i$ function is simplified to linear in Section \ref{sec3}, and the linear coefficient $K$ is estimated by a dynamic simulation method constructed in this paper. Subsequently, the preference parameter is incorporated into the agent’s objective function, enabling the agent to seek strategies that surpass the Nash equilibrium by expressing goodwill in general-sum games. Furthermore, the generalization capability of the PBOS algorithm is validated in stochastic game environments, and its performance is further demonstrated through adversarial tests against other opponent modeling approaches.

    The rest of the paper is organized as follows: Section \ref{sec2} reviews the relevant work in the field of opponent shaping and differentiable games; it provides background on differentiable games and describes the baseline method we used in our experiments. Section \ref{sec3} and \ref{sec4} introduce the PBOS algorithm and theoretical basis in detail. Section \ref{sec5} and \ref{sec6} present the experimental setup and results, showing the performance of PBOS in various game scenarios. In the end, Section \ref{sec7} summarizes the paper and discusses potential directions for future research.

\section{Related Work}\label{sec2}
    The study of non-zero-sum games has a long history in game theory and evolutionary studies \citep{zhang2020model}. With the vigorous development of multi-agent reinforcement learning, research on non-zero-sum games has gained additional perspectives \citep{yang2020overview,hostallero2020inducing,yang2020learningincentivizelearningagents}. This paper focuses on a series of approaches to solve the opponent shaping problem \citep{foerster2017learning, letcher2018stable, willi2022cola, fung2024analysing}.

    The core idea of opponent shaping is to maintain an explicit belief about the opponent and optimize decisions by establishing assumptions about the opponent's strategy \citep{willi2022cola}. This allows the system to reason about the behavior of the opponent and calculate the optimal response strategy, thereby facilitating more effective decision-making in uncertain and dynamic environments.

    There are several types of methods: pre-defined opponent types \citep{synnaeve2011bayesian,weber2009data}, policy reconstruction methods \citep{mealing2015opponent}, and recursive reasoning methods \citep{he2016opponent,albrecht2019reasoning,wen2019probabilistic}. In comparison, PBOS assumes white-box access to the opponent's learning algorithm, rewards, and gradients, placing it within the framework of differentiable games \citep{balduzzi2018mechanics,letcher2018stable,willi2022cola}.

    Learning with Opponent Learning Awareness (LOLA) \citep{foerster2017learning} modifies the learning objective by predicting and differentiating through opponent learning steps \citep{letcher2018stable}, which has been successful in experiments, especially in the Iterated Prisoner's Dilemma (IPD) \citep{lu2022model}. Unfortunately, LOLA fails to guarantee the preservation of stable fixed points (SFPs) \citep{letcher2018stable}. To improve upon LOLA, Stable Opponent Shaping (SOS) \citep{letcher2018stable} applies ad-hoc corrections to the LOLA update, leading to theoretically guaranteed convergence to SFPs \citep{letcher2018stable,willi2022cola}. Competitive Gradient Descent (CGD) \citep{schafer2019competitive} provides an algorithm for numerical computation of Nash equilibria in competitive two-player zero-sum games \citep{schafer2019competitive}.

    The agents using these methods \citep{foerster2017learning,letcher2018stable,schafer2019competitive} are rational and self-interested, focusing solely on optimizing their strategies to maximize personal gains without considering extending goodwill to promote cooperation with their opponents. This selfishness can cause agents to miss the opportunity for better rewards. For instance, in the Stag Hunt \citep{rousseau1984inequality}, agents utilizing the LOLA, SOS, and CGD algorithms tend to converge to the less favorable Nash equilibrium point (1,1) rather than the more advantageous equilibrium (4,4) (Fig. \ref{Stag Hunt-lc}).

    Meanwhile, within non-zero-sum game contexts, a prevalent strategy for managing these challenges is the implementation of loss-sharing mechanisms \citep{gemp2022d3creducingpriceanarchy, kwon2023autoaligning, lupu2020gifting,radke2022importance}, which fall under the purview of Mixed-Motive Interactions \citep{mckee2020social}. The D3C method, as referenced in \citep{gemp2022d3creducingpriceanarchy}, exhibits similarities to our approach; however, our loss-sharing mechanism extends beyond a convex combination to include non-convex summation, with the coefficient being adaptively learned through opponent modeling. Nevertheless, D3C may prove ineffective in stochastic games, as it struggles to attain a level of social welfare comparable to that achieved by cooperative strategies \citep{willis2024resolvingsocialdilemmasminimal}.

    Next, we introduce the concept of \emph{Differentiable games} and outline the baseline methods used in our experiments.
    \begin{definition}(Differentiable games). 
        A differentiable game is a set of $n$ players controlling parameters $\theta_i\in \mathbb{R}^{d_i}$ to minimize twice continuously differentiable losses $L_i(\theta_1,\cdots,\theta_n):\mathbb{R}^d\to \mathbb{R}$, where $\theta_i\in \mathbb{R}^{d_i}$ for $\sum_id_i=d$ \emph{\citep{letcher2018stable,Azizian2020ATA,willi2022cola}}. 
        \label{dm}
    \end{definition}

    This study specifically focuses on the scenario where $n=2$, representing a two-player game. Subsequently, we will introduce the baselines utilized herein.

    \subsection{LOLA}
    LOLA \citep{foerster2017learning} addresses a differentiable game scenario with $n=2$. A LOLA agent updates its parameter $\theta_1$ under the assumption that its opponent behaves as a ``naive" learner, updating its parameter $\theta_2$ using gradient descent. Specifically, agent 1 formulates its modified loss as $\overline{L}_1=L_1(\theta_1,\theta_2+\Delta \overline{\theta_2})$, where $\Delta \overline{\theta_2}= -\alpha \nabla_2 L_2(\theta_1, \theta_2)$.  Here, $\nabla_2$ denotes the gradient with respect to $\theta_2$ and $\alpha$ represents the assumed learning rate of the ``naive" opponent. The first-order Taylor expansion of $\overline{L}_1$ yields $\overline{L}_1\approx L_1+(\nabla_2L_1)^T \Delta \overline{\theta_2}.$ Consequently, agent 1's first-order LOLA update is 

    \begin{equation}
        \Delta\theta_1=-\beta\Big(\nabla_1L_1+(\nabla_1\Delta\overline{\theta_2})^T\nabla_2L_1\Big),
    \end{equation}
    where $\beta$ denotes agent 1's specific learning rate and $\nabla_i L_j=\nabla_{\theta_i} L_j$\footnote{Throughout the remainder of this paper, this notation is adopted.}$,i,j=1,2$.
    
    LOLA has demonstrated empirical success, notably achieving tit-for-tat in the Iterated Prisoner's Dilemma (IPD) \citep{willi2022cola}. However, it has shown limitations in maintaining Stable Fixed Points (SFPs) \citep{letcher2018stable}.
\subsection{SOS}
	SOS \citep{letcher2018stable} represents a significant advancement over LOLA. According to \citet{letcher2018stable}, the \emph{simultaneous gradient} of the game is defined as the concatenation of each player's gradient,
        \begin{equation}
            \xi=(\nabla L_1,...,\nabla L_n)^T\in \mathbb{R}^d.
        \end{equation}

	The \emph{Hessian} of the game, denoted as $H = \nabla \xi$, forms a block matrix
        \begin{equation}
            H=\begin{pmatrix}\nabla_{11}L_1& \cdots & \nabla_{1n}L_1 \\ \vdots & \ddots & \vdots \\ \nabla_{n1}L_n & \cdots & \nabla_{nn}L_n \end{pmatrix}\in \mathbb{R}^{d\times d}.
        \end{equation}

	Notably, $H$ is typically asymmetric and can be viewed as described in \citep{letcher2018stable}. Furthermore, $H$ decomposes into $H=H_d+H_o$, where $H_d$ comprises the diagonal blocks of $H$, and $H_o$ represents the off-diagonal components. By defining $\chi = \text{diag}(H_o^T\nabla L)$, the LOLA gradient can be computed as stated in \citep{letcher2018stable}:
	\begin{equation}
	    \text{LOLA}=(I-\alpha H_o)\xi-\alpha \chi.
	\end{equation} 
	From \citep{letcher2018stable}, the resulting gradient of the SOS algorithm is expressed by:
        \begin{equation}
            \xi_p = (I-\alpha H_o)\xi-p\alpha \chi,
        \end{equation}
	where $\alpha$ denotes the learning rate and $I$ stands for the identity matrix. Here, $p$ is a hyperparameter determined by Algorithm \ref{CPBOS}. Specifically, setting $p=0$ results in $\xi_p$ reducing to LookAhead \citep{zhang2010multi,letcher2018stable}, while $p=1$, $\xi_p$ corresponds to LOLA. However, the maintenance of fixed points is contingent upon $p$ approaching infinitesimal values \citep{letcher2018stable}. To resolve this issue, SOS introduces a dual criterion for the probability $p$ at each learning step, aiming to drive $p$ towards zero. This approach not only combines the advantages of LookAhead and LOLA but also addresses LOLA's challenge in preserving Stable Fixed Points (SFPs). Empirical findings consistently demonstrate that SOS achieves or surpasses LOLA's performance, as evidenced in \citep{letcher2018stable}.

    \subsection{CGD}
	CGD \citep{schafer2019competitive} is an algorithm designed for computing Nash equilibria in competitive two-player games. It extends gradient descent principles to the two-player setting.
 
    The update rule of CGD is defined as:
	\begin{equation}
	    \begin{pmatrix}\Delta\theta_1\\ \Delta\theta_2\end{pmatrix}=-\beta\begin{pmatrix}I & \alpha \nabla_{12}L_1\\ \alpha \nabla_{21}L_2 & I\end{pmatrix}^{-1}\begin{pmatrix}\nabla_1L_1 \\ \nabla_2L_2\end{pmatrix}.
	\end{equation}
    where $I$ denotes the identity matrix, and $\alpha,\beta$ are learning rate coefficients.
    
    By applying the expansion $\lambda_{max}(A)<1\Rightarrow (I-A)^{-1}=\lim\limits_{N\to\infty}\sum_{k=0}^NA^k$ to the update rule, various orders of CGD can be derived \citep{schafer2019competitive}. For instance, setting $N=1$, yields the Linearized CGD (LCGD) \citep{schafer2019competitive}, characterized by the update rule $\Delta \theta_1 := -\alpha \nabla_1L_1+\alpha^2\nabla_{12}L_1\nabla_2L_2$ \citep{willi2022cola}.

    CGD mitigates oscillations and divergent behaviors that may arise in simple bilinear games \citep{schafer2019competitive}. Moreover, for locally convex-concave zero-sum games, CGD has been demonstrated to converge locally, with a potential for exponential convergence rates \citep{schafer2019competitive}. However, in the context of the IPD, CGD does not identify the optimal strategy \citep{willi2022cola}.

\section{Preference-based Opponent Shaping (PBOS)}\label{sec3}
    The three baselines discussed above each have certain limitations, hindering the agent from achieving more favorable outcomes. To address these shortcomings and enhance the agent's ability to secure higher rewards in the game, we propose the PBOS algorithm. In this section, we will provide a detailed, step-by-step explanation of the PBOS algorithm, outlining its key principles and the reasoning behind its design.
    
    Intuitively, focusing solely on agent's own loss function in a game may lead to a non-stationary environment \citep{letcher2018stable,kim2021policy} or suboptimal outcomes that are socially undesirable \citep{foerster2017learning}. Therefore, it is crucial to model behavioral preferences of opponent agents. Our method, PBOS, integrates the opponent's loss function into the objective to model their preferences.

    To validate the effectiveness of incorporating preference parameters into the loss functions, we initially conduct experiments with fixed preference parameters. The original loss functions for agents 1 and 2 are denoted as $L_1$ and $L_2$, respectively. We introduce preference parameters into the loss functions to facilitate agent training. The modified loss functions for agents 1 and 2 are then defined as $L_1' = L_1 + c_1 L_2$ and $L_2' = L_2 + c_2 L_1$, where $c_1$ and $c_2$ are the introduced preference parameters.
    
    Subsequently, we apply the SOS strategy updating to train agents using these modified loss functions, forming the Constant-preference-based Opponent Shaping (CPBOS) approach, detailed in Algorithm \ref{CPBOS}.
    
    \begin{algorithm}
        \caption{CPBOS}\label{CPBOS}
        \begin{algorithmic}[1]
        \State Initialize $\theta$ randomly, set the preference parameters $c_1$ and $c_2$, and the hyperparameters $a$ and $b$ within the interval (0, 1).
        \While{not done}
            
            \State Define the modified loss functions as $L_1'=L_1+c_1L_2, L_2'=L_2+c_2L_1$.
            
            \State Compute $\xi_0=(I-\alpha H_o)\xi$ and $\chi = \text{diag}(H_o^T\nabla L)$ at $\theta$ based on $L_1',L_2'$.
    
            \State \textbf{if} $\langle -\alpha\chi,\xi_0\rangle >0$ \textbf{then} $p_1=1$ \textbf{else} $p_1=\min\{1,\frac{-a\|\xi_0\|^2}{\langle -\alpha\chi,\xi_0\rangle}\}$.
    
            \State \textbf{if} $\|\xi\|<b$ \textbf{then} $p_2=\|\xi\|^2$ \textbf{else} $p_2=1$.
    
            \State Let $p=\min\{p_1,p_2\}$,compute $\xi_p=\xi_0-p\alpha\chi$ and assign $\theta \leftarrow \theta-\alpha\xi_p$.
        \EndWhile
        \end{algorithmic}
    \end{algorithm}

    Through the Algorithm \ref{CPBOS}, we validate the concept of modifying target loss functions by introducing preference parameters. However, directly determining suitable values for these parameters, denoted as $c$, poses a challenge. Therefore, our approach considers leveraging existing opponent shaping methods to identify appropriate values for $c$, which forms the foundational idea of our algorithm.

    In a two-agent game setting, incorporating the opponent's loss function involves adjusting the loss function of agent $i$, denoted as $f_i(L_1, L_2)$, where $i = 1, 2$. This function $f_i$ should exhibit local monotonicity with respect to both variables and is assumed to be differentiable. Specifically, $f_i(L_1, L_2)$ can be expressed as $\sum_{j=1}^{n} c_{ij} L_j$, with $c_{ii} = 1$, making it particularly suitable for local learning methods such as gradient descent and SOS.

    Building upon these principles, this paper introduces the concept of the preference parameter $c$, which represents the degree of goodwill exhibited by agents towards their opponents. A higher value of $c$ indicates a stronger cooperative inclination. For instance, $c > 0$ signifies a cooperative strategy, while $c < 0$ implies a more competitive stance.

    By incorporating the preference parameter $c$, the original loss functions of the two agents, $L_1$ and $L_2$, are adjusted to $L_1 + c_1 L_2$ and $L_2 + c_2 L_1$, respectively. Subsequently, the SOS algorithm is utilized to update the parameter $\theta$ based on these modified loss functions.
    
    \begin{align}
        &\Delta L_1=[L_1(\theta_1',\theta_2') + c_1L_2(\theta_1',\theta_2')]-[L_1(\theta_1,\theta_2)+c_1L_2(\theta_1,\theta_2],
        \\&\Delta L_2=[L_2(\theta_1',\theta_2') + c_2L_1(\theta_1',\theta_2')]-[L_2(\theta_1,\theta_2)+c_2L_1(\theta_1,\theta_2)],
    \end{align}
    where $\theta_1'=\theta_1+\Delta \theta_1$ and $\theta_2'=\theta_2+\Delta \theta_2$. Here, $\Delta \theta = -\alpha \xi_p$, with $\xi_p$ calculated as per the SOS \citep{letcher2018stable} and $\alpha$ being the learning rate.

    Then the above expression is expanded by first-order Taylor with respect to the parameter $\theta$:
    \begin{align}
        &\Delta L_1\approx\nabla_1 (L_1 + c_1 L_2)\cdot\Delta \theta_1 + \nabla_2 (L_1 + c_1 L_2)\cdot \Delta \theta_2 ,
        \\&\Delta L_2\approx\nabla_1 (L_2 + c_2 L_1)\cdot \Delta \theta_1 + \nabla_2 (L_2 + c_2 L_1)\cdot \Delta \theta_2,
        \label{delta L_1,L_2}
    \end{align}
    where the gradients are evaluated at the parameters $\theta_1$ and $\theta_2$, and $\Delta \theta=-\alpha \xi_p$ as previously described.

    Next, we compute the derivatives of $\Delta L_1$ and $\Delta L_2$ with respect to the preference parameter $c$ using the gradient descent method. In this context, the derivation of $c$ is fundamentally similar to the process used in the SOS algorithm. However, this approach can introduce non-stationarity into the system due to the sequential learning of two parameters. Moreover, learning $c$ in this manner may exacerbate the system's dynamics: in scenarios where one agent demonstrates goodwill while the other displays hostility, the learning process can deteriorate over iterations. Specifically, the malevolent agent might exploit the goodwill of their counterpart, leading to increased hostility, while the benevolent agent may attempt to further cultivate a positive relationship.

    In real-life interactions, acts of goodwill often encourage reciprocity, fostering a positive feedback loop in relationships. However, unilateral goodwill that is met with indifference or exploitation can yield unfavorable outcomes for the initiator. This mirrors the complexities of human social dynamics, where individuals are not purely rational and have preferences for reciprocity. To address this complexity, individuals typically adjust their strategies based on the responses to their initial goodwill gestures: indifference or negative responses may reduce future goodwill efforts, whereas positive responses may encourage further acts of kindness. Our study aims to simulate such preference dynamics.

    Therefore, this paper adopts a strategy to model the dynamics of goodwill within a game-theoretic framework. Rather than directly learning the preference parameter $c$, we integrate the shaping of opponent preference changes into the learning process. As agents learn $c$, they concurrently monitor and model changes in their opponents' goodwill adjustments relative to their own. This is represented by $\Delta c_{-i} = g_i(\Delta c_i) + \epsilon_i$, where $-i$ denotes the opponent of agent $i$, and $\epsilon_i$ represents noise. The function $g_i$ is differentiable, particularly suitable for local learning methods, and can be approximated as $\Delta c_{-i} = K_i \Delta c_i + \epsilon_i$, with $K$ estimated using the following equations:
    \begin{align}
        &S_{i,t} = \gamma S_{i,t-1} + (c_{i,t-1} - c_{i,t-2})^2, \quad i = 1, 2;\nonumber\\
        &r_t = \gamma r_{t-1} + (c_{1,t-1} - c_{1,t-2})(c_{2,t-1} - c_{2,t-2});\nonumber
        \\&K_{i,t} = \frac{r_t}{S_{i,t}}, \quad i = 1, 2.
        \label{calculate-K}
    \end{align}
    Here, $t \in \mathbb{N}$ denotes the learning step, and $\gamma \in (0, 1)$ represents the discount factor. As the learning steps increase, given $\gamma\in(0,1)$, the $\lim_{t\to\infty}\frac{r_t}{S_{i,t}}\approx \lim_{t\to\infty}\frac{\Delta c_{-i,t}}{\Delta c_{i,t}}, i=1,2$ can serve as a simple yet effective estimator of $K_i,i=1,2$. Notably, if $\|S_{1,t} \cdot S_{2,t}\|_2 \leq 0.01$, indicating approximate equality in changes between both agents, $K_{1,t} = K_{2,t} = 1.00$. 

    In this study, our target loss functions remain in their original forms, expressed as $L_1+c_1L_2$ and $L_2+c_2L_1$. When differentiating these loss functions, except for the coefficients $c_1$ and $c_2$ present in the variables $\Delta \theta_1$ and $\Delta \theta_2$ (\ref{delta L_1,L_2}), the rest are treated as constants. Additionally, when differentiating with respect to $c_i$, the coefficient $c_{-i}$ is substituted according to the relation $c_{-i}=K_ic_i$. Consequently, the partial derivative of $c_i$ with respect to $c_{-i}$ is explicitly given by $\nabla_{c_i}(c_{-i})=K_i$, where $i=1,2$.
    
    Define the stop-gradient operator as $\perp$. Based on these considerations, we derive the following gradient expressions:
    \begin{align}
        \nabla_{c_1}(\Delta L_1)&=\nabla_{c_1}\left\{\perp(\nabla_1 (L_1 + c_1L_2)) \cdot \Delta \theta_1 + \perp(\nabla_2 (L_1 + c_1L_2))\cdot \Delta \theta_2\right\} \nonumber
        \\&=\nabla_1 (L_1 + c_1L_2)\cdot \nabla_{c_1}(\Delta \theta_1) + \nabla_2 (L_1 + c_1L_2) \cdot \nabla_{c_1}(\Delta \theta_2) \nonumber
        \\&\approx\nabla_1 (L_1 + c_1L_2)\cdot (-\alpha \nabla_1L_2) + \nabla_2 (L_1 + c_1L_2)\cdot(-\alpha\cdot K_1\cdot \nabla_2L_1)
        \label{pbos_c1}
    \end{align}
    \begin{align}
        \nabla_{c_2}(\Delta L_2)&= \nabla_{c_2}\left\{\perp(\nabla_1(L_2+c_2L_1))\cdot \Delta \theta_1 + \perp(\nabla_2(L_2+c_2L_1))\cdot \Delta \theta_2 \right\} \nonumber
        \\&=\nabla_1 (L_2 + c_2L_1)\cdot \nabla_{c_2}(\Delta \theta_1) + \nabla_2 (L_2 + c_2L_1)\cdot \nabla_{c_2}(\Delta \theta_2) \nonumber
        \\&\approx\nabla_1 (L_2 + c_2L_1)\cdot (-\alpha \cdot K_2 \cdot \nabla_1L_2) + \nabla_2 (L_2 + c_2L_1)\cdot(-\alpha \nabla_2L_1).
        \label{pbos_c2}
    \end{align}
    The aforementioned approximations presented in Equations (\ref{pbos_c1}) and (\ref{pbos_c2}) are derived from Equations (\ref{approx theta}), which will be elaborated in detail within the subsequent Section \ref{sec4}.
    
    It is crucial to acknowledge that higher-order terms of $\alpha$ have been omitted in our analysis, with specific justification detailed in Remark \ref{l1}. The final update formula for the preference parameter $c$ is as follows:
    \begin{equation}
        c_i=c_i-\beta\cdot \nabla_{c_i}(\Delta L_i),i=1,2.
    \end{equation}
    Here, $\beta$ represents the learning rate, which is designed to decrease progressively throghout the learning steps.

    The specific algorithmic procedure of PBOS is outlined in Algorithm \ref{PBOS}.

    In summary, our algorithm, referred to as PBOS, comprises three main components. Initially, we incorporate the preference parameter $c$ into the loss functions, thereby adapting these functions for both agents. Subsequently, the parameter $\theta$ is iteratively updated using the SOS algorithm, which operates on the modified loss function. In the second phase, historical data is utilized to model the dynamics of the opponent's preference parameter $c$. Finally, the preference parameter $c$ undergoes optimization via gradient descent.

    \begin{algorithm}
        \caption{PBOS}\label{PBOS}
        \begin{algorithmic}[1]
        \State Initialize $\theta$ randomly, set the preference parameters $c_1,c_2$ and $K=[1,1]$, and the hyperparameters $a$ and $b$ within the interval (0, 1).
        \While{not done}
            
            \State Define the modified loss functions as $L_1'=L_1+c_1L_2, L_2'=L_2+c_2L_1$.
            
            \State Compute $\xi_0=(I-\alpha H_o)\xi$ and $\chi = \text{diag}(H_o^T\nabla L)$ at $\theta$ based on $L_1',L_2'$
    
            \State \textbf{if} $\langle -\alpha\chi,\xi_0\rangle >0$ \textbf{then} $p_1=1$ \textbf{else} $p_1=\min\{1,\frac{-a\|\xi_0\|^2}{\langle -\alpha\chi,\xi_0\rangle}\}$.
    
            \State \textbf{if} $\|\xi\|<b$ \textbf{then} $p_2=\|\xi\|^2$ \textbf{else} $p_2=1$.
    
            \State Let $p=\min\{p_1,p_2\}$,compute $\xi_p=\xi_0-p\alpha\chi$ and assign $\theta \leftarrow \theta-\alpha\xi_p$.
    
            \State \textbf{Compute} $K_1,K_2$.
        
            \State \textbf{Update} and \textbf{Record} $c_1,c_2$.
    
            \State \textbf{Update} $\beta$.
            
            \State \textbf{Return} losses $L_1,L_2$.  
        \EndWhile
        \end{algorithmic}
    \end{algorithm}

    The fundamental difference between PBOS and CPBOS lies in PBOS's adaptive learning capability regarding preference parameters, contrasting with the fixed parameters initially set in CPBOS. PBOS achieves this adaptive learning by integrating opponent shaping techniques, which enable the model to dynamically adjust the preference parameter $c$ in response to the evolving competitive environment.

\section{Theoretical Results}\label{sec4}
In this section, we will give theoretical properties of PBOS.

    \begin{example}
    \emph{In the Tandem Game, selecting an appropriate value for the preference parameter $c$ is advantageous for both agents.}
    \end{example}
    \begin{proof}
        The original loss functions of the Tandem game are defined as follows: \begin{align*}
            &L_1=(x+y)^2-2x,\\&L_2=(x+y)^2-2y.
        \end{align*}
        Applying the SOS algorithm, we derive the conditions for convergence:
        \begin{align*}
            &\nabla_1L_1=2(x+y)-2=0;\\
            &\nabla_2L_2=2(x+y)-2=0.
        \end{align*}
        This leads to $x+y=1$, corresponding to the NE of the game, where both agents minimize their loss jointly at $x=y=0.5$, resulting in $L_1=L_2=0$ \citep{letcher2018stable}.

       Introducing the preference parameter with $c_1=c_2=1$, the revised objective functions become:
        \begin{align*}
            &L_1'= 2(x+y)^2-2(x+y),
            \\&L_2'= 2(x+y)^2-2(x+y).
        \end{align*}
        Applying the SOS algorithm to these recised objectives yields:
        \begin{align*}
            &\nabla_1L_1'=4(x+y)-2=0;\\
            &\nabla_2L_2'=4(x+y)-2=0.
        \end{align*}
        In this scenario, convergence occurs at $x+y = 0.5$. The solution minimizing losses for both agents is $x = y = 0.25$, resulting in $L_1 = L_2 = -0.25$. While this outcome does not constitute a NE, it yields lower losses compared to the NE of the game.
    \end{proof}

        This example illustrates the potential of preference parameters to reduce losses for both agents in the Tandem Game when appropriately chosen, forming the foundational premise of our research approach.
        
        \begin{remark}
            If $\alpha\to 0, \alpha>0$, higher-order term of $\alpha$ in the update of $\theta$ can be neglected when updating the preference parameter $c$.
            \label{l1}
        \end{remark}
        
        During the update of $\theta$ using the SOS algorithm, the changes at each step are given by:
        \begin{align*}
            \Delta \theta_1&=-\alpha \nabla_1(L_1+c_1L_2)
            \\&\quad+\alpha^2\nabla_{12}(L_1+c_1L_2)\cdot \nabla_2(L_2+c_2L_1)
            \\&\quad+\alpha^2 p\nabla_{21}(L_2+c_2L_1)\cdot \nabla_2(L_1+c_1L_2);
            \\\Delta \theta_2&=-\alpha \nabla_2(L_2+c_2L_1)
            \\&\quad+\alpha^2\nabla_{21}(L_2+c_2L_1)\cdot \nabla_1(L_1+c_1L_2)
            \\&\quad+\alpha^2 p\nabla_{12}(L_1+c_1L_2)\cdot \nabla_1(L_2+c_2L_1).
        \end{align*}
        Here, $p$ is a constant described in \citep{letcher2018stable}, starting at $1$ and decreasing such that $\frac{p}{\alpha^2} \to 0$ as the learning progresses.  Thus, when $\alpha \to 0$, the higher-order terms involving $\alpha$ in $\Delta \theta_i$ for $i = 1, 2$ can be disregarded in favor of the first-order term of $\alpha$. 
        
        Based on the considerations above, when updating the preference parameter $c,$ we can approximate the following expression:
        \begin{align}
            \Delta \theta_1 \approx -\alpha \nabla_1(L_1+c_1L_2),\quad \Delta \theta_2\approx -\alpha \nabla_2(L_2+c_2L_1)
            \label{approx theta}
        \end{align}
        This simplification assumes that the opponent behaves like a ``naive" agent \citep{foerster2017learning} during the update of the preference parameter $c$, akin to applying direct gradient descent to parameter updates $\theta$. This assumption aligns with established principles in \citep{foerster2017learning}.
        
    \begin{lemma}
        If $\alpha \to 0,\alpha>0$ and $\beta \to 0,\beta>0$, then $\cfrac{\Delta c} {\|\Delta \theta\|}\ll 1$, where $\Delta c$ represents the change in the preference parameter $c$, and $\Delta \theta$ represents the change in the parameter $\theta$ during the update process.\footnote{Throughout the entire paper, $0<\beta\ll \alpha\ll 1.$ The norm employed throughout this paper is the 2-norm.}
        \label{beta<<0}
    \end{lemma}

    \begin{proof}
        Based on  Remark \ref{l1} and the given Equations (\ref{pbos_c1}) and (\ref{pbos_c2}), the gradients of the loss difference with respect to the preference parameters $c_1$ and $c_2$ are given by:
        \begin{align}
            &\nabla_{c_1}(\Delta L_1)=\nabla_1(L_1 + c_1L_2)\cdot(-\alpha \cdot \nabla_1L_2) + \nabla_2 (L_1 + c_1L_2)\cdot(-\alpha\cdot K_1\cdot \nabla_2L_1), \nonumber
            \\&\nabla_{c_2}(\Delta L_2)=\nabla_1 (L_2 + c_2L_1)\cdot(-\alpha \cdot K_2 \cdot \nabla_1L_2) + \nabla_2 (L_2 + c_2L_1)\cdot(-\alpha \cdot \nabla_2L_1).
            \label{nabla c12}
        \end{align}
            
        The equality holds in this Equation (\ref{nabla c12}) because approximations have been made during the actual update process. The change in the preference parameter $c$ at each step is then:
        \begin{align}
            &\Delta c_1=-\beta \cdot \nabla_{c_1}(\Delta L_1),  \nonumber
            \\&\Delta c_2 = -\beta \cdot \nabla_{c_2}(\Delta L_2).
            \label{delta c12}
        \end{align}
        Given that $\alpha \to 0,\beta\to 0$, we compare the magnitude of the change in the preference parameter $c$ with the parameter $\theta$ at each step. The ratio of $\Delta c_1$ to $\|\Delta \theta_1\|$ can be expressed as:
        \begin{align*}
            \cfrac{\Delta c_1}{\|\Delta \theta_1\|}&=\cfrac{-\beta\left[\begin{aligned}
              &\nabla_1 (L_1 + c_1L_2)\cdot(-\alpha\nabla_1L_2)\\
              &\quad+\nabla_2(L_1 + c_1L_2)\cdot(-\alpha K_1\nabla_2L_1)
            \end{aligned}\right]}{\begin{aligned}
              &-\alpha\nabla_1(L_1+c_1L_2)\\
              &\quad+\alpha^2\nabla_{12}(L_1+c_1L_2)\cdot \nabla_2(L_2+c_2L_1)\\
              &\quad+\alpha^2p\nabla_{21}(L_2+c_2L_1)\cdot \nabla_2(L_1+c_1L_2)
            \end{aligned}}\\[1ex]
            &=\cfrac{\beta\left[\begin{aligned}
              &\nabla_1 (L_1 + c_1L_2)\cdot\nabla_1L_2\\
              &\quad+\nabla_2 (L_1 + c_1L_2)\cdot(K_1\nabla_2L_1)
            \end{aligned}\right]}{\begin{aligned}
              &-\nabla_1(L_1+c_1L_2)\\
              &\quad+\alpha\nabla_{12}(L_1+c_1L_2)\cdot \nabla_2(L_2+c_2L_1)\\
              &\quad+\alpha p\nabla_{21}(L_2+c_2L_1)\cdot \nabla_2(L_1+c_1L_2)
            \end{aligned}}
            \\&=\cfrac{\beta \cdot A_1}{A_2+\alpha\cdot A_3}.
        \end{align*}
        By assuming $\alpha$ and $\beta$ are set to be sufficiently small compared to the magnitude of any derivatives, the coefficients $A_i,i=1,2,3$ are effectively $\mathcal{O}(1)$ with respect to $\alpha$ and $\beta$. This yields $\cfrac{\Delta c_1}{\|\Delta \theta_1\|}=o(\beta)$.
            
        Thus, if $\alpha \to 0,\beta \to 0$, $\cfrac{\Delta c_1}{\|\Delta \theta_1\|}=o(\beta)\ll 1$. Similarly, $\cfrac{\Delta c_2}{\|\Delta \theta_2\|}=o(\beta)\ll 1$ under the same conditions. This completes the proof.
    \end{proof}
    \begin{corollary}
        Given Lemma \ref{beta<<0}, if $\alpha \to 0$ and $\beta\to 0$, the parameter $\theta$ can be considered to be approximately converged during the update of the preference parameter $c$.\footnote{As established by Lemma \ref{beta<<0}, the ratio $\Delta c /\Delta \theta \to 0$. This implies that $\theta$ is adjusted on a significantly faster timescale compared to $c$. Therefore, especially in the early stages of the learning process, $\theta$ is expected to converge to a stationary point more quickly than $c$.}
        \label{ignore c}
    \end{corollary}

    \begin{lemma}
        If $\theta$ can converge, then the preference parameter $c$ also converges. \label{c converge}
    \end{lemma}

    \begin{proof}
        According to Lemma \ref{beta<<0}, $\theta$ converges before $c$. When $\theta$ approaches to convergence, $\nabla_1(L_1+c_1L_2)=0$ and $\nabla_2(L_2+c_2L_1)=0$ can be approximated.

        Initially, if $\nabla_1L_1=0,\nabla_1L_2=0,\nabla_2L_2=0$ and $\nabla_2L_1=0$, $\theta$ has completely converged, resulting in $\Delta c_1=0$ and $\Delta c_2=0$, indicating algorithm convergence.
        
        Next, consider the scenario where not all gradients are 0. As $\theta$ converges to a stable point, the change in preference parameter $c$ per step is:
        \begin{align}
            \Delta c_1 &= -\beta \cdot \nabla_{c_1}(\Delta L_1) \nonumber
            \\&= -\beta \cdot [\nabla_1(L_1 + c_1L_2)\cdot(-\alpha \cdot \nabla_1L_2) + \nabla_2 (L_1 + c_1L_2)\cdot(-\alpha\cdot K_1\cdot \nabla_2L_1)] \nonumber
            \\ &=\alpha\beta\cdot \nabla_2(L_1+c_1L_2)\cdot (K_1\cdot \nabla_2L_1) \nonumber
            \\ &=\alpha\beta\cdot (1-c_1c_2)\cdot K_1 \cdot \|\nabla_2L_1\|^2.
        \end{align}
        The third equality holds by invoking the assumption that $\nabla_1(L_1+c_1L_2)=0$. The fourth equality is established by utilizing the assumption that $\nabla_2(L_2+c_2L_1)=0$, which implies $\nabla_2 L_2=-c_2\nabla_2L_1.$
        \begin{align}
            \Delta c_2 &= -\beta \cdot \nabla_{c_2}(\Delta L_2) \nonumber
            \\&= -\beta \cdot [\nabla_1(L_2 + c_2L_1)\cdot(-\alpha \cdot K_2\cdot  \nabla_1L_2) + \nabla_2 (L_2 + c_2L_1)\cdot(-\alpha\cdot \nabla_2L_1)] \nonumber
            \\ &=\alpha\beta\cdot \nabla_1(L_2+c_2L_1)\cdot (K_2\cdot \nabla_1L_2) \nonumber
            \\ &=\alpha\beta\cdot (1-c_1c_2)\cdot K_2 \cdot \|\nabla_1L_2\|^2.
        \end{align}
        The third equality holds by invoking the assumption that $\nabla_2(L_2+c_2L_1)=0$. The fourth equality is established by utilizing the assumption that $\nabla_1(L_1+c_1L_2)=0$, which implies $\nabla_1 L_1=-c_1\nabla_1L_2.$

        From formula (\ref{calculate-K}), $K_1K_2 = \frac{r^2}{S_1S_2} > 0$. Then, we have:
        \begin{equation}
            \Delta c_1 \Delta c_2=\alpha^2\beta^2\cdot (1-c_1c_2)^2\cdot K_1\cdot K_2 \cdot \|\nabla_2L_1\|^2\cdot \|\nabla_1L_2\|^2>0.
            \label{c1c2}
        \end{equation}
        \begin{equation}
            r=\sum\limits_{n=1}^T \gamma^{T-n}\Delta c_{1,n}\Delta c_{2,n}.
            \label{r}
        \end{equation}
        where $\gamma\in [0,1]$ represents the discount factor and $T$ is the learning step.

        From the above formula (\ref{c1c2}) and (\ref{r}), when $T$ is large enough, we have $r >0$. Here, the requirement that $T$ be sufficiently large stems from the assumptions that $\nabla_1(L_1+c_1L_2)=0$ and $\nabla_2(L_2+c_2L_1)=0$, which necessitates an adequate number of learning steps. Consequently, after a sufficient numbers of iterations, $K_i>0$ for $i=1,2$. Consequently, $\Delta c_1$ and $\Delta c_2$ will have the same sign, leading to either
        \begin{equation}
           \Delta c_1>0,\Delta c_2>0, 1- c_1c_2 >0 \quad \text{or} \quad  \Delta c_1<0,\Delta c_2<0, 1- c_1c_2 <0.
            \label{same sign} 
        \end{equation}

        There are four distinct cases to consider for the convergence of the preference parameters $c_1$ and $c_2$:

        There are four cases to consider.

        (a). $c_i>0$ for $i=1,2$.
        
        If $1-c_1c_2>0$, then $c_1,c_2$ will increase until $1-c_1c_2\leq 0$. If $1 - c_1c_2 > 0$, then $c_1$ and $c_2$ will incrementally increase until $1 - c_1c_2 \leq 0$. Conversely, if $1 - c_1c_2 < 0$, they will decrease until $1 - c_1c_2 \geq 0$. This process ensures convergence when $1 - c_1c_2 = 0$.

        (b). $c_1<0<c_2$ or $c_2<0<c_1$.
        
        $c_1c_2<0\Rightarrow1 - c_1c_2 > 0$, so both $c_1$ and $c_2$ will increase until they become positive, essentially converging towards case (a).

        (c). $c_1<0,c_2<0, 1-c_1c_2>0$.
        In this scenario, increasing both $c_1$ and $c_2$ (where $\Delta c_1 > 0$ for $i=1,2$) from an initial state where $1 - c_1c_2 > 0$ will lead them to conditions resembling either case (a) or (b).

        (d). $c_1<0,c_2<0, 1-c_1c_2<0$.

        Initially, $1 - c_1c_2$ starts positive, indicating that this case evolves from scenario (c). As the number of iterations increases, scenario (c) naturally transitions into either case (a) or (b), precluding the persistence of $c_1 < 0, c_2 < 0, 1 - c_1c_2 < 0$.

        In all scenarios, when $\theta$ reaches a stable point, $c$ converges to a suitable value. Thus, if $\theta$ converges, $c$ also converges.
    \end{proof}

        \begin{theorem}
            SOS converges locally to stable fixed points for $\alpha \to 0, \alpha>0$. 
        \label{sos converges}
        \end{theorem}

        \begin{proof}
            This theorem has already been proved in reference \citep{letcher2018stable}.
        \end{proof}

        \begin{theorem}
            The PBOS algorithm converges for $\alpha\to 0,\alpha>0$ and $\beta\to 0,\beta >0$.
        \end{theorem}

        \begin{proof}
            Given that $\alpha\to0,\alpha >0$ and $\beta\to0,\beta > 0$, Lemma \ref{beta<<0} establishes that $\frac{\Delta c}{\|\Delta \theta\|} \ll 1$, indicating a favorable condition for convergence. Building on this, Corollary \ref{ignore c} and Theorem \ref{sos converges} confirm that for $\alpha\to 0,\alpha > 0$, the parameter $\theta$ converges due to the local stability properties of the SOS algorithm, as detailed in \citep{letcher2018stable}.
            
            By Lemma \ref{c converge}, which asserts that $c$ converges when $\theta$ converges, we conclude that both $\theta$ and $c$ converge simultaneously. This collective convergence implies the overall convergence of the PBOS algorithm.
        \end{proof}

        \begin{definition}
            (Maximization of Cooperation). 
            Two agents are said to be in a state of Maximization of Cooperation if their objective is to minimize the same loss function, which may differ by a positive constant factor.
            \label{max cooperation}
        \end{definition}
        \begin{example}
            \emph{Consider two agents with loss functions $L_1$ and $L_2$. If these loss functions satisfy the relationship $L_1=c\cdot L_2$, where $c>0$ and $c$ is a real number $c\in \mathbb{R}$, then the two agents have achieved a state of} Maximization of Cooperation.

            \emph{Each agent's goal is to minimize their respective loss function. Since these loss functions are proportional to each other, optimizing one agent's objective inherently optimizes the other's as well, leading to a cooperative outcome.}
        \end{example}

        \begin{corollary}
            If the PBOS algorithm converges to $c_1c_2=1$, then both agents achieve a state of Maximization of Cooperation.
        \end{corollary}
        \begin{proof}
            Suppose the PBOS algorithm converges such that $c_1c_2=1$. The modified objectives for the two agents are given by:
            \begin{align*}
                &L_1'=L_1+c_1L_2,\\
                &L_2'=L_2+c_2L_1.
            \end{align*}
            Given that $1-c_1c_2=0$, we can derive the following relationship:
            \begin{equation}
                c_2L_1'=c_2L_1+c_1c_2L_2=L_2+c_2L_1=L_2'.
            \end{equation}
            This equation indicates that the objectives of both agents are equivalent up to a constant factor, which aligns with the definition of \emph{Maximization of Cooperation}. 
        \end{proof}
        \begin{remark}
        \emph{The convergence of the PBOS algorithm is not limited to the scenario where $c_1=c_2=1$. The algorithm can converge to values where $c_i \ne 1$ for $i=1,2$ (Fig. \ref{C of IPD-lc}, \ref{C of Ultimatum-lc}, \ref{C of Matching Pennies-lc}, \ref{C of Stackelberg Leader-lc}, \ref{C of Stag Hunt-lc}). Consequently, the final outcome is not merely a straightforward summation of the individual loss functions of the two agents.}
        \end{remark}
        \begin{remark}
        \emph{If the PBOS algorithm converges to a state where $c_1c_1\ne 1$, this indicates that the cooperation between the two agents is not maximized. This situation can arise even when $\theta$ has converged and the gradients $\nabla_1L_1=0,\nabla_1L_2=0,\nabla_2L_2=0$ and $\nabla_2L_1=0$. In such case, the preference parameters $c$ may still converge even if $c_1c_1\ne 1$. This suggests that even in the absence of} Maximization of Cooperation,\emph{ the final outcome of the algorithm can be considered satisfactory.}
        \end{remark}

\section{Experiment Environments}\label{sec5}

         We conducted an empirical evaluation to assess the efficacy of the PBOS algorithm across a spectrum of six distinct differentiable games: Tandem Game, Iterated Prisoner's Dilemma (IPD), Matching Pennies, Ultimatum Game, Stackelberg Leader Game, and Stag Hunt. 
         The six experimental environments include zero-sum games, non-zero-sum games with single or multiple Nash equilibria, and polynomial-form games. Compared to the baseline algorithms' environments, ours are more diverse, effectively highlighting the PBOS algorithm's superiority.
        
        Furthermore, we implemented PBOS alongside the aforementioned baseline algorithms to govern individual agents. Subsequently, these agents were engaged in a quartet of games characterized by symmetrical loss functions. This experimental setup aimed to scrutinize PBOS's effectiveness in fostering both cooperative and competitive outcomes.
        
	\subsection{Tandem Game}

        The Tandem game, introduced by  \citet{letcher2018stable}, is a polynomial game defined on the continuous space $\mathbb{R}^2$. The loss functions for agents 1 and 2 are specified as $L_1(x,y)=(x+y)^2-2x$ and $L_2(x,y)=(x+y)^2-2y$, respectively. This game structure yields symmetric SFPs at ${x+y=1}$, as identified in \citep{letcher2018stable}.
        
        It has been noted in previous research \citep{letcher2018stable, willi2022cola} that LOLA does not sustain these SFPs. Instead, LOLA converges to Pareto-dominated solutions, attributed to a phenomenon where both agents adopt an ``arrogant" strategy.

	\subsection{IPD}

        The Iterated Prisoner's Dilemma (IPD) serves as a paradigmatic model for elucidating the emergence of cooperation within complex dynamical systems \citep{May1981TheEO,Harper2017ReinforcementLP}. In the IPD, the payoffs are typically subject to a discount factor $\gamma\in [0,1]$, reflecting the reduced value of future payoffs relative to immediate ones. In this study, we have selected a discounted factor $\gamma=0.96$ to capture the temporal dynamics of the game. 

        The traditional payoff matrix for the IPD is presented in Table \ref{pm-ipd}, illustrating outcomes for each combination of strategies chosen by the two agents. In our formulation, each agent $i$ is characterized by five parameters, including the probability $P^i(C|\text{state})$ of cooperating given the initial state $s_0=\emptyset$ or any subsequent state $s_t=(a_{t-1}^1,a_{t-1}^2)$ for $t>0$ \citep{letcher2018stable}. These probabilities, contingent on the current game state, play a crucial role in shaping the strategic decisions of the agents throughout the IPD.

        The game features two NEs. One is the always-defect strategy (DD), resulting in a loss of $2$ for both agents. The other NE is the \emph{tit-for-tat} (TFT) strategy, where agents initially cooperate and subsequently mirror the opponent's previous action \citep{letcher2018stable}. TFT incurs a loss of $1$ for each agent, recognized as a simple yet effective strategy \citep{Axelrod1980EffectiveCI}.

	\subsection{Matching Pennies}

        The Matching Pennies game, introduced in \citep{Lee1967TheAO}, is a quintessential example of a zero-sum game. The game's structure is characterized by a payoff matrix that is displayed in Table \ref{pm-mp}.

    In this game, each agent's policy is encapsulated by a single paremeter, which is the probability of choosing the ``heads" option \citep{willi2022cola}. It is a well-established result that the unique NE for this game involves each player choosing their strategy with equal probability, specifically, a probability of 0.5 for each \citep{Budinich2011RepeatedMP}.

	\subsection{Ultimatum Game}
	The single-shot Ultimatum game, extensively studied in various literature  \citep{guth1982experimental,sanfey2003neural,oosterbeek2004cultural,smith2006foundations,Sandoval2015ReciprocityIH,willi2022cola}, exists in numerous variants. In this study, we focus on a version where two players are tasked with dividing ten dollars. Player 1 proposes a split to Player 2, who then decides to accept or reject the offer. Rejection leads to a null outcome for both players, whereas acceptance results in the proposed division of funds. 

    The principle of backward induction suggests that Player 2 will accept any positive offer, as it is preferable to receiving nothing \citep{oosterbeek2004cultural}. Anticipating this, Player 1 may offer the minimal possible amount. However, this study considers two notable solutions: an equitable split of five dollars each or an inequitable division of eight dollars for Player 1 and two dollars for Player 2 \citep{Falk1999OnTN}.

    In this context, Player 1's strategy is parameterized by the log-odds of proposing a fair split, denoted as $p_{fair} = \sigma(\theta_1)$. Similarly, Player 2’s strategy is captured by the log-odds of accepting an inequitable split, given that equitable split is always accepted, represented as  $p_{accept} = \sigma(\theta_2)$ \citep{willi2022cola}. The loss functions of the two players are definited as follows:
    \begin{align*}
        &L_1=-(5P_{fair}+8(1-P_{fair})P_{accept}),
        \\&L_2=-(5P_{fair}+2(1-P_{fair})P_{accept}).
    \end{align*}

    \subsection{Stackelberg Leader Game}
    
    The Stackelberg Leader game, as discussed in \citep{VANDAMME1999105}, is characterized by a payoff matrix depicted in Table \ref{pm-Stackelberg Leader}. This game is known for having a unique NE, which is the strategy profile (D, L). However, this equilibrium does not correspond to the most optimal outcome within the game's framework.

    Traditional analyses often assume heterogeneity among the players in the Stackelberg Leader game: one player acts as the leader, making the initial move, while the other serves as the follower, responding subsequently \citep{hu2023modeling}. Under this assumption, a more advantageous outcome, such as (3, 2), can be achieved.

    In this study, leveraging our algorithm, we demonstrate convergence towards the superior outcome of (3, 2) without relying on the assumption of player heterogeneity, as illustrated in Fig. \ref{Stackelberg Leader-lc}. This discovery indicates that our algorithm effectively navigates the strategic complexities inherent in the Stackelberg Leader game, potentially uncovering cooperative or efficient solution paths that may not be readily apparent under traditional game-theoretic assumptions.
  
    \subsection{Stag Hunt}
    The Stag Hunt game, originally introduced in \citep{rousseau1984inequality}, serves as a classic illustration of cooperation challenges in game theory. It depicts a scenario where two hunters must decide between pursuing a large stag quietly or switching to hunt a hare that suddenly appears.

    The payoff matrix for the Stag Hunt game is presented in Table \ref{pm-Stag Hunt}. This game features two pure strategy NEs: the Stag NE, where both hunters cooperate to hunt the stag, and the Hare NE, where both opt to hunt the hare.
    
    The Stag NE, characterized by mutual cooperation in pursuing the stag, yields a higher collective payoff of 4 and is often viewed as a``risky" strategy \citep{Tang2021Discovering}. This risk stems from the possibility that if one hunter defects and hunts the hare instead, they can still secure a lower but positive payoff of 3, while the other, who continues pursuing the stag, incurs a significant loss of -10.

    In contrast, the Hare NE represents a ``safe" non-cooperative equilibrium where both hunters choose the less lucrative but certain payoff from hunting the hare, resulting in a payoff of 1 for each participant \citep{Tang2021Discovering}. This equilibrium ensures a positive outcome for each hunter individually, irrespective of the other's decision, but at the expense of forgoing the higher collective payoff achievable through mutual cooperation.

        \begin{table}[h]
        \begin{minipage}[t]{0.4950\textwidth}
            \caption{IPD}\label{pm-ipd}
	   \begin{tabular*}{0.3\textheight}{@{\extracolsep\fill}ccc}
            \toprule
            & C  & D \\
            \midrule
            C    & (-1,-1)  & (-3,0)  \\
            D    & (0, -3)   & (-2,-2) \\
            \botrule
            \end{tabular*}
        \end{minipage}
        \hfill
        \begin{minipage}[t]{0.4950\textwidth}
        \caption{Mathcing Pennies}\label{pm-mp}
        \begin{tabular*}{0.3\textheight}{@{\extracolsep\fill}ccc}
            \toprule
            & H  & T \\
            \midrule
            H    & (1,-1)  & (-1,1)  \\
            T    & (-1,1)  & (1,-1)  \\
            \botrule
            \end{tabular*}
        \end{minipage}

        \begin{minipage}[t]{0.4950\textwidth}
            \caption{Stackelberg Leader}\label{pm-Stackelberg Leader}
            \begin{tabular*}{0.3\textheight}{@{\extracolsep\fill}ccc}
            \toprule
            & L  & R \\
            \midrule
            U    & (1,0)  & (3,2)  \\
            D    & (2,1)  & (4,0)  \\
            \botrule
            \end{tabular*}
        \end{minipage}
        \hfill
        \begin{minipage}[t]{0.4950\textwidth}
        \caption{Stag Hunt}\label{pm-Stag Hunt}
        \begin{tabular*}{0.3\textheight}{@{\extracolsep\fill}ccc}
            \toprule
            & Stag  & Hare \\
            \midrule
            Stag    & (4,4)  & (-10,3)  \\
            Hare    & (3,-10)  & (1,1)  \\
            \botrule
            \end{tabular*}
        \end{minipage}
    \end{table}

\section{Experiment Results}  \label{sec6} 
    In this section, we present a series of experiments to demonstrate the performance of the PBOS algorithm and discuss its limitations.
    \subsection{Fixed Preference Parameters \texorpdfstring{$c$}{c}} 
    We have developed an experimental framework to assess the efficacy of the CPBOS algorithm in conjunction with three established baseline algorithms. Our primary objective is to empirically validate the integration of preference parameters into the learning process. The experimental results are depicted in Fig. \ref{tandem}-\ref{Stag Hunt}, presenting a visual comparison of algorithm performance.

\begin{figure}[h!]
    \centering
    
    \begin{subfigure}[b]{0.495\textwidth}
        \centering
        \includegraphics[width=\textwidth]{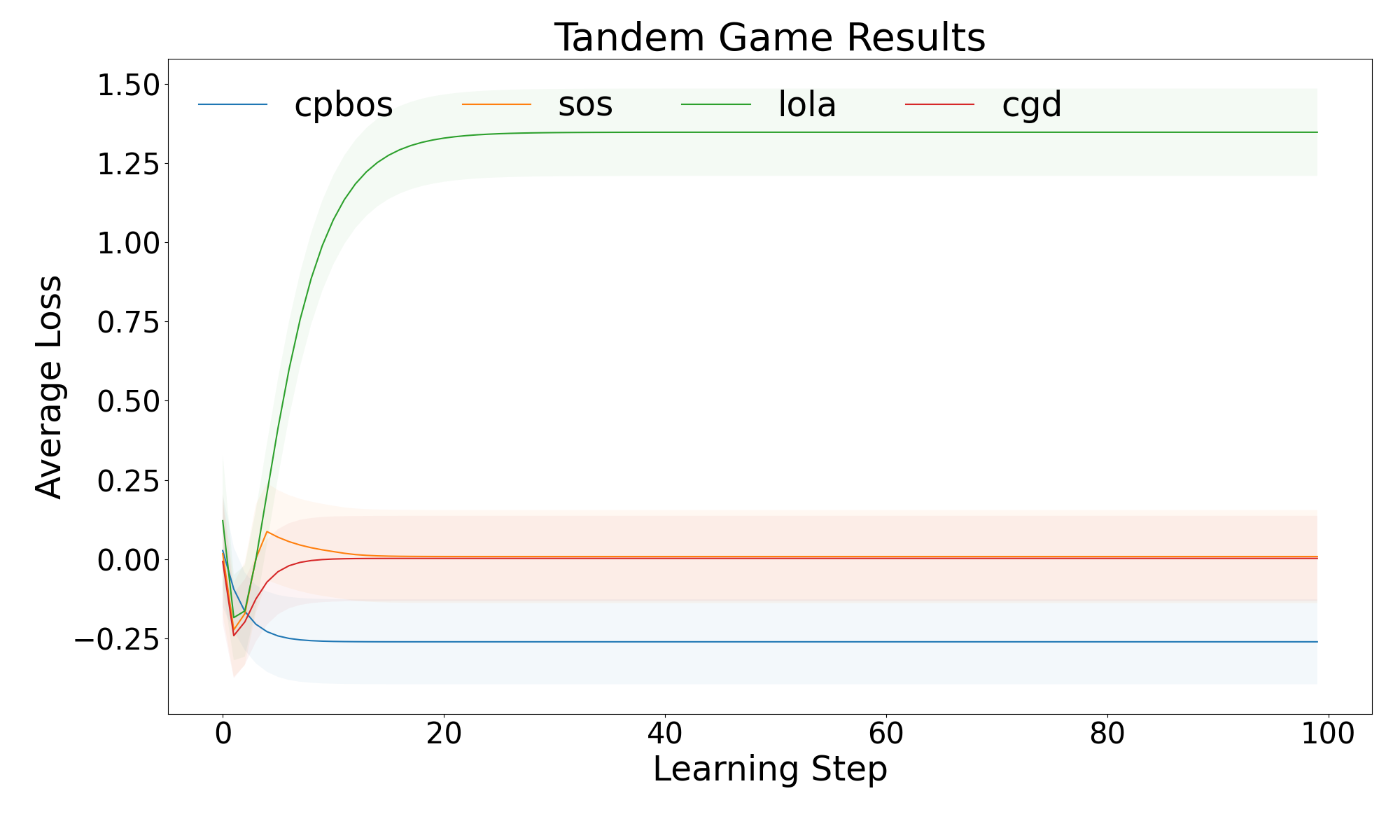}
        \caption{Tandem}
        \label{tandem}
    \end{subfigure}
    \hfill
    \begin{subfigure}[b]{0.495\textwidth}
        \centering
        \includegraphics[width=\textwidth]{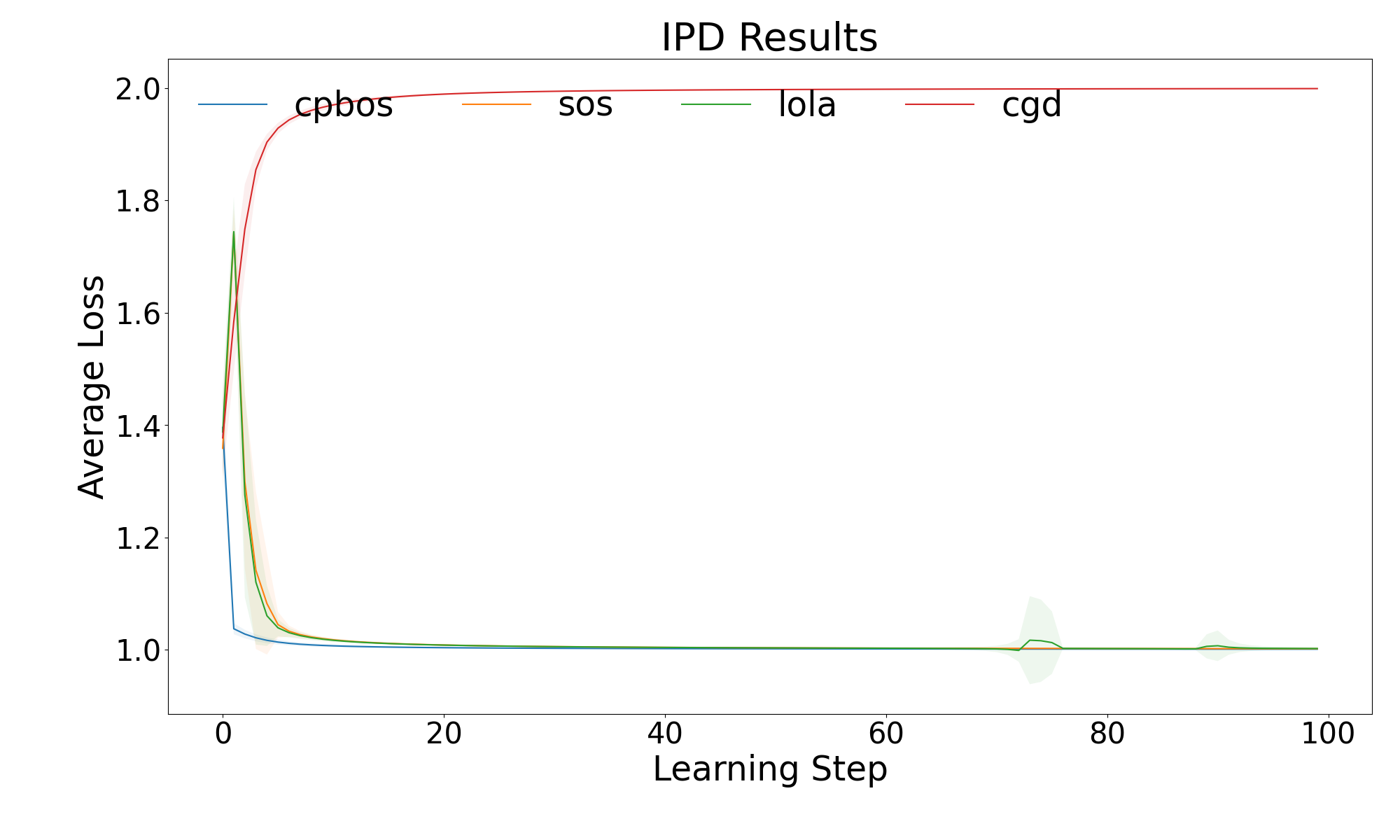}
        \caption{IPD}
        \label{ipd}
    \end{subfigure}

    \begin{subfigure}[b]{0.495\textwidth}
        \centering
        \includegraphics[width=\textwidth]{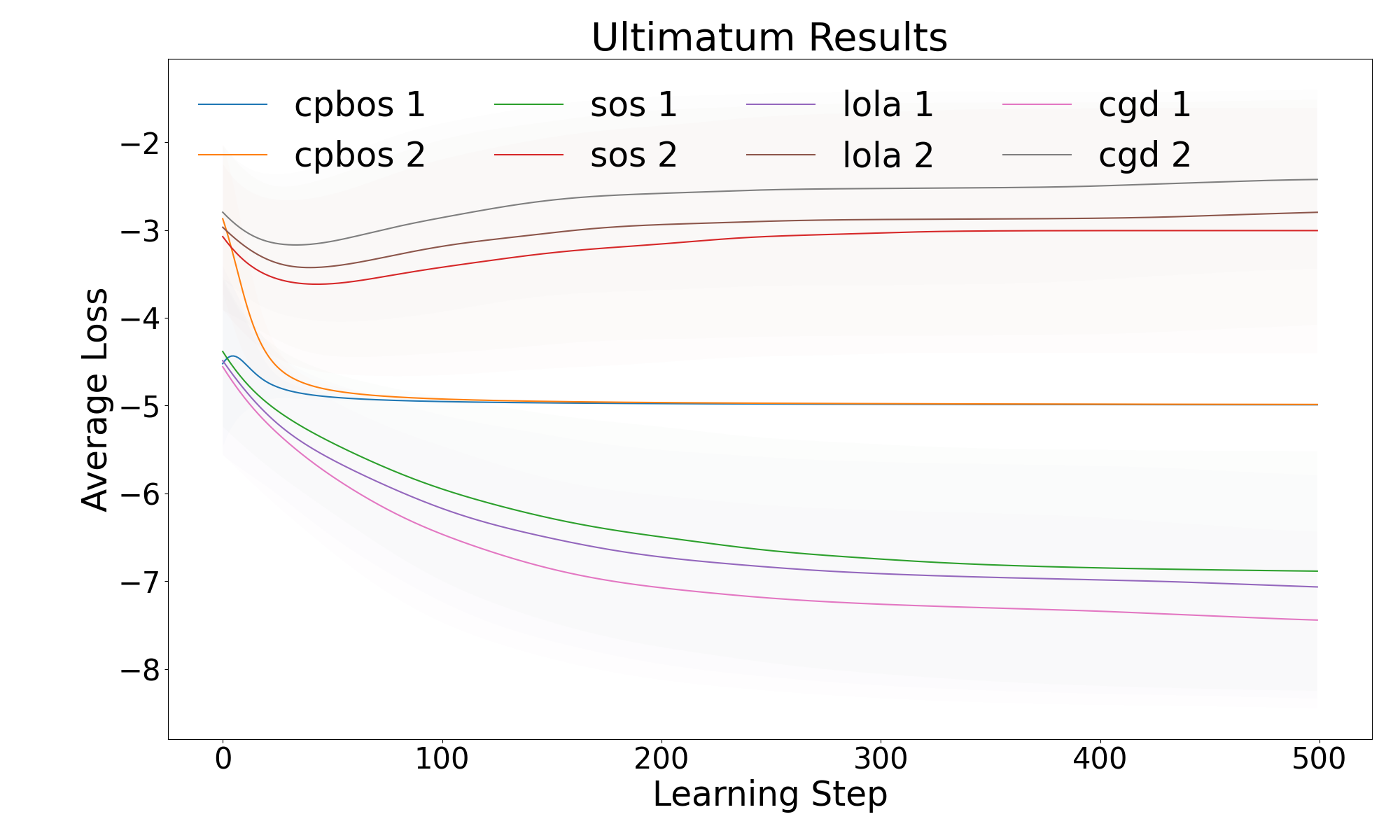}
        \caption{Ultimatum}
        \label{Ultimatum}
    \end{subfigure}
    \hfill
    \begin{subfigure}[b]{0.495\textwidth}
        \centering
        \includegraphics[width=\textwidth]{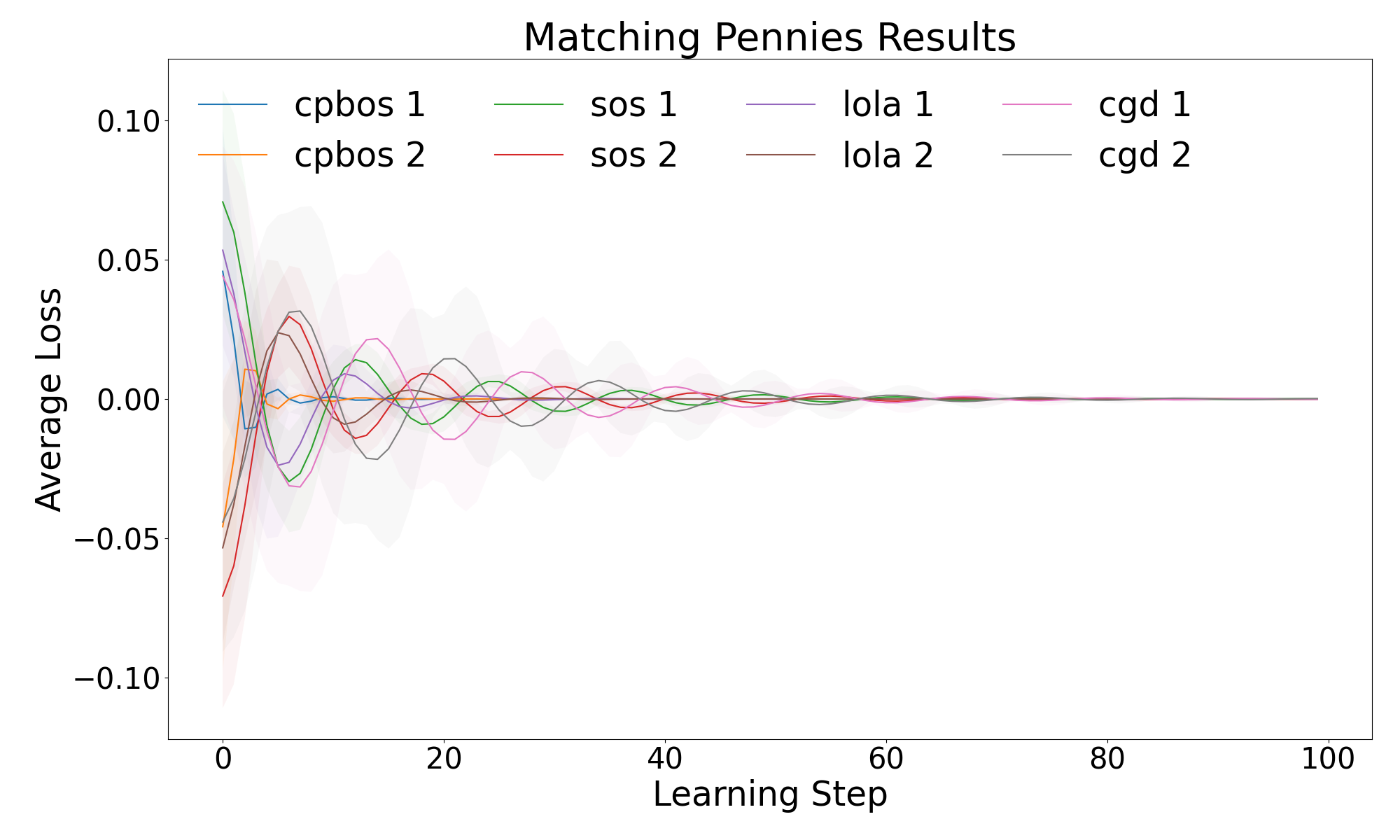}
        \caption{Matching Pennies}
        \label{mp}
    \end{subfigure}

    \begin{subfigure}[b]{0.495\textwidth}
        \centering
        \includegraphics[width=\textwidth]{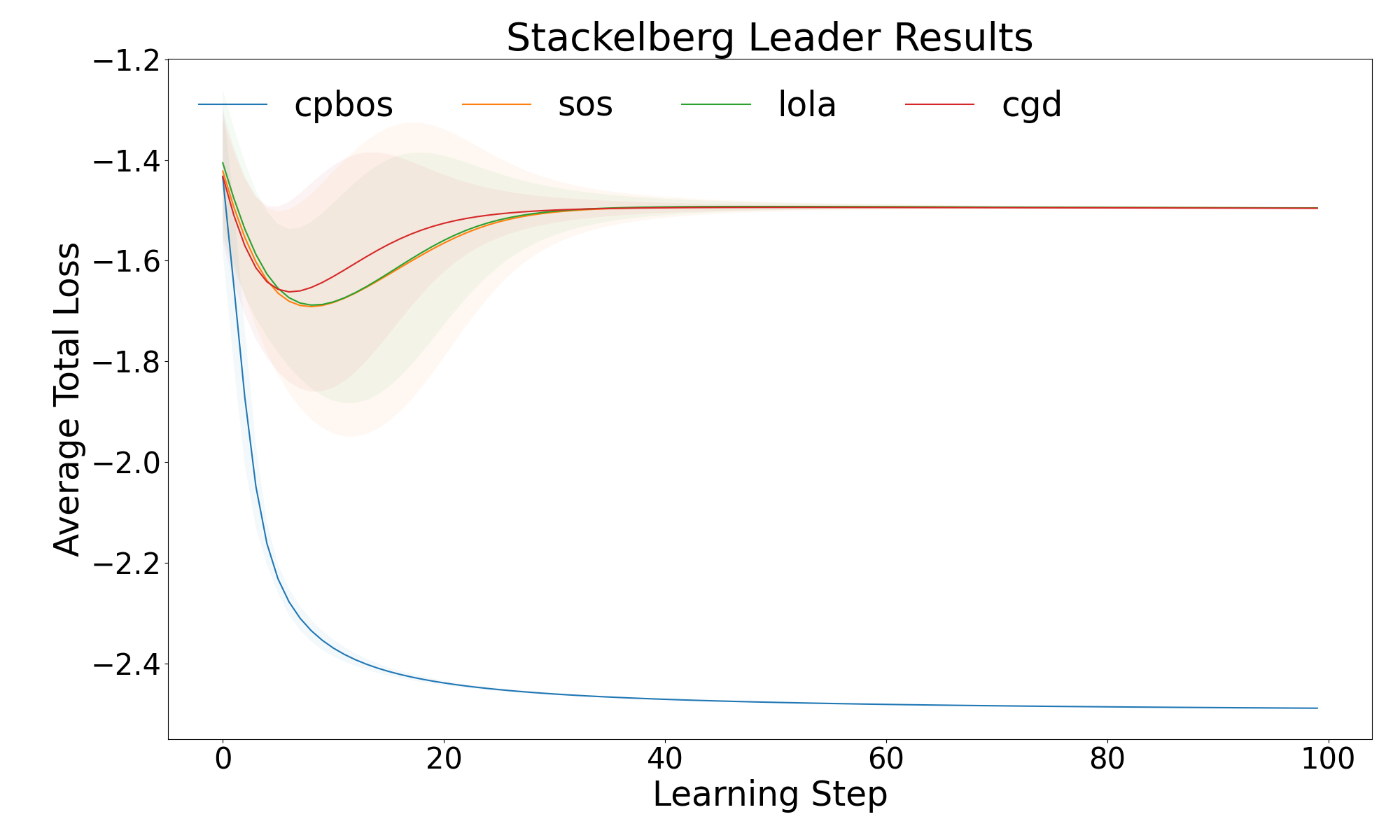}
        \caption{Stackelberg Leader}
        \label{Stackelberg Leader}
    \end{subfigure}
    \hfill
    \begin{subfigure}[b]{0.495\textwidth}
        \centering
        \includegraphics[width=\textwidth]{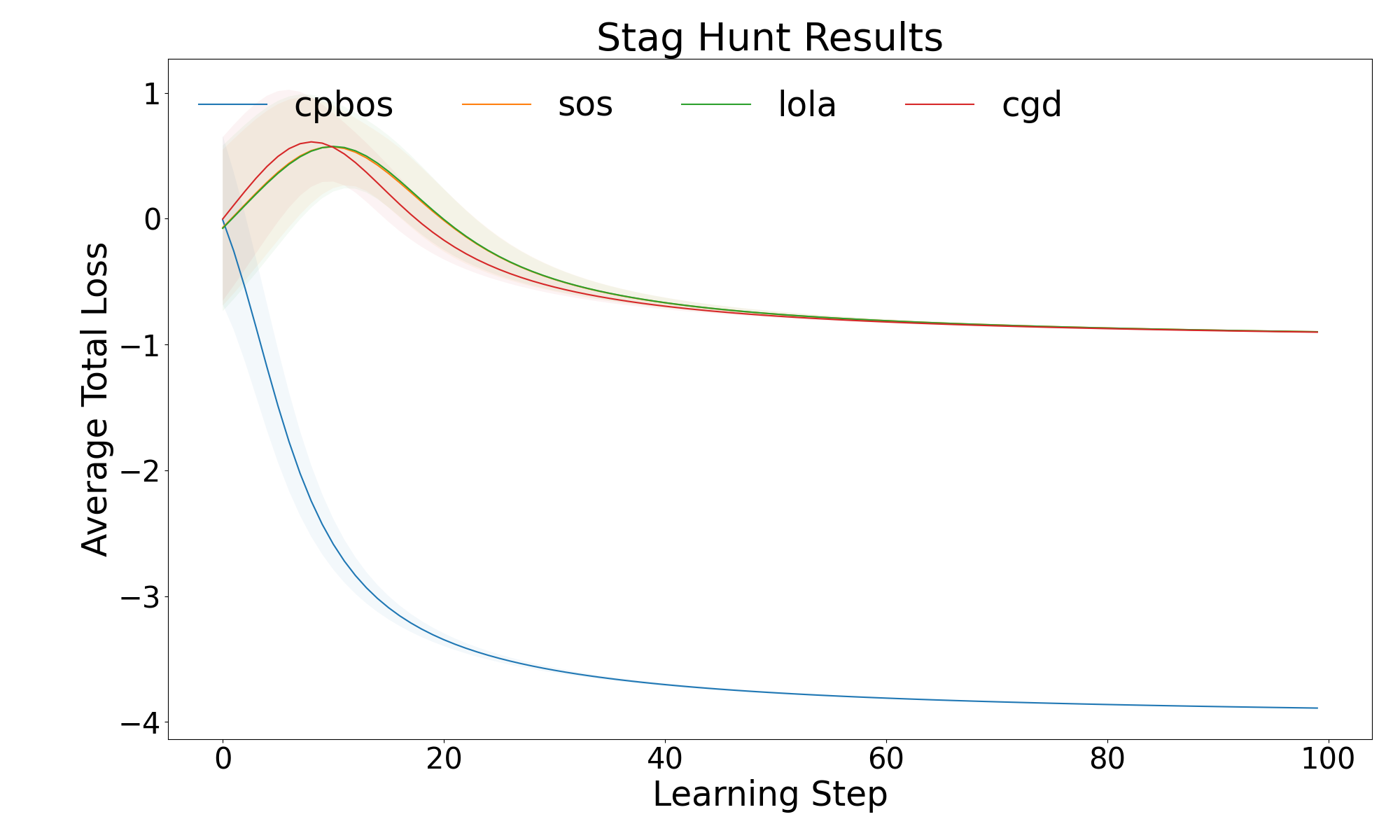}
        \caption{Stag Hunt}
        \label{Stag Hunt}
    \end{subfigure}

    \caption{Results from 100 independent executions of each algorithm. `cpbos' represents our proposed algorithm, whereas `sos', `lola', and `cgd' are the baseline algorithms. Figures \ref{tandem} and \ref{ipd} show the loss of agent 1, while Figures \ref{Ultimatum} and \ref{mp} display the losses of both agents. Figures \ref{Stackelberg Leader} and \ref{Stag Hunt} present the average losses of the two agents. For Figures \ref{tandem}, \ref{ipd}, \ref{Stackelberg Leader}, and \ref{Stag Hunt}, lower loss values reflect better performance. For Figures \ref{Ultimatum} and \ref{mp}, the closer the losses of the two agents are, the better the performance. Across all scenarios, the our algorithm demonstrates superior performance compared to the baselines.}
\end{figure}

        In the context of the Tandem Game and IPD, our analysis centers exclusively on the loss function of the first agent. Conversely, when studying the Ultimatum Game and Matching Pennies game, we depict the loss functions for both participating agents. For the Stackelberg Leader Game and Stag Hunt Game, our representation involves plotting the average combined loss incurred by both players. From Fig. \ref{tandem}-\ref{Stag Hunt}, we derive the following conclusions.

        (a). In the Tandem Game, IPD, Stackelberg Leader Game, and Stag Hunt Game, we set $c_1=c_2=1.00$, reflecting a cooperative stance adopted by both agents towards their counterparts. This mutual goodwill leads to favorable outcomes for both parties. Specifically, in the Tandem Game, it results in reduced losses for both agents (Fig. \ref{tandem}). In the IPD, the CPBOS algorithm exhibits quicker convergence towards cooperative behavior compared to the three baseline algorithms (Fig. \ref{ipd}). In the Stackelberg Leader Game and Stag Hunt Game, agents utilizing CPBOS identify reward structures that are more advantageous than the traditional NE (Fig. \ref{Stackelberg Leader} and \ref{Stag Hunt}).

        (b). In the Ultimatum Game, the parameters are set as $c_1=1.00$ and $c_2=-1.00$, indicating a cooperative stance by agent 1 and a confrontational stance by agent 2 towards their counterparts. This dynamic results in an equitable division of the ten-dollar pool between the two agents (Fig. \ref{Ultimatum}).

        (c). In the Matching Pennies, the scenario is characterized by $c_1=c_1=-1.00$, with both agents exhibiting hostility towards their opponents. This antagonistic interaction aligns with the zero-sum nature of the game, yet both agents achieve a mixed strategy NE (Fig. \ref{mp}).

        The detailed numerical results of our experiments are systematically presented in Table \ref{table-cc}, providing a comprehensive view of the performance metrics across various games and algorithms.
        \begin{table}[h]
            \caption{Results of CPBOS and three baseline algorithms across six different games.}\label{table-cc}
            \setlength\tabcolsep{1pt}
            \begin{tabular*}{\textwidth}{@{\extracolsep\fill}ccccccc}
            \toprule%
            & Tandem & IPD & Ultimatum & \shortstack{Matching\\Pennies} & \shortstack{Stackelberg\\Leader} & Stag Hunt  \\
            \midrule
           CPBOS & \textbf{(-0.24,-0.26)} & \textbf{(1.00,1.00)} & \textbf{(-5.00,-5.00)} & \textbf{(0.00,0.00)} & \textbf{(-3.00, -2.00)} & \textbf{(-3.89,-3.89)}  \\
           LOLA & (1.35,1.28) & \textbf{(1.00,1.00)} & (-7.20,-2.80) & \textbf{(0.00,0.00)} & (-2.02,-0.97) & (-0.90,-0.90)  \\
           SOS & (-0.01,0.01) & \textbf{(1.00,1.00)} & (-6.89,-3.01) & \textbf{(0.00,0.00)} & (-2.02,-0.97) & (-0.90,-0.90)  \\
           CGD & (-0.01,0.01) & (2.00,2.00) & (-7.44,-2.43) & \textbf{(0.00,0.00)} & (-2.02,-0.98) & (-0.90,-0.90) \\
            \botrule
            \end{tabular*}
            \footnotetext{Note: In the Fig. \ref{tandem},\ref{ipd},\ref{Stackelberg Leader} and \ref{Stag Hunt}, $c_1=c_2=1.00$. In the Fig. \ref{Ultimatum}, $c_1=1.00,c_2=-1.00$. In the Fig. \ref{mp}, $c_1=c_2=-1.00$. Numbers in brackets denote losses for each agent, with bolded figures indicating optimal algorithmic results.}
        \end{table}
        Based on these experiments, we conclude that incorporating appropriate preference parameters can enable agents to discover more advantageous outcomes than traditional NE. This finding underscores the potential of preference parameters to enhance cooperation in game-theoretic interactions, leading to more optimal and mutually beneficial solutions, especially in games characterized by complex dynamics and interdependencies.

        \subsection{Learning Preference Parameters \texorpdfstring{$c$}{c}}
        In this section, we apply opponent shaping to learn the preference parameters $c$. The outcomes are depicted in Fig. \ref{tandem-lc}-\ref{C of Stag Hunt-lc}. This approach allows for adaptive tuning of preference parameters in response to the opponent's actions, thereby enhancing the sophistication and responsiveness of strategy selection.  Fig. \ref{tandem-lc}-\ref{C of Stag Hunt-lc} visually illustrate the efficacy of this learning methodology within the examined game contexts.

       \begin{figure}[h!]
    \centering
    
    \begin{subfigure}{0.4950\linewidth}
        \centering
        \includegraphics[width=1.0\linewidth]{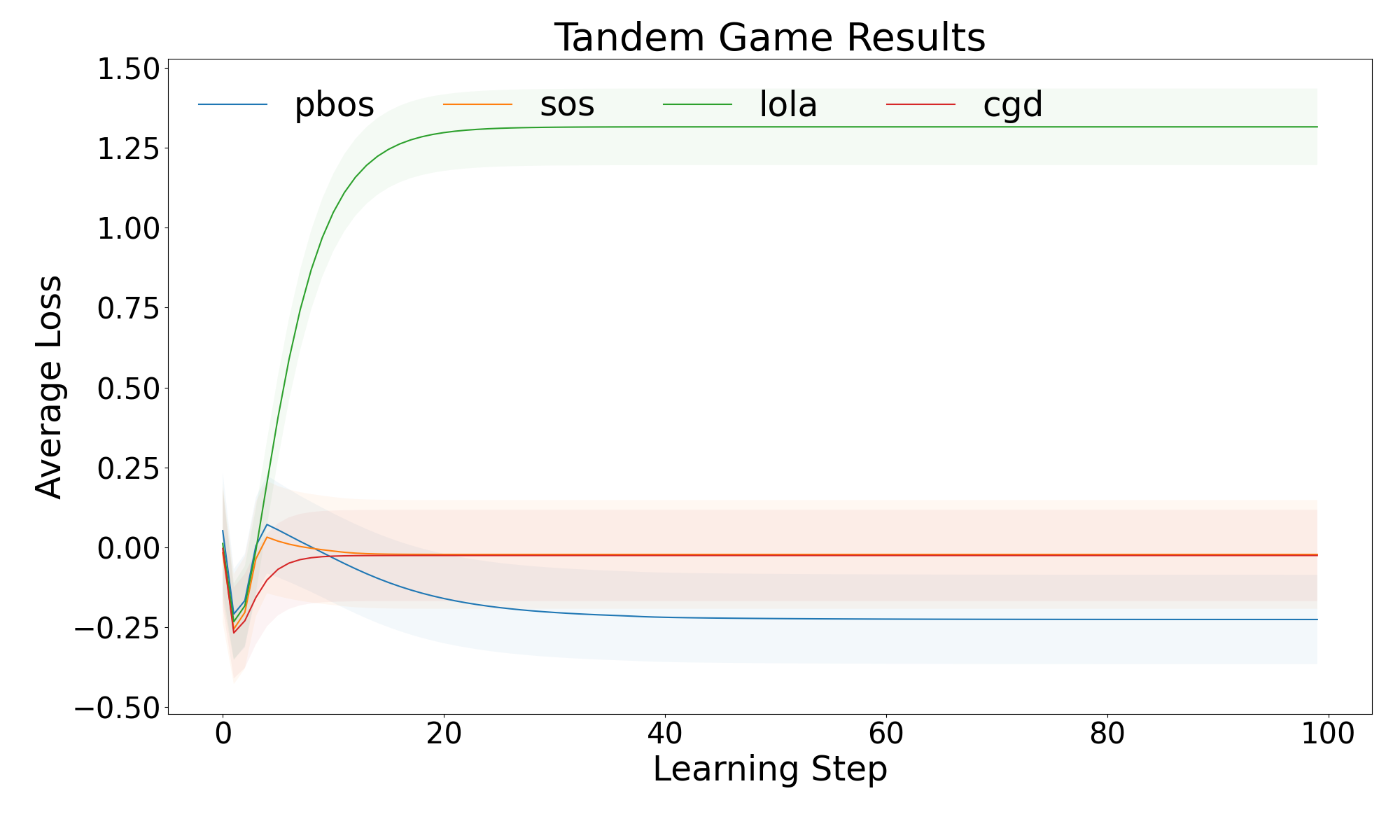}
        \caption{Tandem Game}
        \label{tandem-lc}
    \end{subfigure}
    \hfill
    \begin{subfigure}{0.4950\linewidth}
        \centering
        \includegraphics[width=1.0\linewidth]{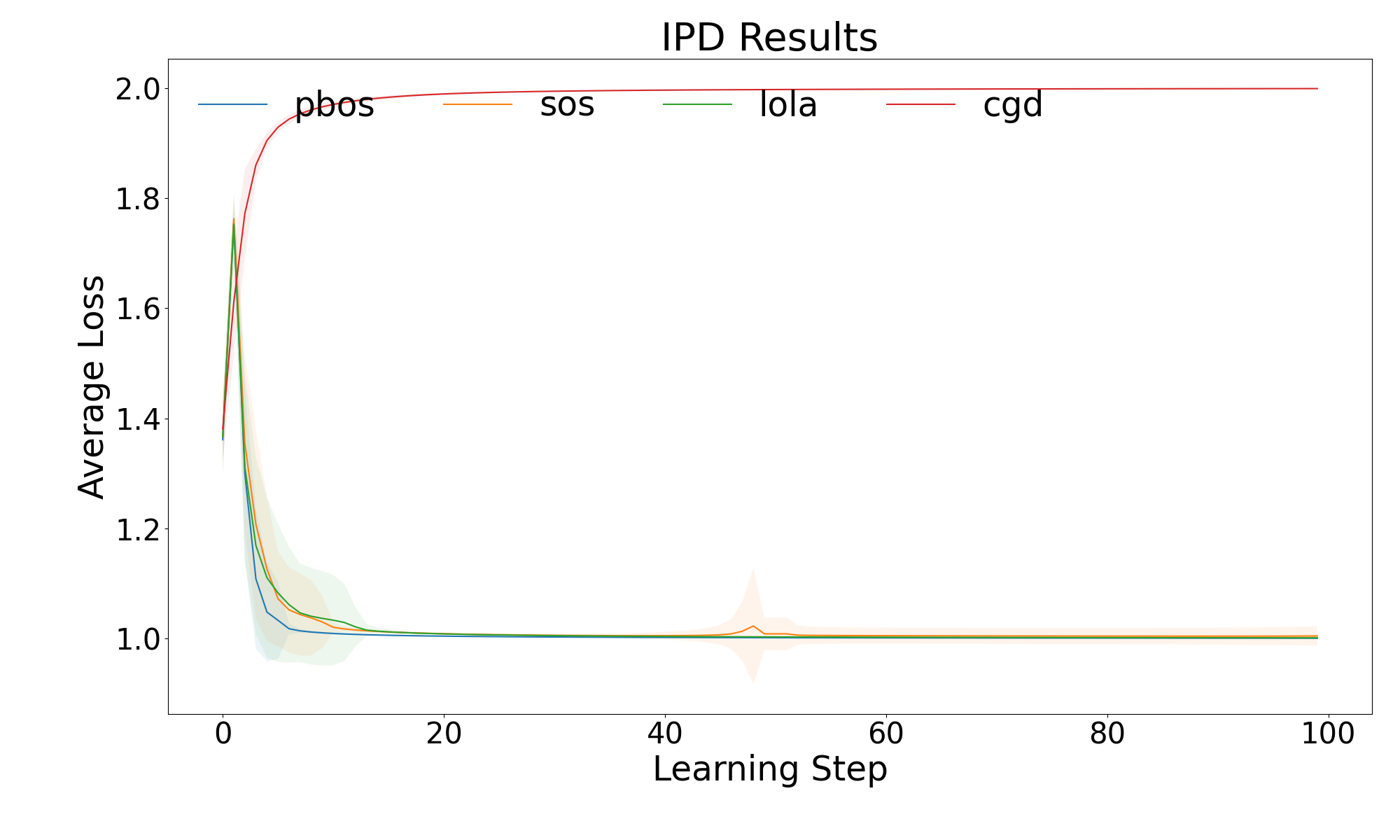}
        \caption{IPD}
        \label{IPD-lc}
    \end{subfigure}

    \begin{subfigure}{0.4950\linewidth}
        \centering
        \includegraphics[width=1.0\linewidth]{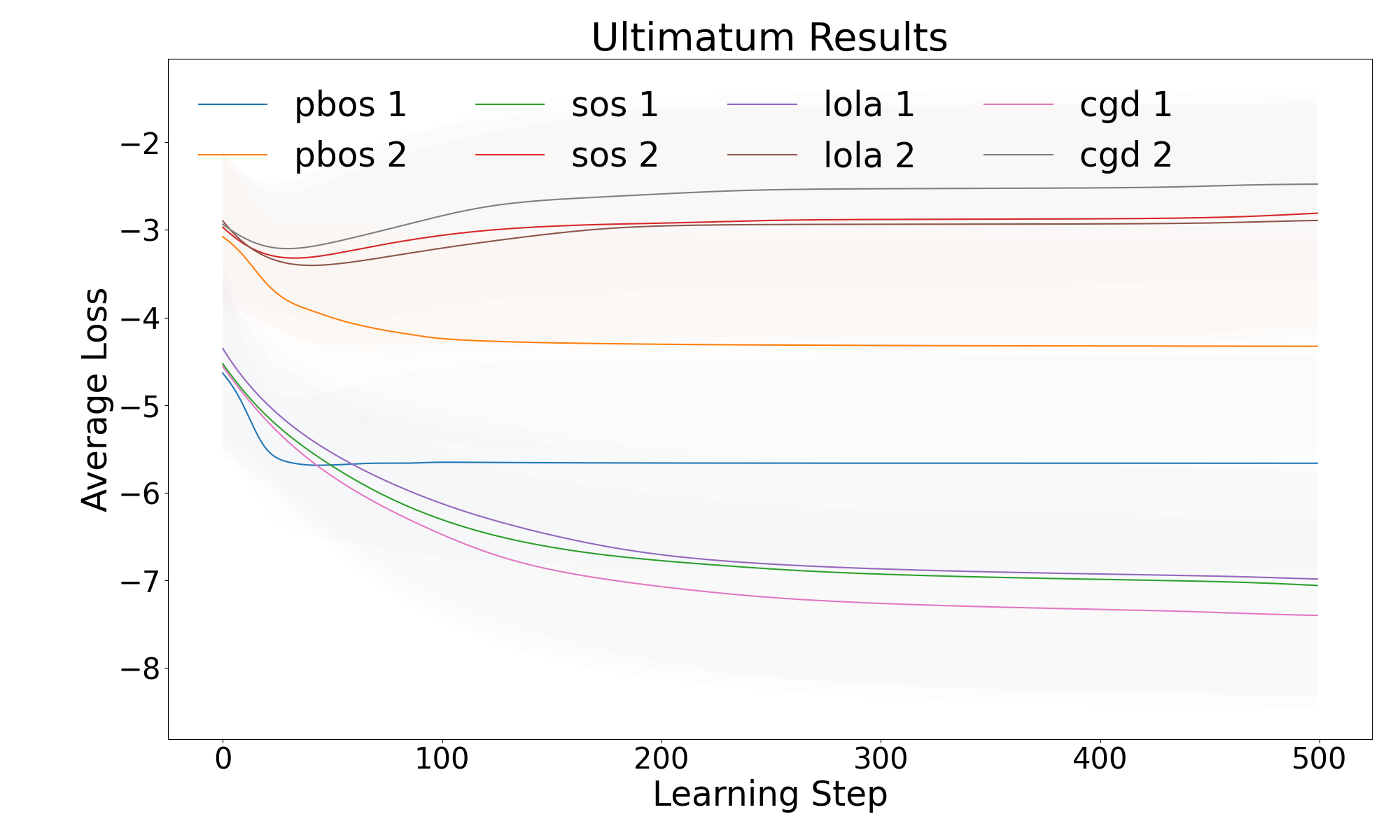}
        \caption{Ultimatum}
        \label{Ultimatum-lc}
    \end{subfigure}
    \hfill
    \begin{subfigure}{0.4950\linewidth}
        \centering
        \includegraphics[width=1.0\linewidth]{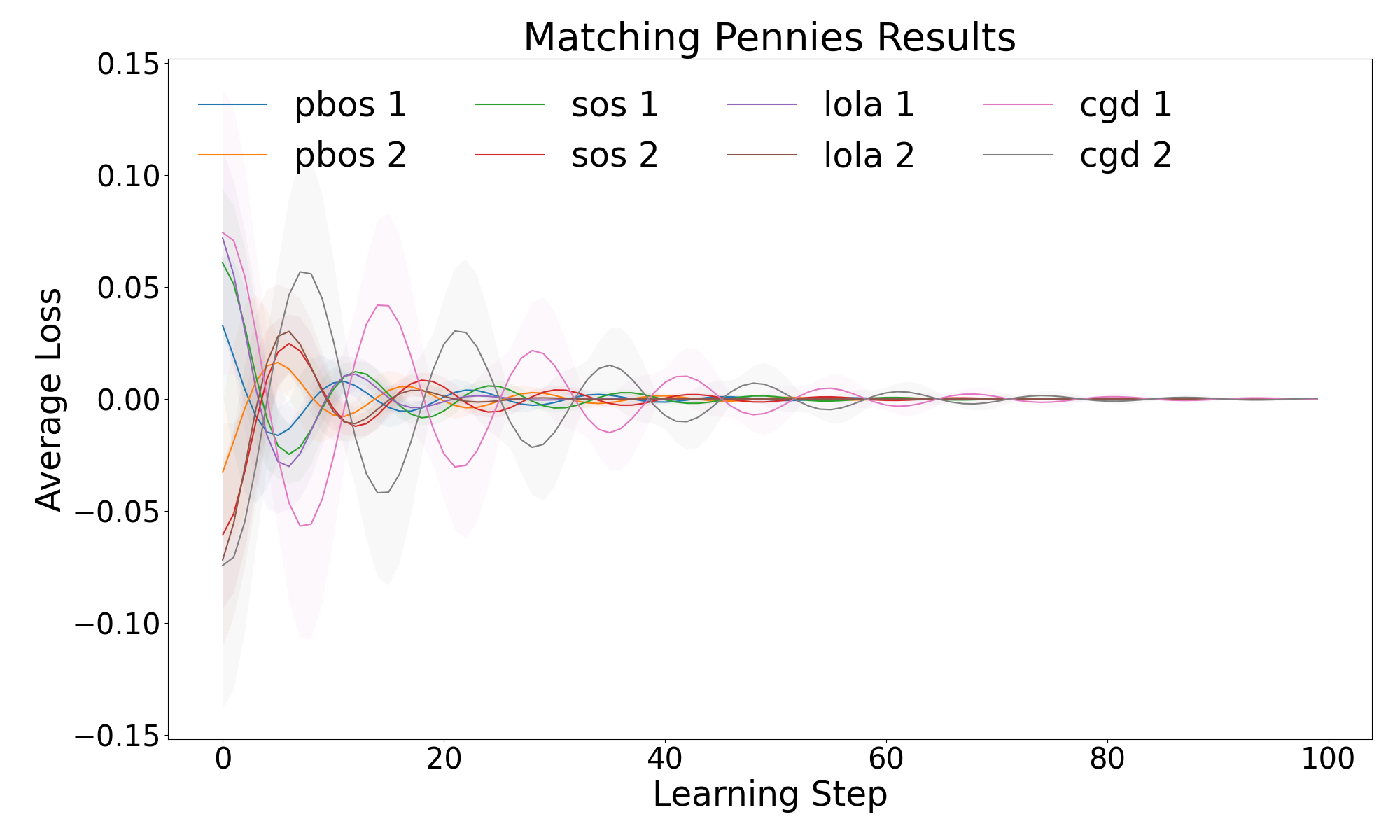}
        \caption{Matching Pennies}
        \label{Matching Pennies-lc}
    \end{subfigure}

    \begin{subfigure}{0.4950\linewidth}
        \centering
        \includegraphics[width=1.0\linewidth]{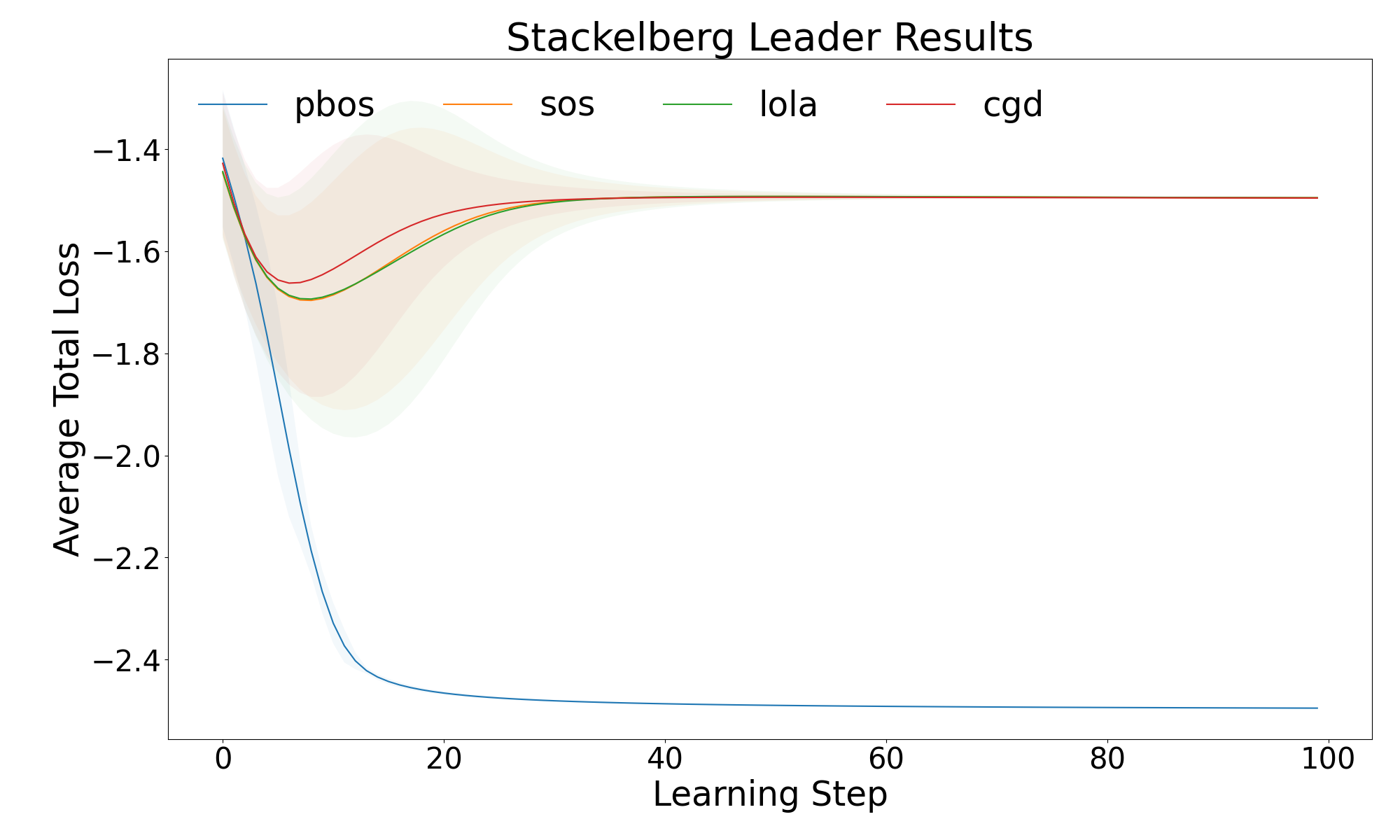}
        \caption{Stackelberg Leader}
        \label{Stackelberg Leader-lc}
    \end{subfigure}
    \hfill
    \begin{subfigure}{0.4950\linewidth}
        \centering
        \includegraphics[width=1.0\linewidth]{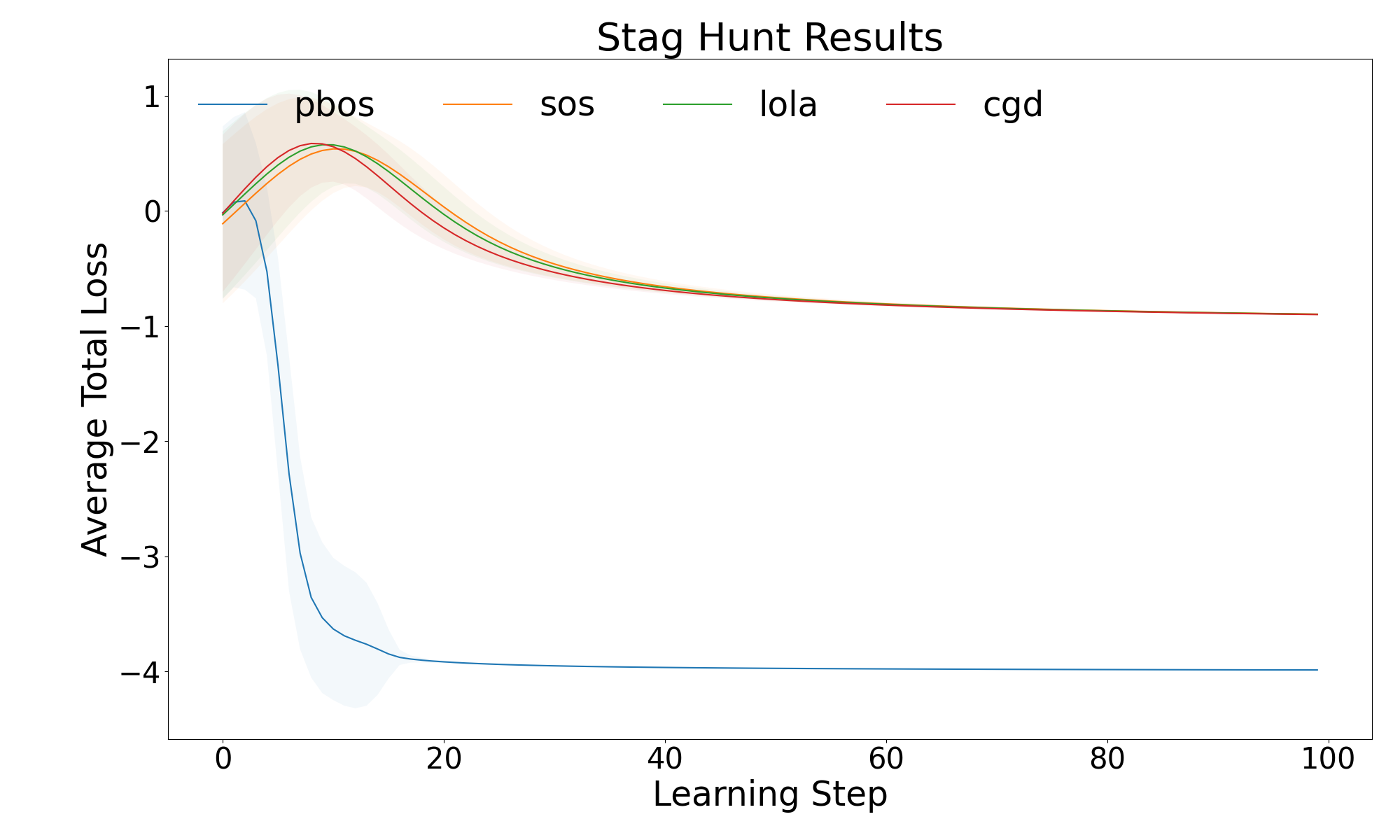}
        \caption{Stag Hunt}
        \label{Stag Hunt-lc}
    \end{subfigure}

    \caption{Results from 100 independent executions of each algorithm. `pbos' represents our proposed algorithm, whereas `sos', `lola', and `cgd' are the baseline algorithms. Fig. \ref{tandem-lc} and \ref{IPD-lc} show the loss of agent 1, while Fig. \ref{Ultimatum-lc} and \ref{Matching Pennies-lc} display the losses of both agents. Fig. \ref{Stackelberg Leader-lc} and \ref{Stag Hunt-lc} present the average losses of the two agents. For Fig. \ref{tandem-lc}, \ref{IPD-lc}, \ref{Stackelberg Leader-lc}, and \ref{Stag Hunt-lc}, lower loss values reflect better performance. For Fig. \ref{Ultimatum-lc} and \ref{Matching Pennies-lc}, the closer the losses of the two agents are, the better the performance. Across all scenarios, the proposed algorithm demonstrates superior performance compared to the baselines.}
\end{figure}
\begin{figure}[htbp]
    \centering

    \begin{subfigure}{0.4950\linewidth}
        \centering
        \includegraphics[width=1.0\linewidth]{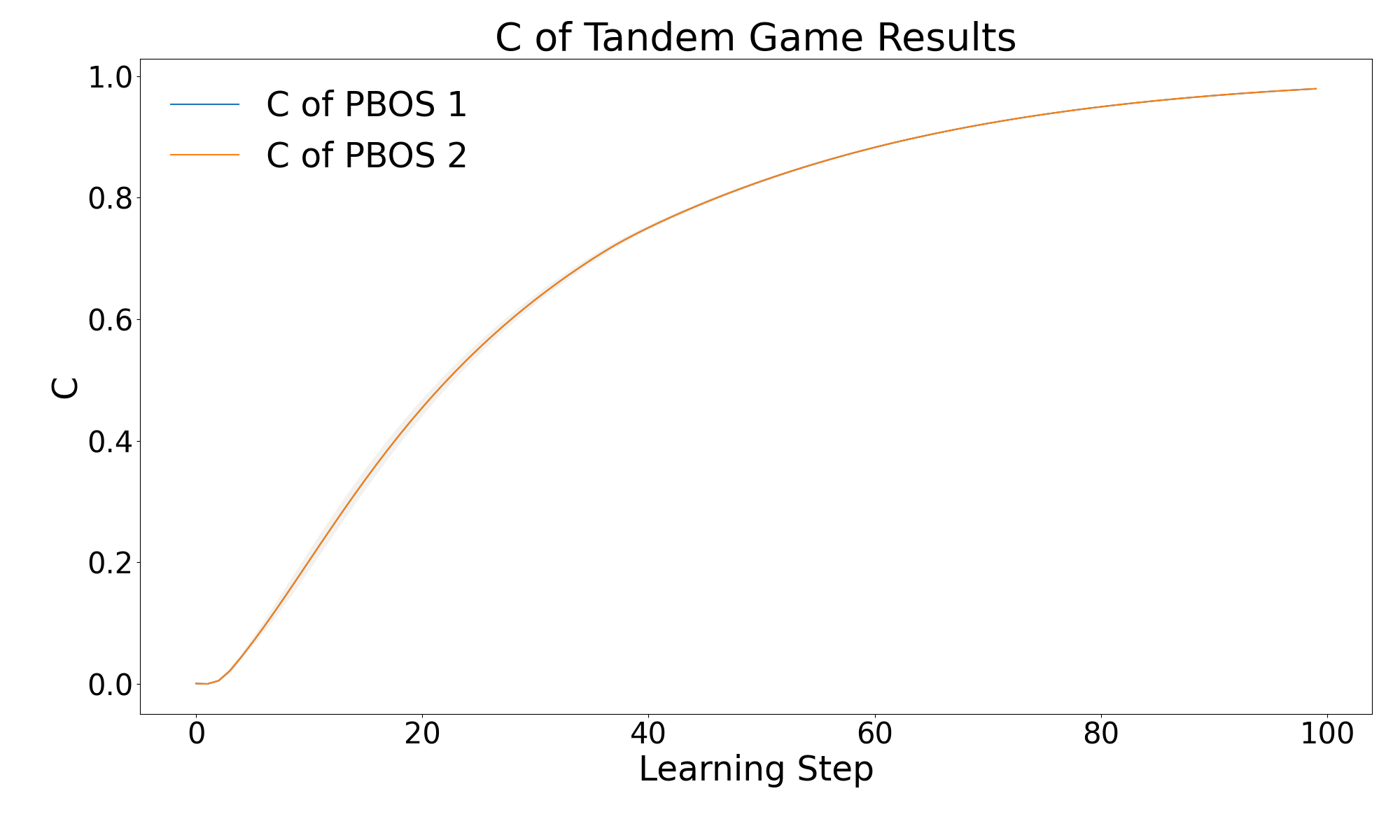}
        \caption{C of Tandem Game}
        \label{C of Tandem-lc}
    \end{subfigure}
    \hfill
    \begin{subfigure}{0.4950\linewidth}
        \centering
        \includegraphics[width=1.0\linewidth]{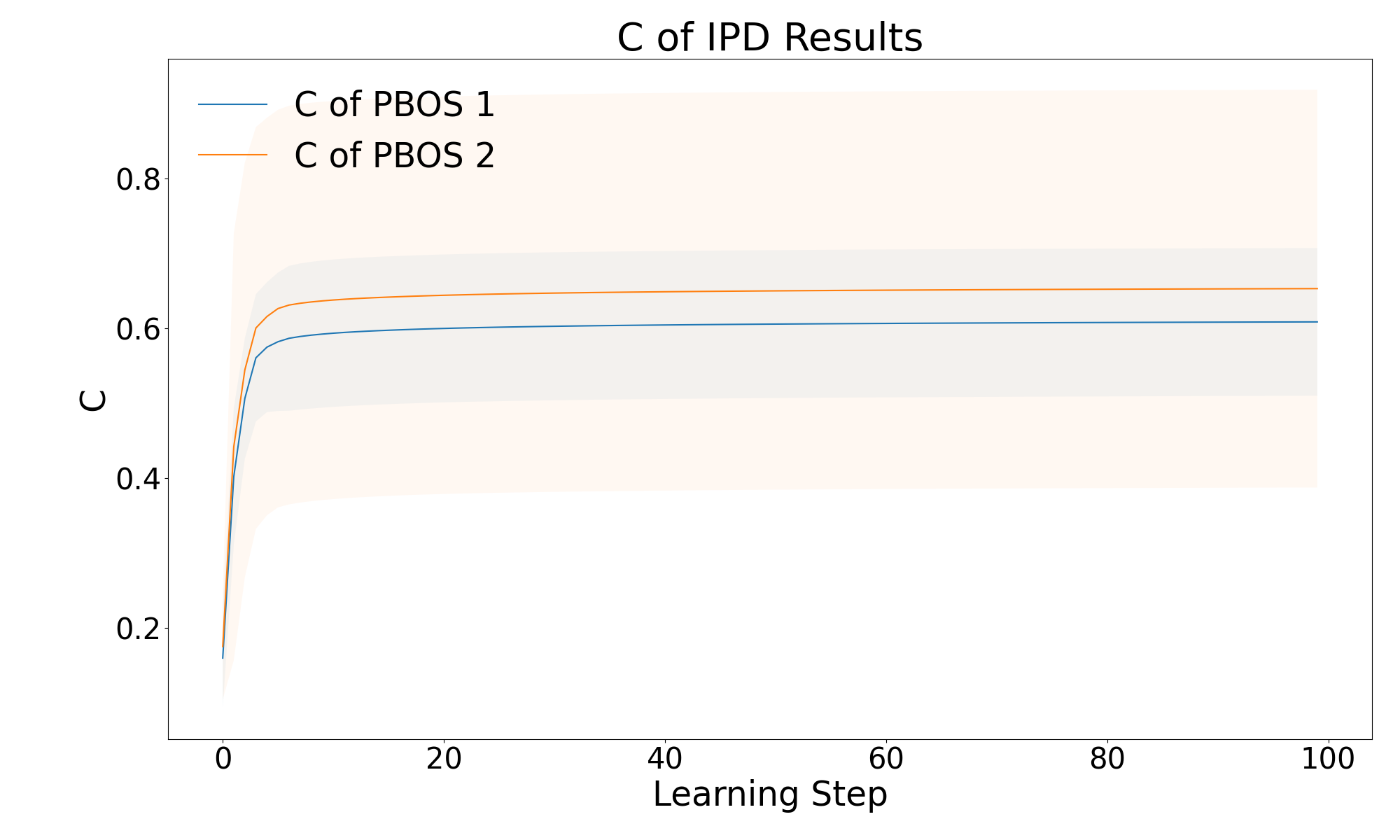}
        \caption{C of IPD}
        \label{C of IPD-lc}
    \end{subfigure}

    \begin{subfigure}{0.4950\linewidth}
        \centering
        \includegraphics[width=1.0\linewidth]{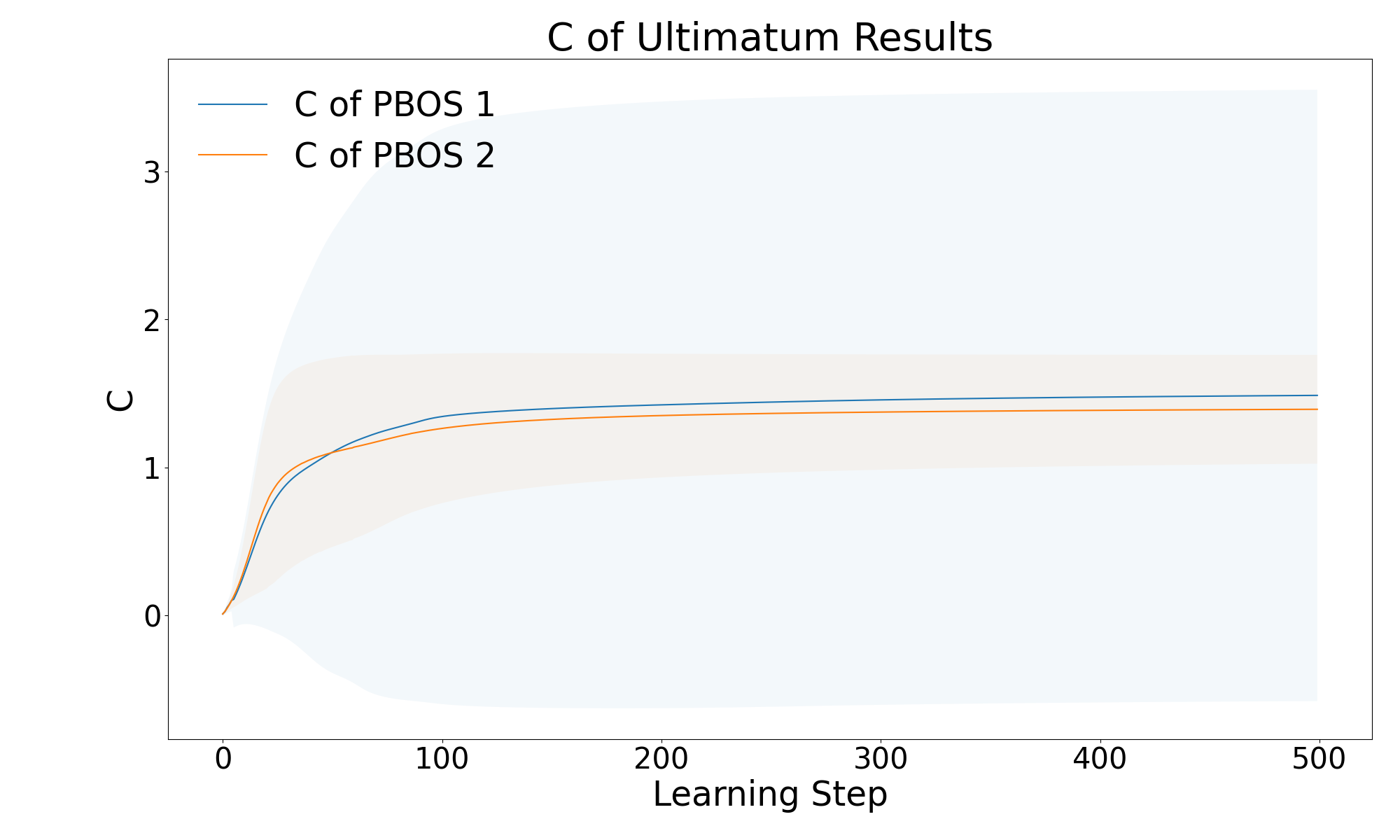}
        \caption{C of Ultimatum}
        \label{C of Ultimatum-lc}
    \end{subfigure}
    \hfill
    \begin{subfigure}{0.4950\linewidth}
        \centering
        \includegraphics[width=1.0\linewidth]{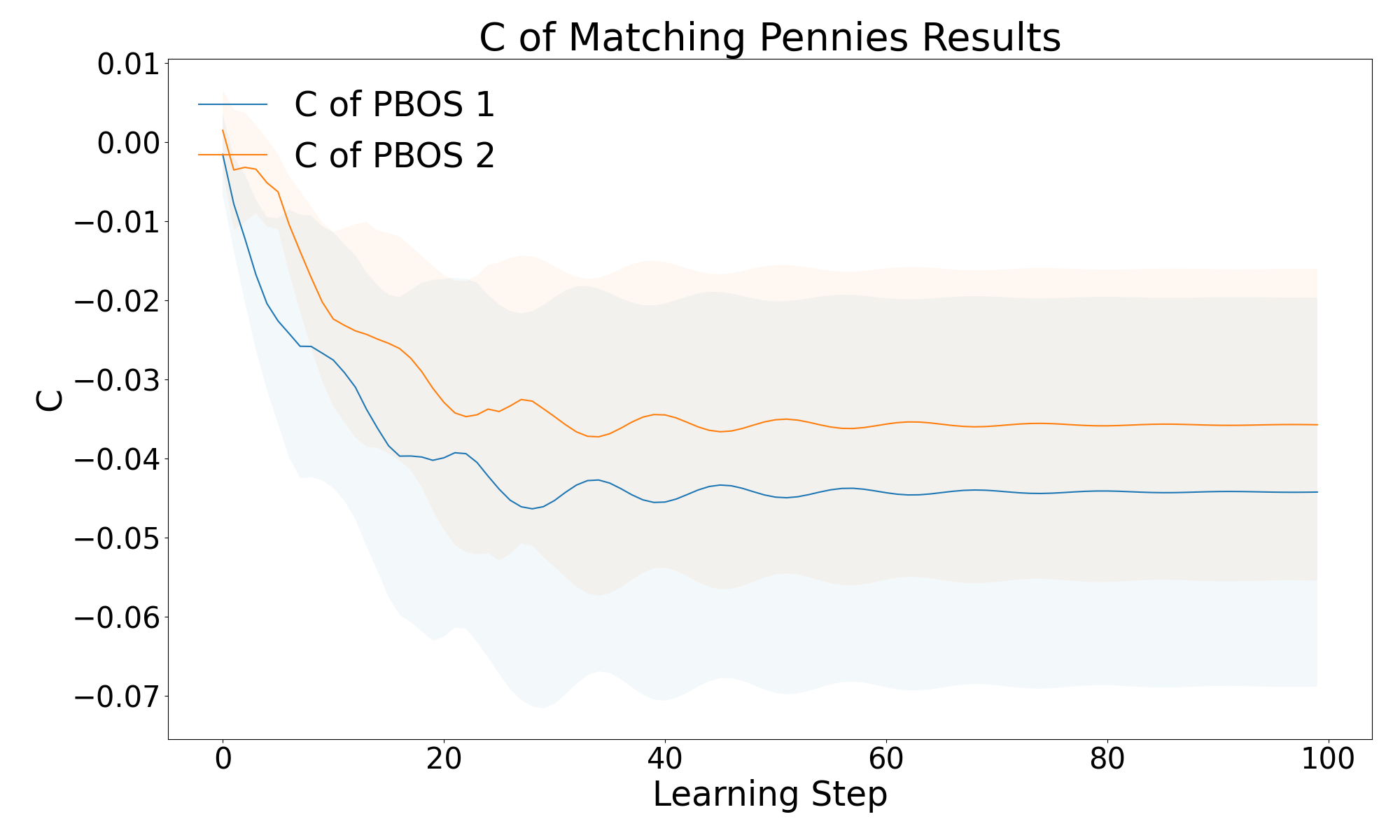}
        \caption{C of Matching Pennies}
        \label{C of Matching Pennies-lc}
    \end{subfigure}

    \begin{subfigure}{0.4950\linewidth}
        \centering
        \includegraphics[width=1.0\linewidth]{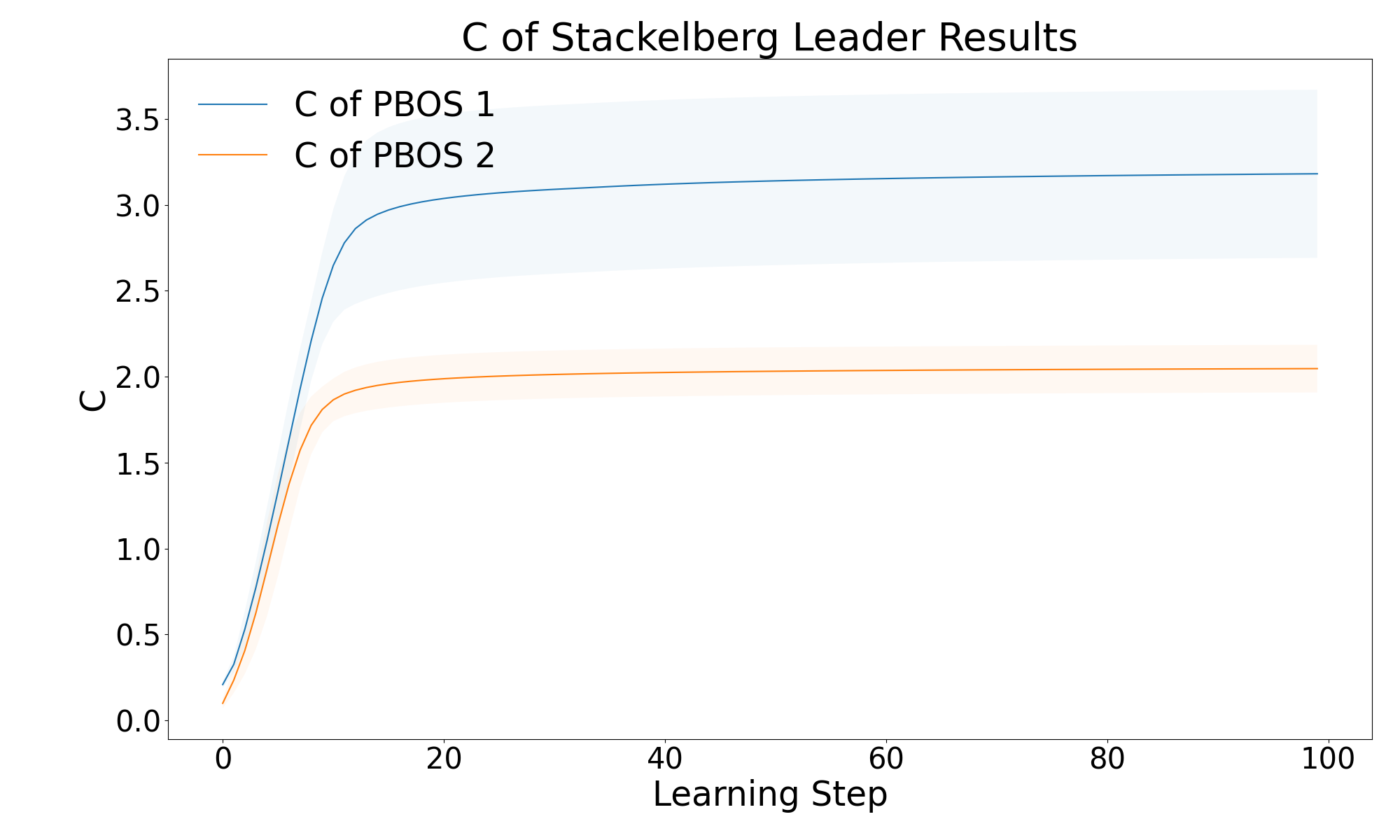}
        \caption{C of Stackelberg Leader}
        \label{C of Stackelberg Leader-lc}
    \end{subfigure}
    \hfill
    \begin{subfigure}{0.4950\linewidth}
        \centering
        \includegraphics[width=1.0\linewidth]{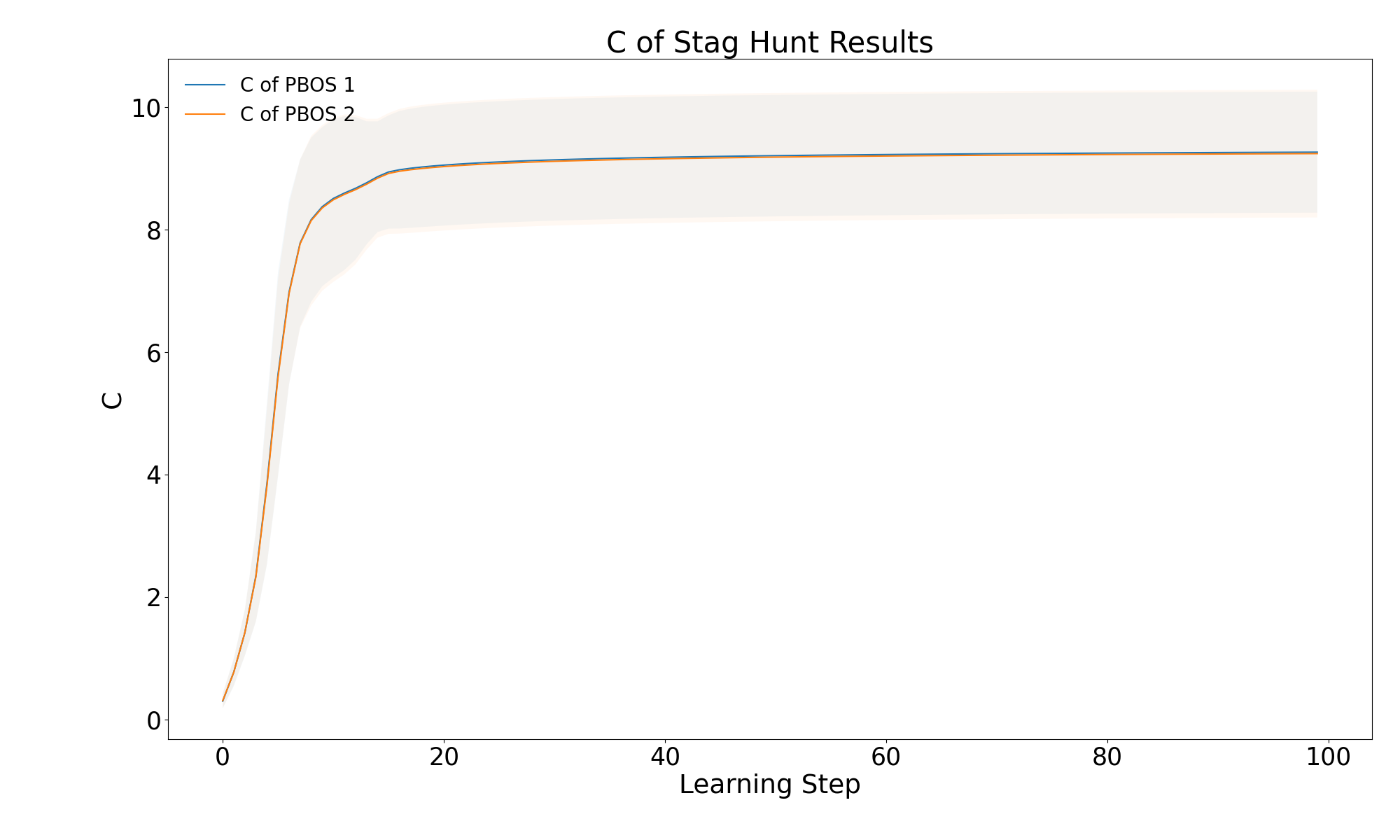}
        \caption{C of Stag Hunt}
        \label{C of Stag Hunt-lc}
    \end{subfigure}

    \caption{Fig. \ref{C of Tandem-lc}, \ref{C of IPD-lc}, \ref{C of Ultimatum-lc}, \ref{C of Matching Pennies-lc}, \ref{C of Stackelberg Leader-lc} and \ref{C of Stag Hunt-lc} illustrate the preference parameters our algorithm learned across diverse game environments across 100 runs. Higher $c$ indicates more cooperation and a lower $c$ indicates more competitiveness.}
\end{figure}

        The outcomes are summarized as follows:
        
        (a). Tandem Game: We observed that the preference parameters converge to $c_1\approx1.00$ and $c_2\approx1.00$, resulting in a loss of approximately $-0.25$ for both agents, corresponding to the optimal outcome of the game (Fig. \ref{tandem-lc}, \ref{C of Tandem-lc}).
        
        (b). IPD: Learned preference parameters are $c_1\approx0.61$ and $c_2\approx0.65$, with a loss of $1.00$ for both agents, indicative of the best NE in IPD (Fig. \ref{IPD-lc}, \ref{C of IPD-lc}).
        
        (c). Ultimatum Game: Preference parameters $c_1\approx1.49$ and $c_2\approx1.40$ resulted in losses of $-5.66$ and $-4.33$ for agents respectively, indicating near attainment of a fair division of the ten-dollar pool (Fig. \ref{Ultimatum-lc}, \ref{C of Ultimatum-lc}).
        
        (d). Matching Pennies: Parameters $c_1\approx-0.04$ and $c_2\approx-0.04$ yielded near-zero losses for both agents, consistent with the mixed strategy NE characteristic of a zero-sum game (Fig. \ref{Matching Pennies-lc}, \ref{C of Matching Pennies-lc}).
        
        (e). Stackelberg Leader Game: Convergence to $c_1\approx3.18$ and $c_2\approx2.05$, with the average total loss being approximately $-2.50$, represents the optimal outcome for both agents in this game (Fig. \ref{Stackelberg Leader-lc}, \ref{C of Stackelberg Leader-lc}).
        
        (f). Stag Hunt: Learned preference parameters $c_1\approx9.26$ and $c_2\approx9.24$ resulted in an average total loss of approximately $-3.99$, indicating the best possible outcome for the game (Fig. \ref{Stag Hunt-lc}, \ref{C of Stag Hunt-lc}).

        The detailed numerical results of these experiments are presented in Table \ref{table-lc}, providing a comprehensive overview of how opponent shaping influences the learning of preference parameters and subsequently affects game outcomes.

        \begin{table}[h]
            \caption{Results of PBOS and three baseline algorithms across six different games.}\label{table-lc}
            \setlength\tabcolsep{1pt}
            \begin{tabular*}{\textwidth}{@{\extracolsep\fill}ccccccc}
            \toprule%
            & Tandem & IPD & Ultimatum & \shortstack{Matching\\Pennies} & \shortstack{Stackelberg\\Leader} & Stag Hunt  \\
            \midrule
            \multirow{2}{*}{PBOS}& (1.00,1.00)\footnotemark[1] & (0.61,0.65)\footnotemark[1] & (1.49,1.40)\footnotemark[1] & (-0.04,-0.04)\footnotemark[1] & (3.18,2.05)\footnotemark[1] & (9.26,9.24)\footnotemark[1]  \\ 
           & \textbf{(-0.24,-0.26)} & \textbf{(1.00,1.00)} & \textbf{(-5.66,-4.33)} & \textbf{(0.00,0.00)} & \textbf{(-3.00, -2.00)} & \textbf{(-3.99,-3.99)}  \\
           \midrule
            LOLA & (1.31,1.31) & \textbf{(1.00,1.00)} & (-6.98,-2.89) & \textbf{(0.00,0.00)} & (-2.02,-0.97) & (-0.90,-0.90)  \\ 
           SOS & (-0.02,0.02) & \textbf{(1.00,1.00)} & (-7.01,-2.81) & \textbf{(0.00,0.00)} & (-2.02,-0.97) & (-0.90,-0.90)\\
            CGD & (-0.03,0.03) & (2.00,2.00) & (-7.40,-2.48) & \textbf{(0.00,0.00)} & (-2.02,-0.98) & (-0.90,-0.90)\\
            \botrule
            \end{tabular*}
            \footnotetext{Note: Numbers in brackets indicate losses incurred by the respective agents, with bolded figures denoting optimal outcomes achieved by the four algorithms.}
            \footnotetext[1]{These numbers are the learned preference parameters $(c_1,c_2)$ by the two agents using PBOS.}
        \end{table}

        From Table \ref{table-lc}, it is evident that except for the Matching Pennies game, which is zero-sum, the PBOS algorithm consistently converges to positive preference parameters $c_i>0$ for $i=1,2$. This signifies that PBOS facilitates cooperative strategy learning in each non-zero-sum game mentioned. The inclination towards cooperation underpins the agents' ability to adopt strategies that yield higher joint rewards than those achievable under traditional NE. Notably, in the Tandem Game, the agents learn such that $c_1c_2=1$, achieving a state of \emph{Maximization of Cooperation} as per Definition \ref{max cooperation}. 

        Fig. \ref{tandem-d}-\ref{Stag Hunt-g} provide a comprehensive visualization of strategic dynamics across five distinct games using the PBOS algorithm. Directional and field diagrams depict the evolutionary trajectory of strategies adopted by both agents throughout the gaming process.

        To assess PBOS's generalizability, extensive experiments were conducted with randomly generated bimatrix games. We created 1000, 2000, and 5000 games, where each entry in the bimatrix was an integer randomly selected from the interval $[-7,7]$. The objective was to optimize the average joint reward of the two agents. The results, summarized in Fig. \ref{1000 Stochastic Game-lc}-\ref{5000 Stochastic Game-lc}, demonstrate PBOS's effectiveness in optimizing joint rewards across these stochastic games. This figure underscores PBOS's robustness and adaptability in diverse stochastic game environments, showcasing its capacity to converge towards strategies that maximize collective payoff across varied game configurations.

\begin{figure}[htbp]
    \centering

    \begin{subfigure}{0.495\linewidth}
        \centering
        \includegraphics[width=1.0\linewidth]{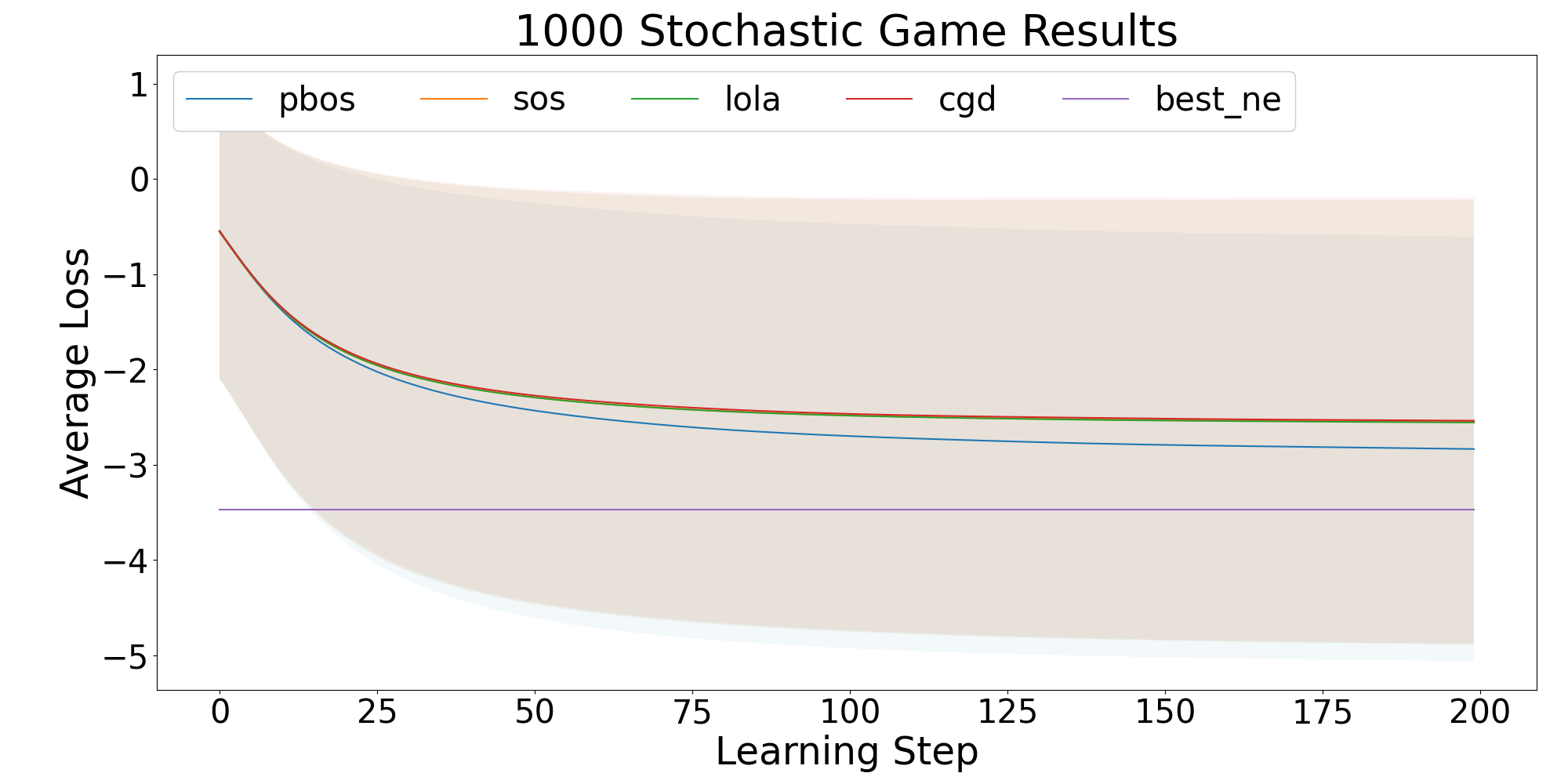}
        \caption{1000 Stochastic Game Results}
        \label{1000 Stochastic Game-lc}
    \end{subfigure}
    \hfill
    \begin{subfigure}{0.495\linewidth}
        \centering
        \includegraphics[width=1.0\linewidth]{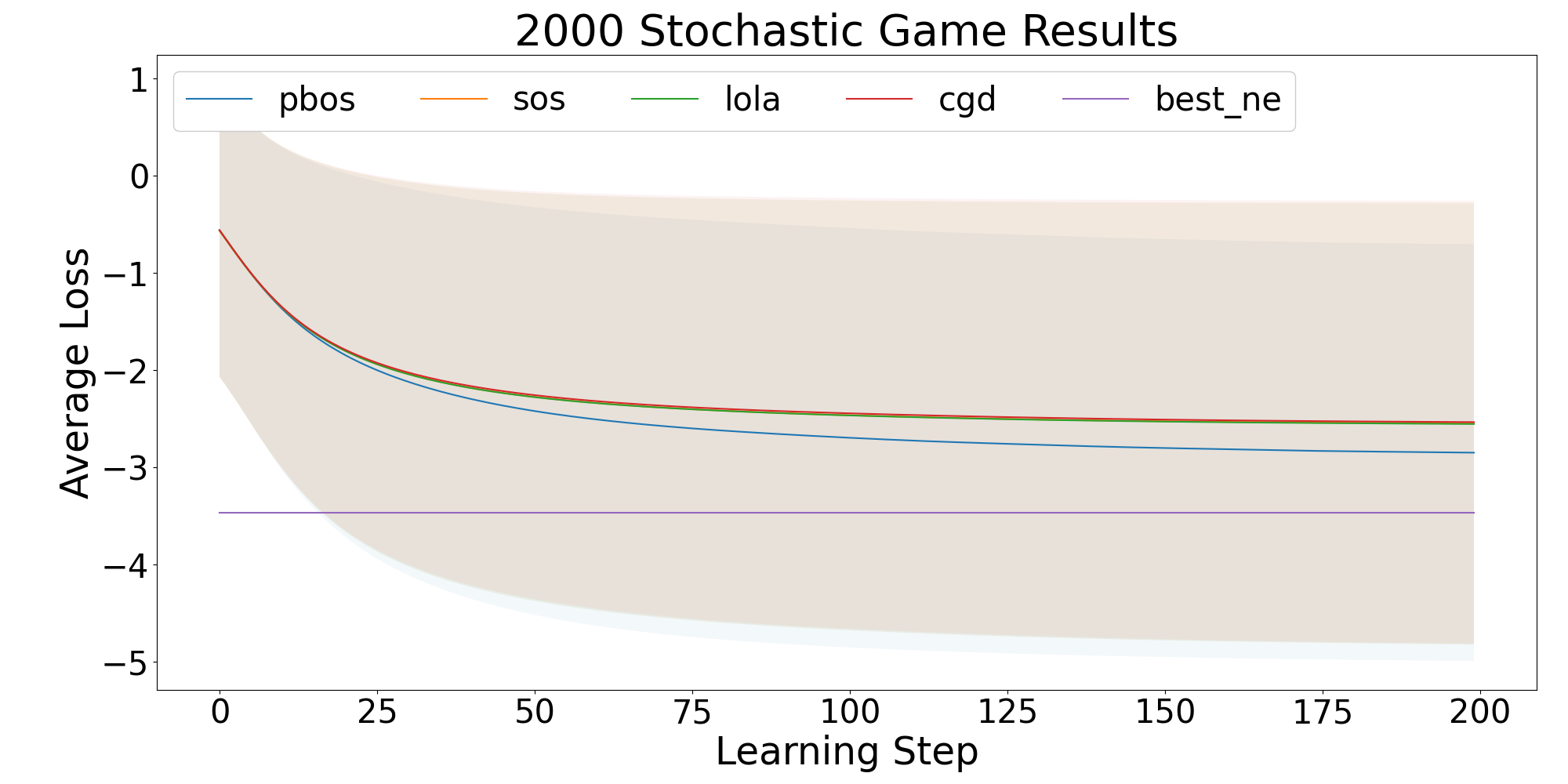}
        \caption{2000 Stochastic Game Results}
        \label{2000 Stochastic Game-lc}
    \end{subfigure}

    \begin{subfigure}{0.495\linewidth}
        \centering
        \includegraphics[width=1.0\linewidth]{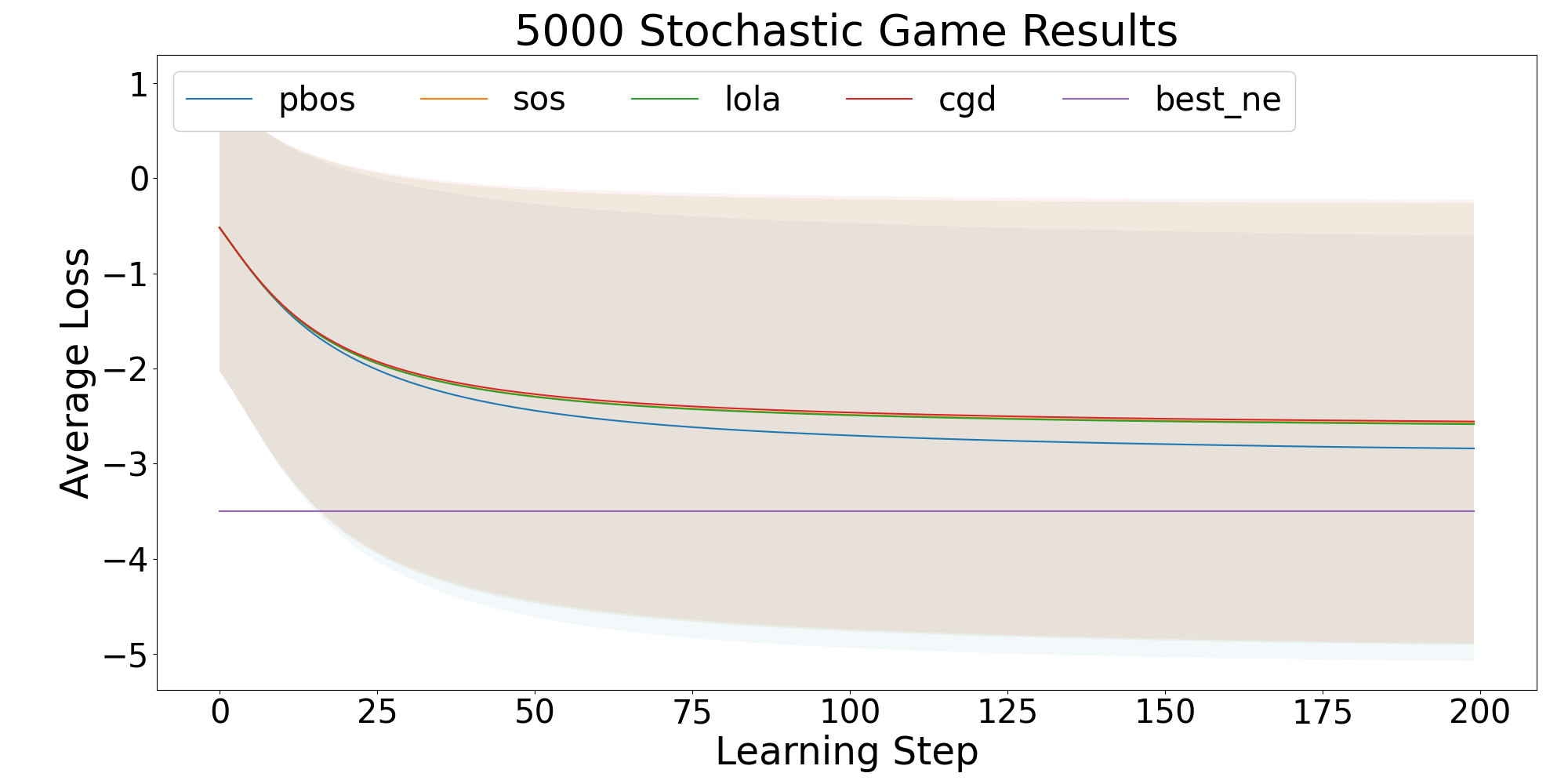}
        \caption{5000 Stochastic Game Results}
        \label{5000 Stochastic Game-lc}
    \end{subfigure}

    \caption{Fig. \ref{1000 Stochastic Game-lc}, \ref{2000 Stochastic Game-lc} and \ref{5000 Stochastic Game-lc} present results from 1000,2000, and 5000 stochastic game experiments. ‘pbos’ is our algorithm; ‘sos’, ‘lola’, and ‘cgd’ are baselines; ‘best\_ne’ denotes the minimal cumulative average loss. Our algorithm outperforms the baselines but hasn't reached the optimal.}
\end{figure}

        As indicated in Table \ref{table-ss}, none of the four algorithms consistently achieves the optimal outcome across multiple stochastic games. However, when compared to the other three algorithms, PBOS demonstrates a significant improvement of approximately 10\% in approximating the optimal value. This comparative analysis underscores PBOS's superior performance in converging towards near-optimal solutions across a diverse range of game scenarios.

        \begin{table}[h]
            \caption{Average total losses of PBOS and three baseline algorithms in randomly generated bimatrix games of 1000, 2000, and 5000}.\label{table-ss}
            \setlength{\tabcolsep}{3pt}
            \begin{tabular*}{\textwidth}{@{\extracolsep\fill}ccccccc}
            \toprule%
            & BEST\_NE\footnotemark[1] & PBOS & SOS & LOLA & CGD & Proximity Improvement  \\ 
            \midrule
           1000 & -3.47 & \textbf{-2.84}$\pm$ 0.22 & -2.55$\pm$ 0.23 & -2.57 $\pm$ 0.23 & -2.54$\pm$ 0.23 & 10.05\% \\
           2000 & -3.47 & \textbf{-2.85}$\pm$ 0.21 & -2.55$\pm$ 0.23 & -2.55 $\pm$ 0.23& -2.53$\pm$ 0.23 & 11.76\% \\
           5000 & -3.50 & \textbf{-2.84}$\pm$ 0.22 & -2.58$\pm$ 0.23 & -2.59 $\pm$ 0.23 & -2.57$\pm$ 0.23 & 9.65\% \\
            \botrule
            \end{tabular*}
            \footnotetext{Note: Bolded numbers indicate optimal results achieved by the four algorithms.}
            \footnotetext[1]{BEST\_NE denotes the average of the sum of the minimum losses of the two agents in each game.}
        \end{table}

        \subsection{Adversarial Testing}
        In the context of four game environments with symmetrical loss functions, PBOS and three baseline algorithms have participated in gameplay. The comparative results of these interactions are depicted in Fig. \ref{Tandem Competitive-lc}-\ref{C of Stag Hunt Competitive-lc}.

        \begin{figure}[htbp]
    \centering

    \begin{subfigure}{0.435\textwidth}
        \centering
        \includegraphics[width=\textwidth]{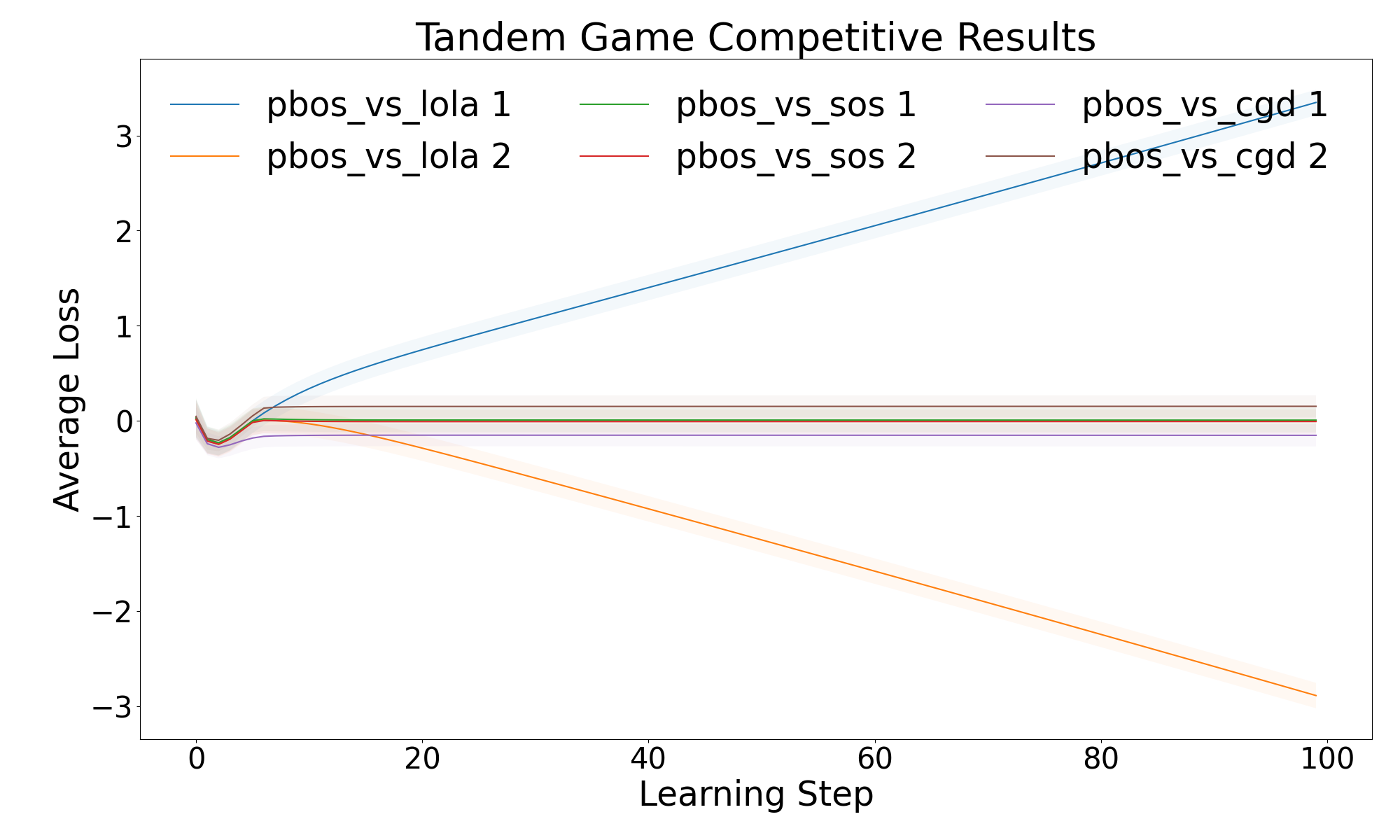}
        \caption{Tandem Game}
        \label{Tandem Competitive-lc}
    \end{subfigure}
    \hfill
    \begin{subfigure}{0.435\textwidth}
        \centering
        \includegraphics[width=\textwidth]{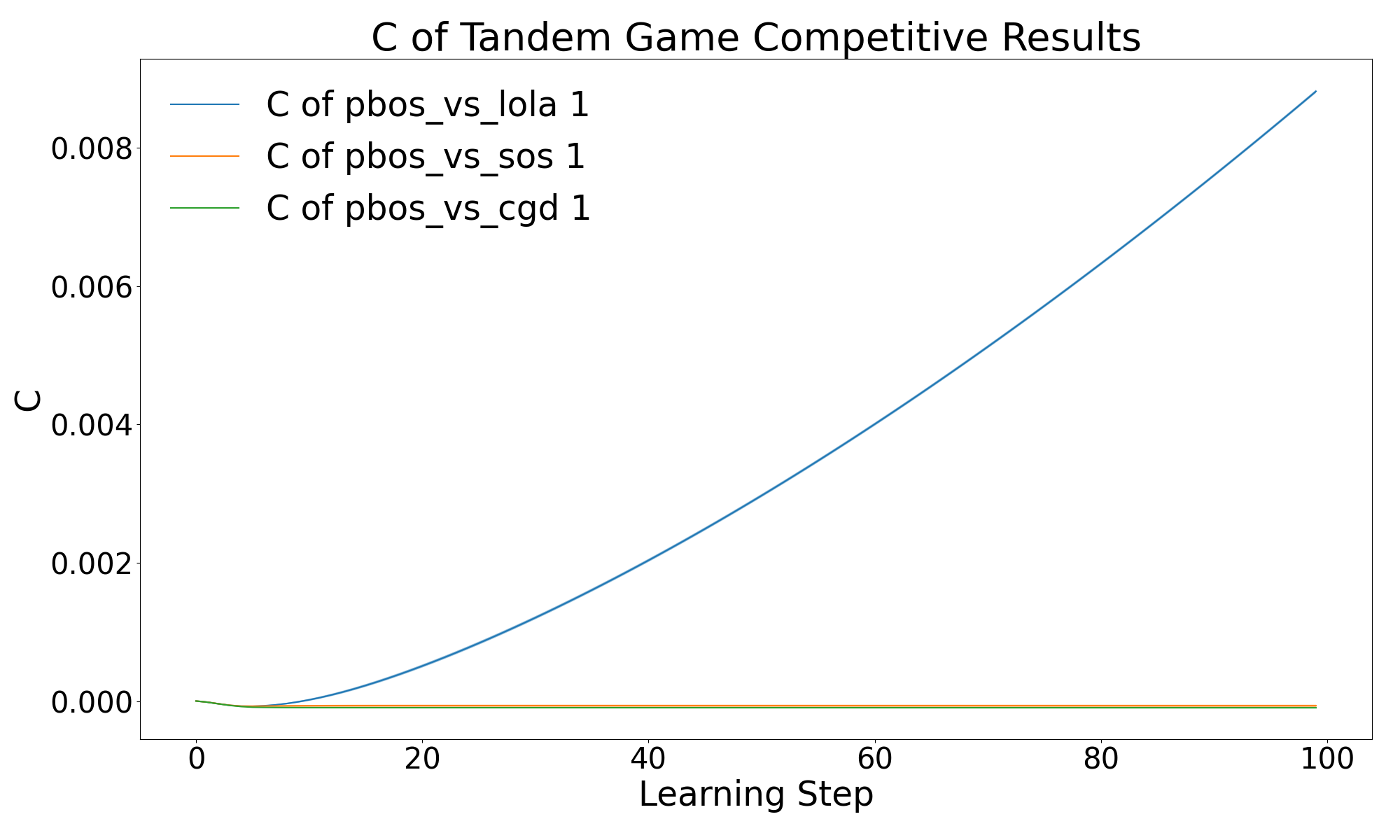}
        \caption{C of Tandem Competitive}
        \label{C of Tandem Competitive-lc}
    \end{subfigure}

    \begin{subfigure}{0.435\textwidth}
        \centering
        \includegraphics[width=\textwidth]{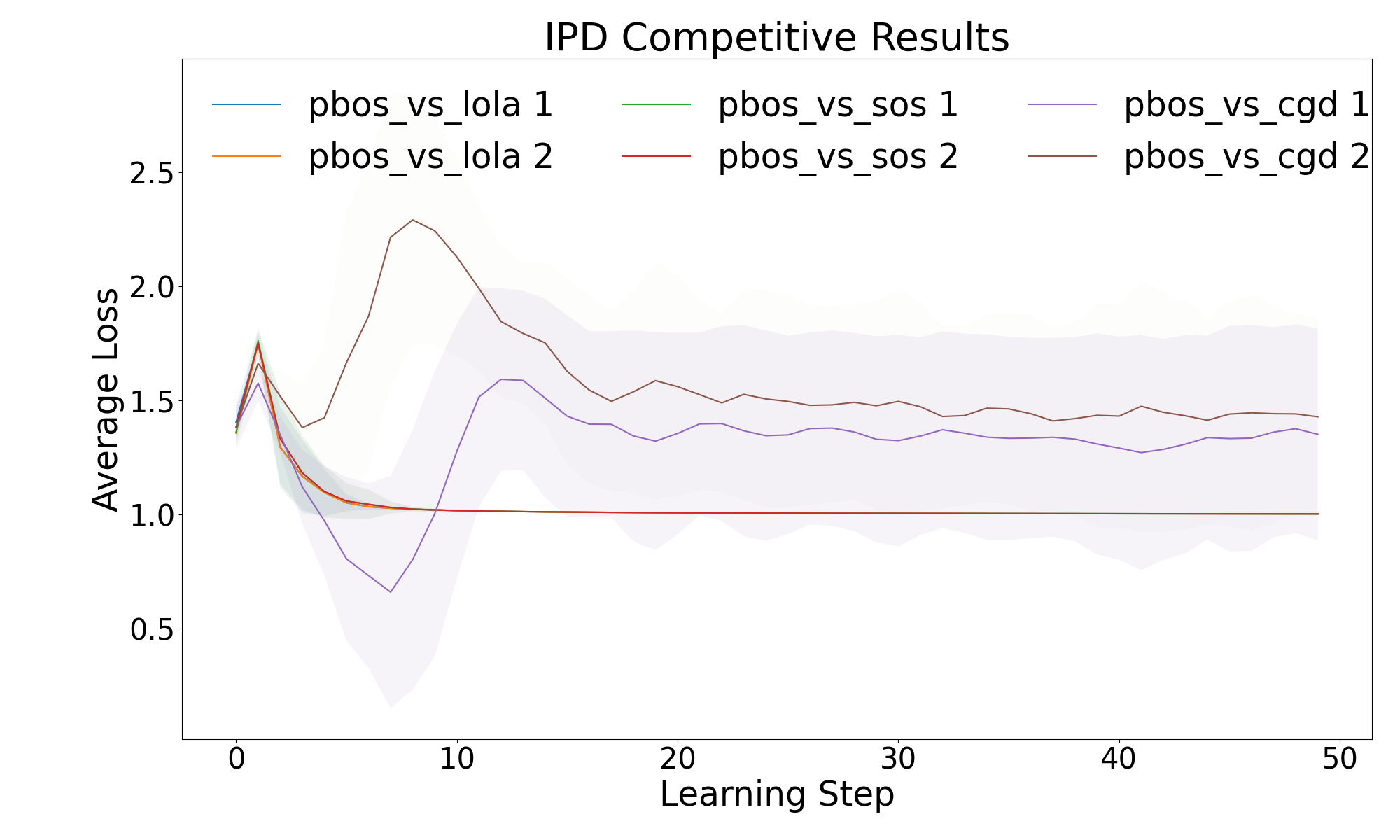}
        \caption{IPD}
        \label{IPD Competitive-lc}
    \end{subfigure}
    \hfill
    \begin{subfigure}{0.435\textwidth}
        \centering
        \includegraphics[width=\textwidth]{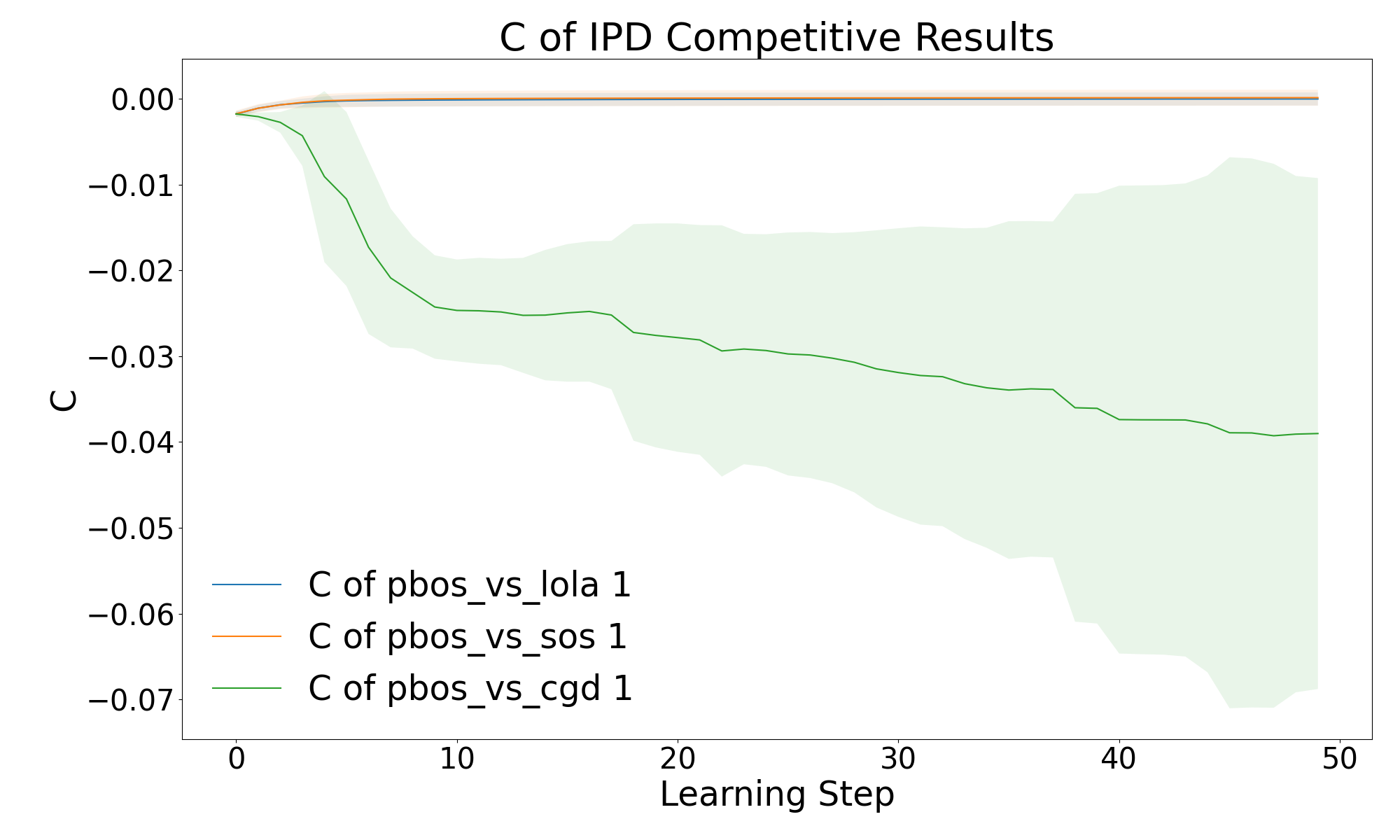}
        \caption{C of IPD}
        \label{C of IPD Competitive-lc}
    \end{subfigure}

    \begin{subfigure}{0.435\textwidth}
        \centering
        \includegraphics[width=\textwidth]{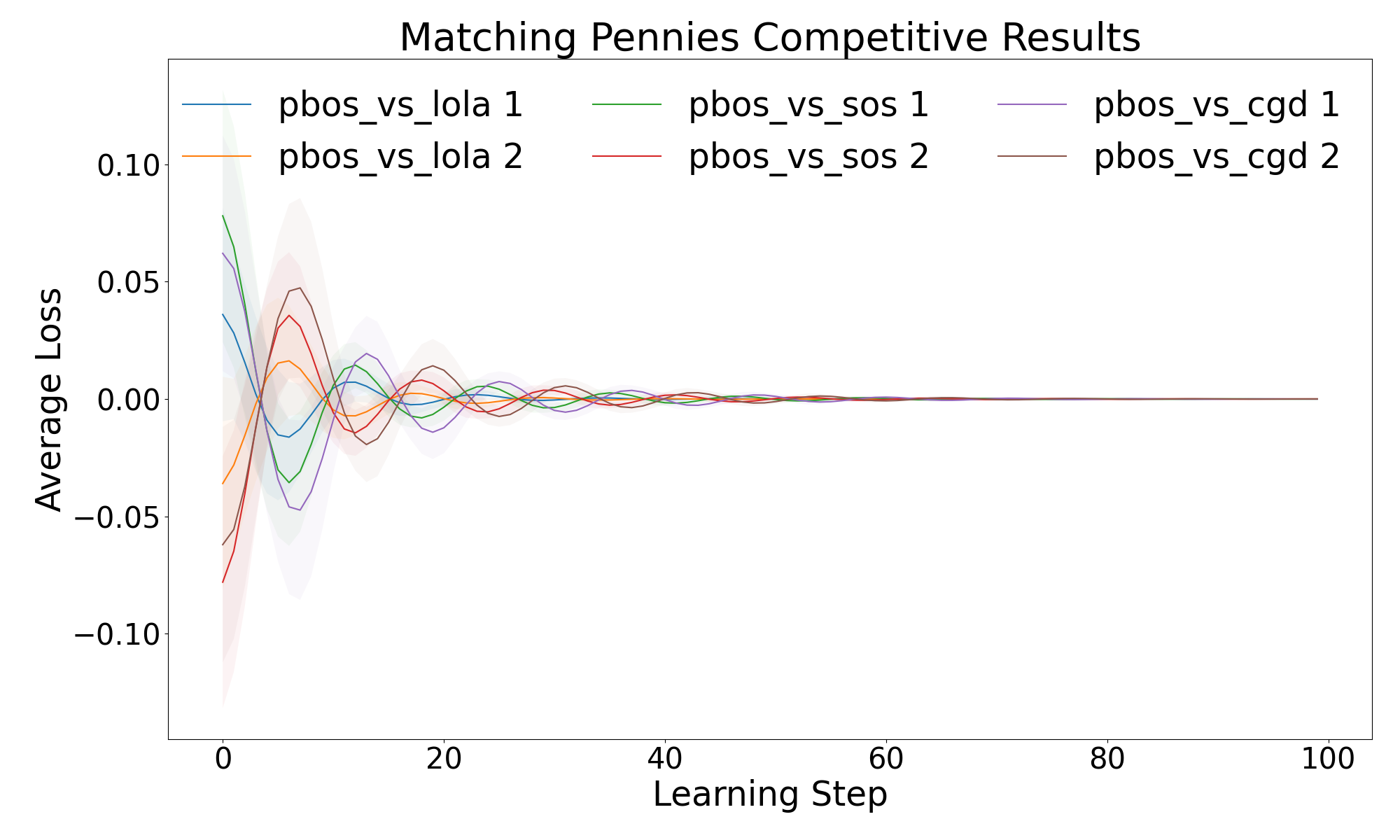}
        \caption{Matching Pennies}
        \label{Matching Pennies Competitive-lc}
    \end{subfigure}
    \hfill
    \begin{subfigure}{0.435\textwidth}
        \centering
        \includegraphics[width=\textwidth]{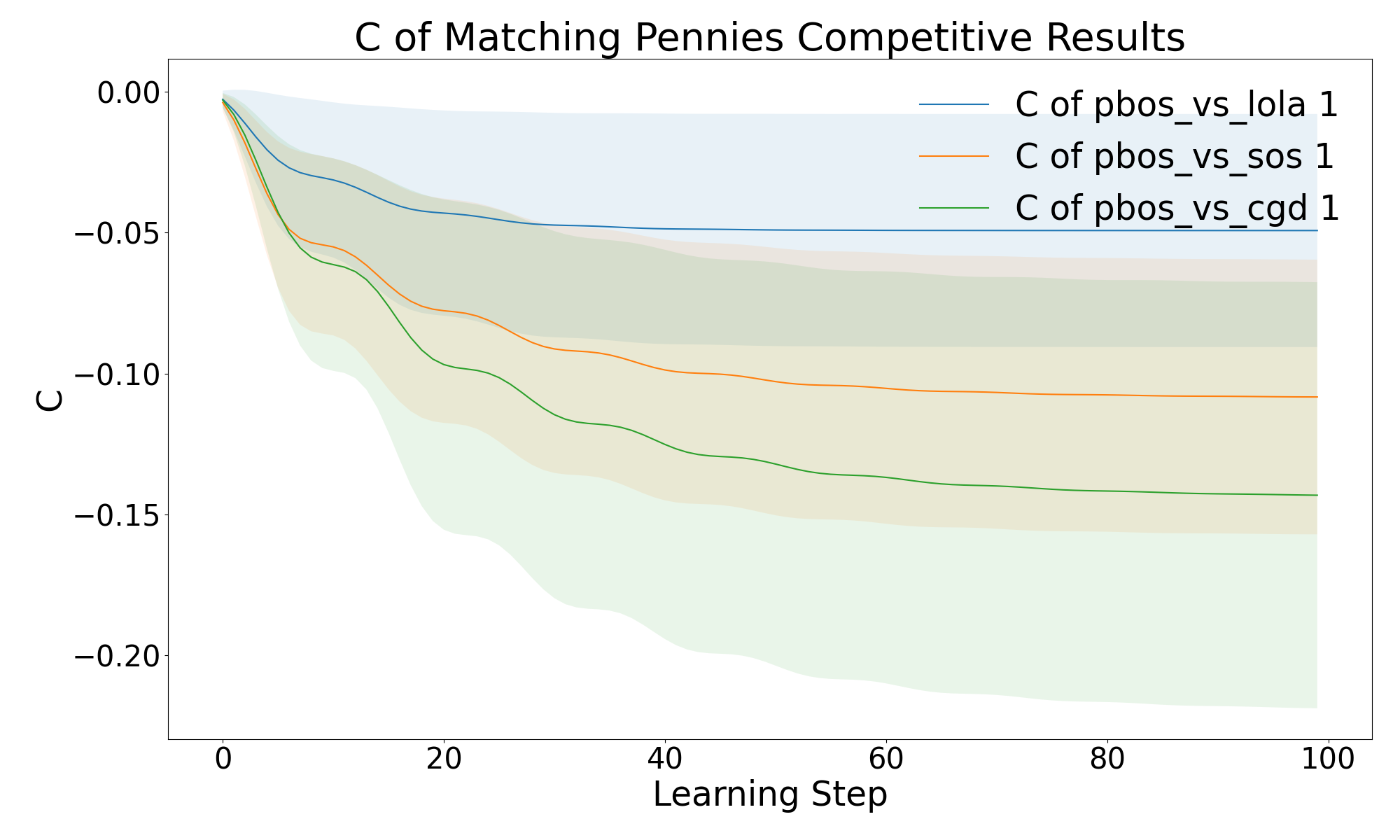}
        \caption{C of Matching Pennies}
        \label{C of Matching Pennies Competitive-lc}
    \end{subfigure}

    \begin{subfigure}{0.435\textwidth}
        \centering
        \includegraphics[width=\textwidth]{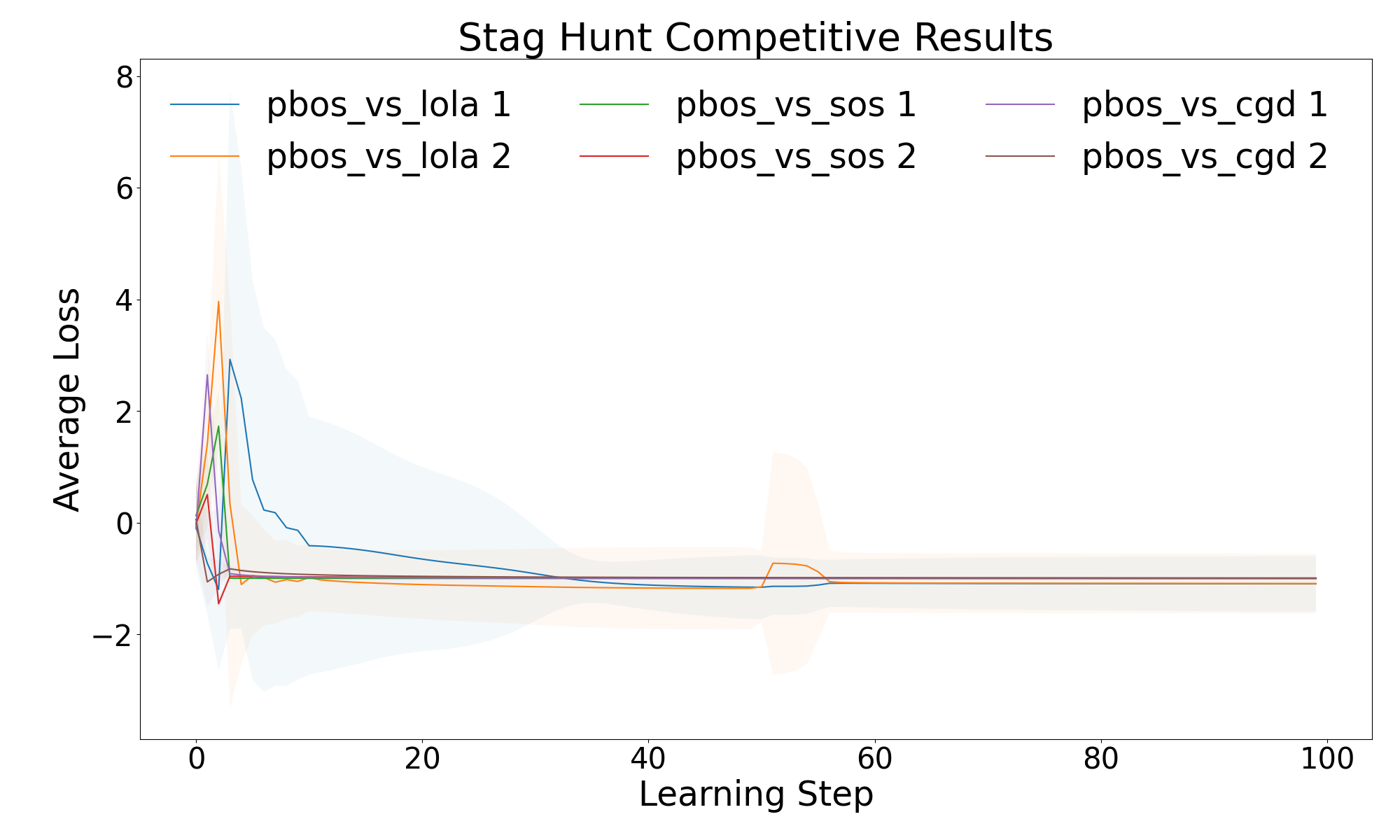}
        \caption{Stag Hunt}
        \label{Stag Hunt Competitive-lc}
    \end{subfigure}
    \hfill
    \begin{subfigure}{0.435\textwidth}
        \centering
        \includegraphics[width=\textwidth]{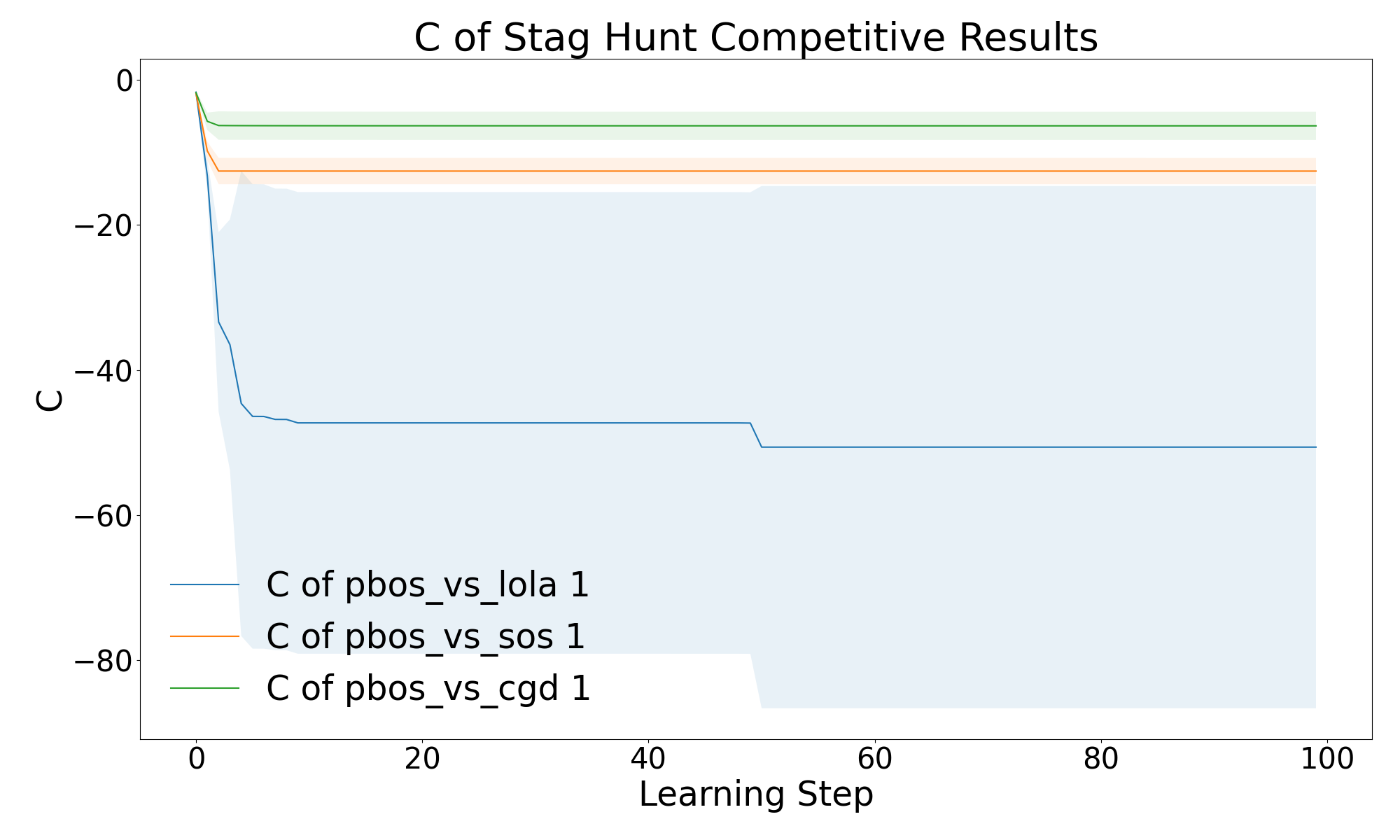}
        \caption{C of Stag Hunt}
        \label{C of Stag Hunt Competitive-lc}
    \end{subfigure}

    \caption{Results from 100 independent executions of each adversarial testing. For labels like `pbos\_vs\_lola 1' and `pbos\_vs\_lola 2', agent 1 uses the PBOS algorithm, and agent 2 uses the LOLA algorithm. Each experiment is conducted 100 times. Fig. \ref{C of Tandem Competitive-lc}, \ref{C of IPD Competitive-lc}, \ref{C of Matching Pennies Competitive-lc} and \ref{C of Stag Hunt Competitive-lc} illustrate the evolution of the preference parameters of agent 1 when using the PBOS algorithm.}
\end{figure}

        Based on Fig. \ref{Tandem Competitive-lc}-\ref{C of Stag Hunt Competitive-lc}, we derive the following insights regarding the performance of the PBOS algorithm compared to three baseline algorithms across four game environments with symmetrical loss functions:

        (a). Tandem Game: PBOS achieves NE when competing against agents using the SOS and CGD algorithms (Fig. \ref{Tandem Competitive-lc}). However, it falls short against LOLA due to the misalignment in learning a cooperative strategy $c_1>0$, which is exploited by the opponent (Fig. \ref{C of Tandem Competitive-lc}). 

        (b). IPD: PBOS converges to the NE when competing against agents using LOLA or SOS algorithms (Fig. \ref{IPD Competitive-lc}). Conversely, when confronting an agent using the CGD algorithm, PBOS adopts an antagonistic strategy ($c_1<0$) and ultimately converges to a less favorable outcome (Fig. \ref{C of IPD Competitive-lc}).

        (c). Matching Pennes: In this zero-sum game, PBOS learns a confrontational strategy ($c_1<0$) and achieves a zero loss when competing against agents using LOLA, SOS, or CGD (Fig. \ref{Matching Pennies Competitive-lc}, \ref{C of Matching Pennies Competitive-lc}).  This result aligns with the nature of Matching Pennies, where the mixed strategy NE is equal probabilities of each strategy.

        (d). Stag Hunt: PBOS agents adopt a confrontational strategy ($c_1<0$) and converge to a loss of $-1$, representing the ``safe" NE, when competing against agents employing LOLA, SOS, or CGD (Fig. \ref{Stag Hunt Competitive-lc}, \ref{C of Stag Hunt Competitive-lc}).

        From these observations, we can infer that, in the majority of cases, the PBOS algorithm is capable of reaching NE when competing with baseline algorithms.  However, there are specific scenarios where PBOS may underperform.  The primary reason for this suboptimal performance is attributed to the foundational premise of PBOS, which relies on modeling the opponent's preference changes.  If the opponent's preference remains constant ($c_2=0$), the interaction between PBOS and the baseline algorithms may yield less satisfactory results.  This underscores the importance of accurate modeling of opponent behavior in the efficacy of learning algorithms within game-theoretic contexts.

        \subsection{Failures}
        This section will present some scenarios where PBOS fails.
        The PBOS algorithm isn't suitable for black-box scenarios. During iteration, it requires knowledge of the opponent's loss function and gradient information. When pitted against baseline algorithms in adversarial tests, the PBOS algorithm fails to learn any information about the opponent's preference parameters, since those of the baseline algorithms are a constant ($c_2=0$). Future work will focus on designing strategies specifically for opponents with constant preference parameters (e.g., always cooperating or always defecting). Moreover, when extending the PBOS algorithm to multi-player ($n\geq 3$) game scenarios, the formula $\Delta c_{-i}=K_i\Delta c_i+\epsilon_i,i=1,2$ used for learning $c$ may no longer be applicable. This is because the relationships among multiple players can be composite, unlike the linear relationships in two-player cases.  

    \subsection{Discussion on Hyperparameter Selection}

    A pivotal component of the proposed Preference-Based Opponent Shaping (PBOS) algorithm is the careful selection of its principal hyperparameters: the policy learning rate $\alpha$ and the preference learning rate $\beta$. Although an exhaustive sweep of the parameter space is left for future work, the hyperparameter choices in our experiments were far from arbitrary; they were guided by the theoretical insights developed in Section \ref{sec4}.

    Our convergence analysis rests on the separation-of-timescales condition $0 < \beta \ll \alpha \ll 1$ (Lemma~\ref{beta<<0}, Corollary~\ref{ignore c}), which guarantees stable learning dynamics. Specifically, the policy parameters $\theta$ must evolve much faster than the preference parameters $c$, and this disparity is critical for two reasons:
    
    \begin{enumerate}
        \item \textbf{Stability}: When $\theta$ updates on a faster timescale, it can approach a quasi-stationary point for a slowly varying $c$. If $\beta$ were comparable to $\alpha$, the resulting non-stationary landscape would cause the policy to chase a moving target, risking oscillation or divergence.
        \item \textbf{Validity of Opponent Modelling}: The update rules for $c$ in Equations~(\ref{pbos_c1})--(\ref{pbos_c2}) rely on an approximation of the policy increment $\Delta\theta$, which is accurate only when $\theta$ is near a stationary point. Ensuring $\alpha \gg \beta$ preserves this approximation throughout training.
    \end{enumerate}

    Accordingly, throughout all experiments, we maintained the condition $\beta \ll \alpha$, for example by setting $\alpha = 0.1$ and $\beta = 0.01$. PBOS’s consistent success across the Tandem Game, the Iterated Prisoner’s Dilemma, and the Stackelberg Leader Game attests to the robustness of this theoretically motivated choice.

\section{Conclusion}\label{sec7}

    In this paper, we propose a novel preference-based adversary shaping (PBOS) method to improve the strategy learning process by using the opponent's objectives as preferences in the loss function. We introduce the preference parameter to avoid the limitation of agents considering only their own loss functions. By employing a method for shaping changes in opponent preference parameters, PBOS achieves higher reward in cooparative and competive game environments. Theoretical analysis shows that PBOS has good convergence properties and can obtain Nash equilibrium in games where other opponent shaping algorithms fail. We also conduct a series of experiments to demonstrate the effectiveness of PBOS. In many classic game environments, PBOS improves both the convergence and rewards of strategy learning. Furthermore, the PBOS algorithm exhibits strong generalization capabilities in randomly generated games and yields a 10\% improvement over baseline algorithms with respect to proximity to the NE. Future research will focus on enhancing the PBOS algorithm to better navigate complex game environments and refine opponent modeling, as well as exploring methods to detect and adapt to the dynamics of opponent preferences. A key area of exploration involves extending PBOS to multi-agent systems with more than two agents ($n \geq 3$), where addressing the computational complexity and coordination challenges becomes significantly more difficult. Additionally, enhancing PBOS’s adaptability to non-stationary environments, where opponent strategies shift over time, could expand its applicability, potentially via online learning or adaptive preference mechanisms. We hope that this work can promote the research of opponent modeling in game environments.

\bmhead{Acknowledgements}
This paper is supported by National Key R\&D Program of China (2021YFA1000403)
and the National Natural Science Foundation of China (Nos. 11991022,U23B2012).

\bibliography{ref}

\clearpage
\newcounter{mainfigures}
\setcounter{mainfigures}{\value{figure}}
\begin{appendices}
\section{Figures of direction and gradient in different games.}\label{secA1}
\counterwithout{figure}{section}
\setcounter{figure}{\value{mainfigures}}

\begin{figure}[!ht]
    \centering

        \begin{subfigure}{0.4950\linewidth}
        \centering
        \includegraphics[width=1.0\linewidth]{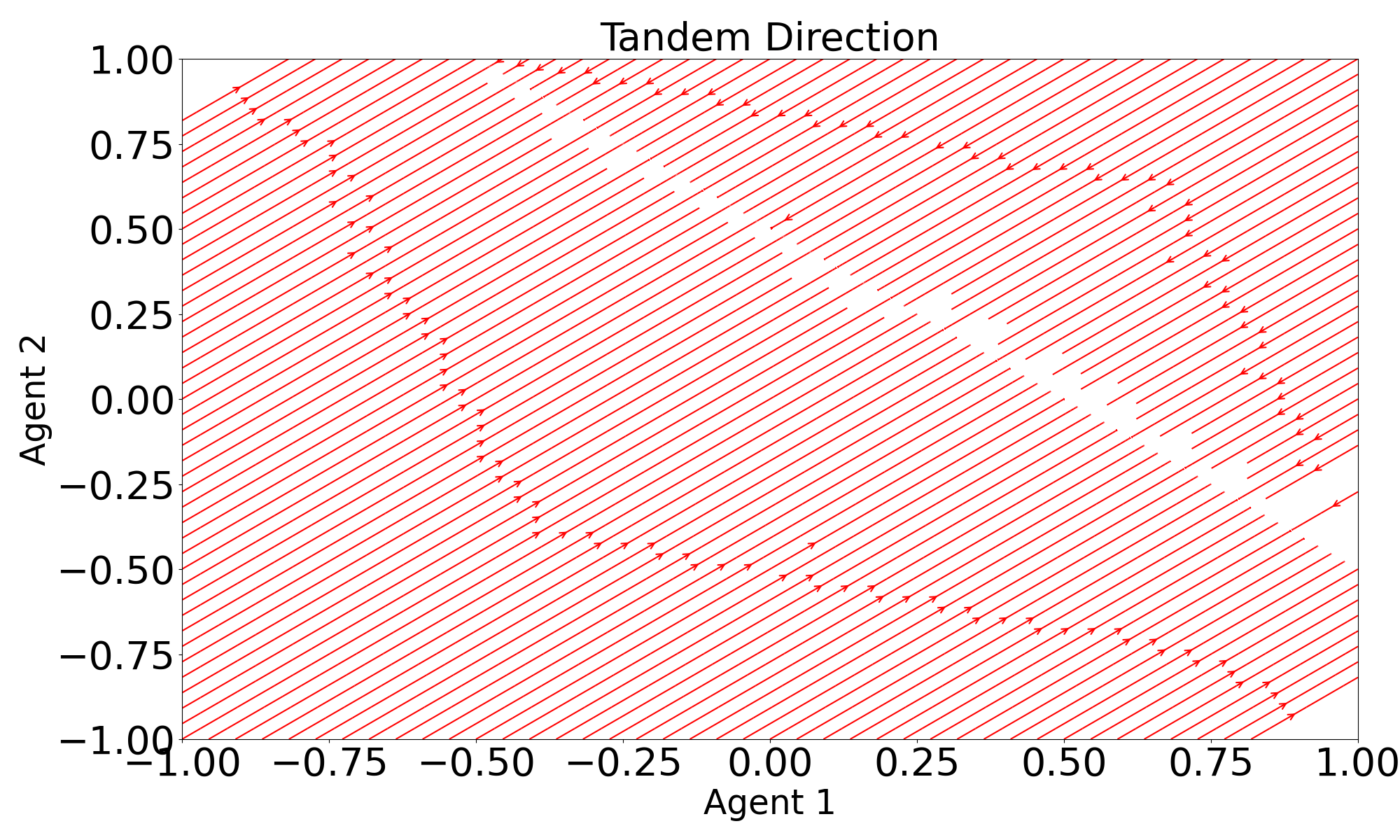}
        \caption{Tandem direction}
        \label{tandem-d}
    \end{subfigure}
    \hfill
    \begin{subfigure}{0.4950\linewidth}
        \centering
        \includegraphics[width=1.0\linewidth]{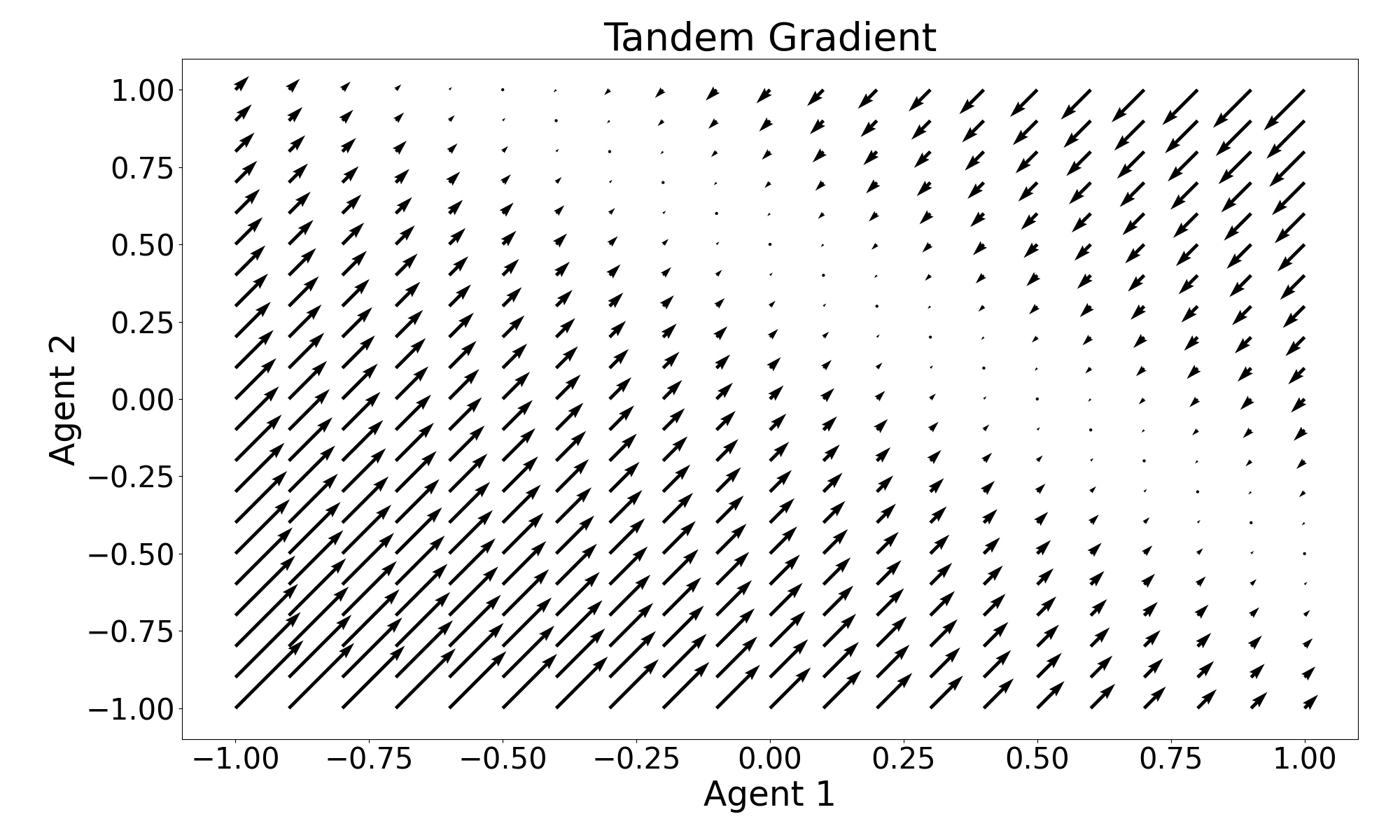}
        \caption{Tandem gradient}
        \label{tandem-g}
    \end{subfigure}

    \begin{subfigure}{0.4950\linewidth}
        \centering
        \includegraphics[width=1.0\linewidth]{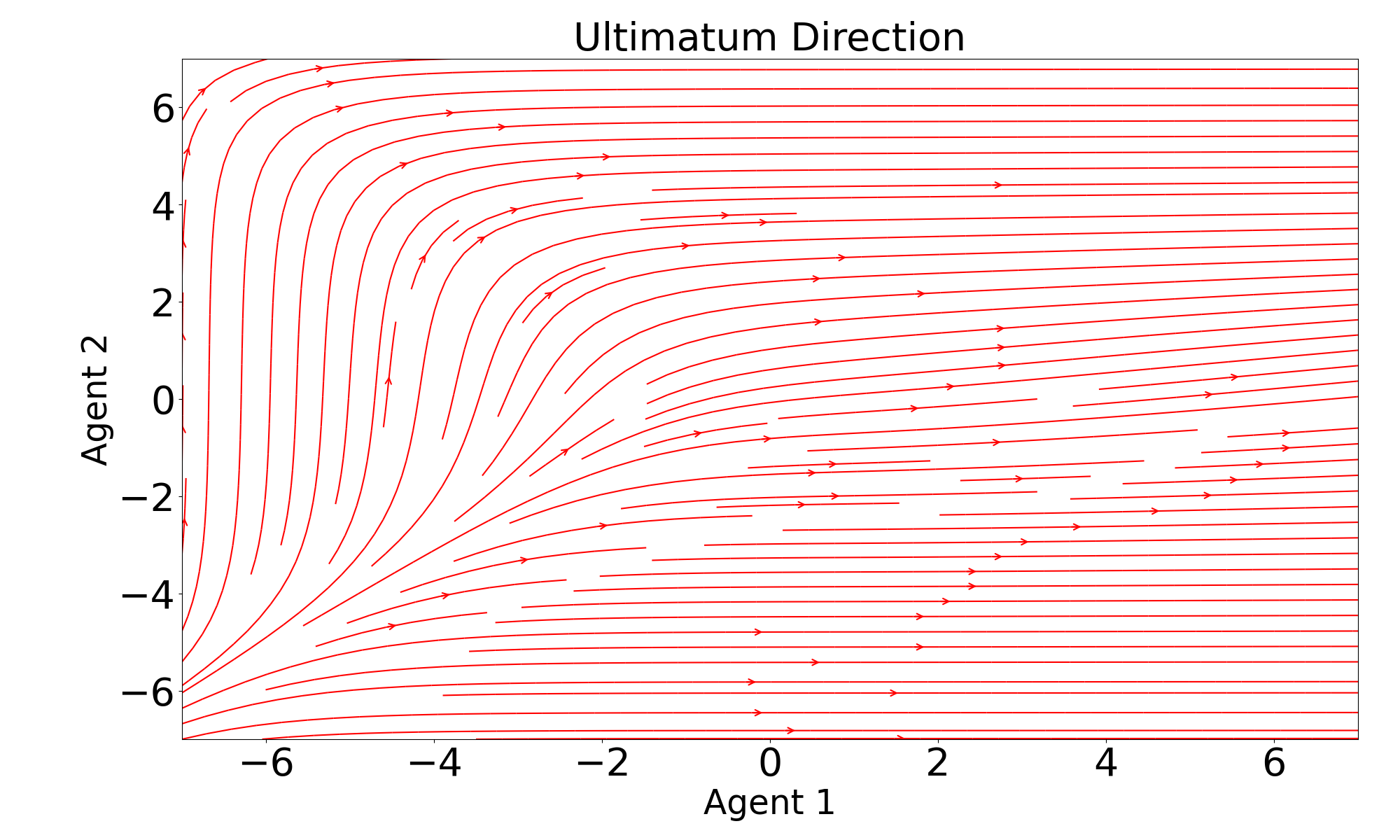}
        \caption{Ultimatum direction}
        \label{Ultimatum-d}
    \end{subfigure}
    \hfill
    \begin{subfigure}{0.4950\linewidth}
        \centering
        \includegraphics[width=1.0\linewidth]{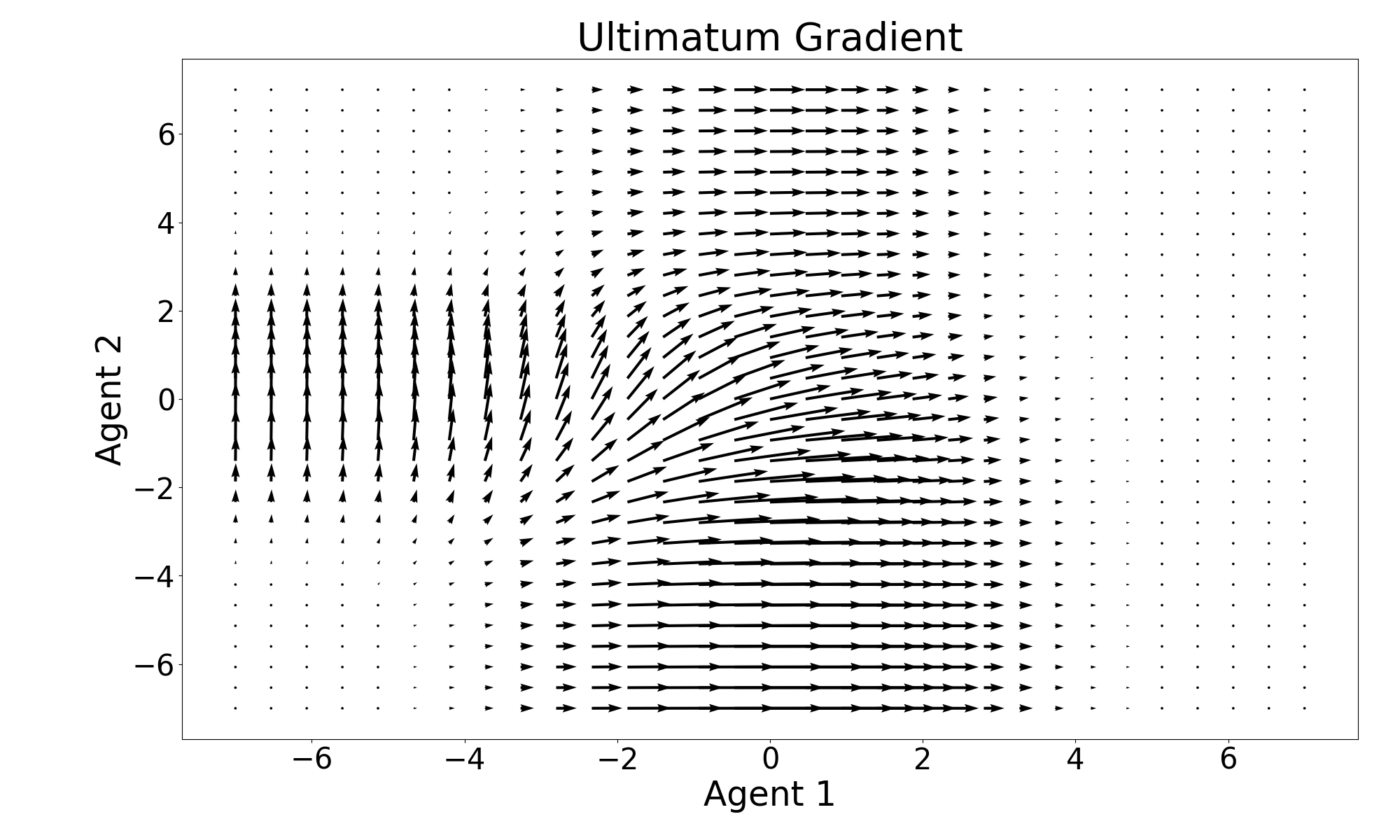}
        \caption{Ultimatum gradient}
        \label{Ultimatum-g}
    \end{subfigure}

    \begin{subfigure}{0.4950\linewidth}
        \centering
        \includegraphics[width=1.0\linewidth]{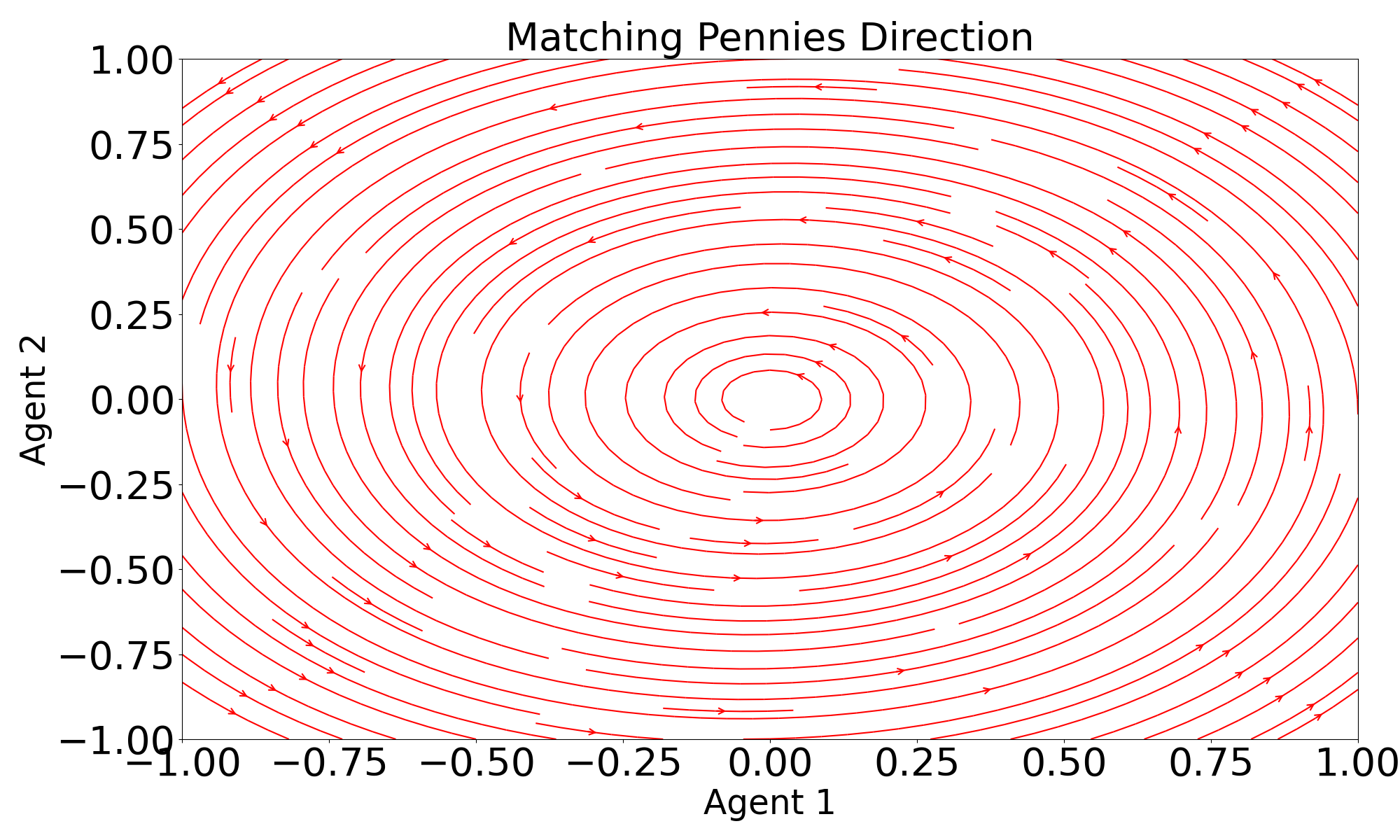}
        \caption{Matching Pennies direction}
        \label{Matching Pennies-d}
    \end{subfigure}
    \hfill
    \begin{subfigure}{0.4950\linewidth}
        \centering
        \includegraphics[width=1.0\linewidth]{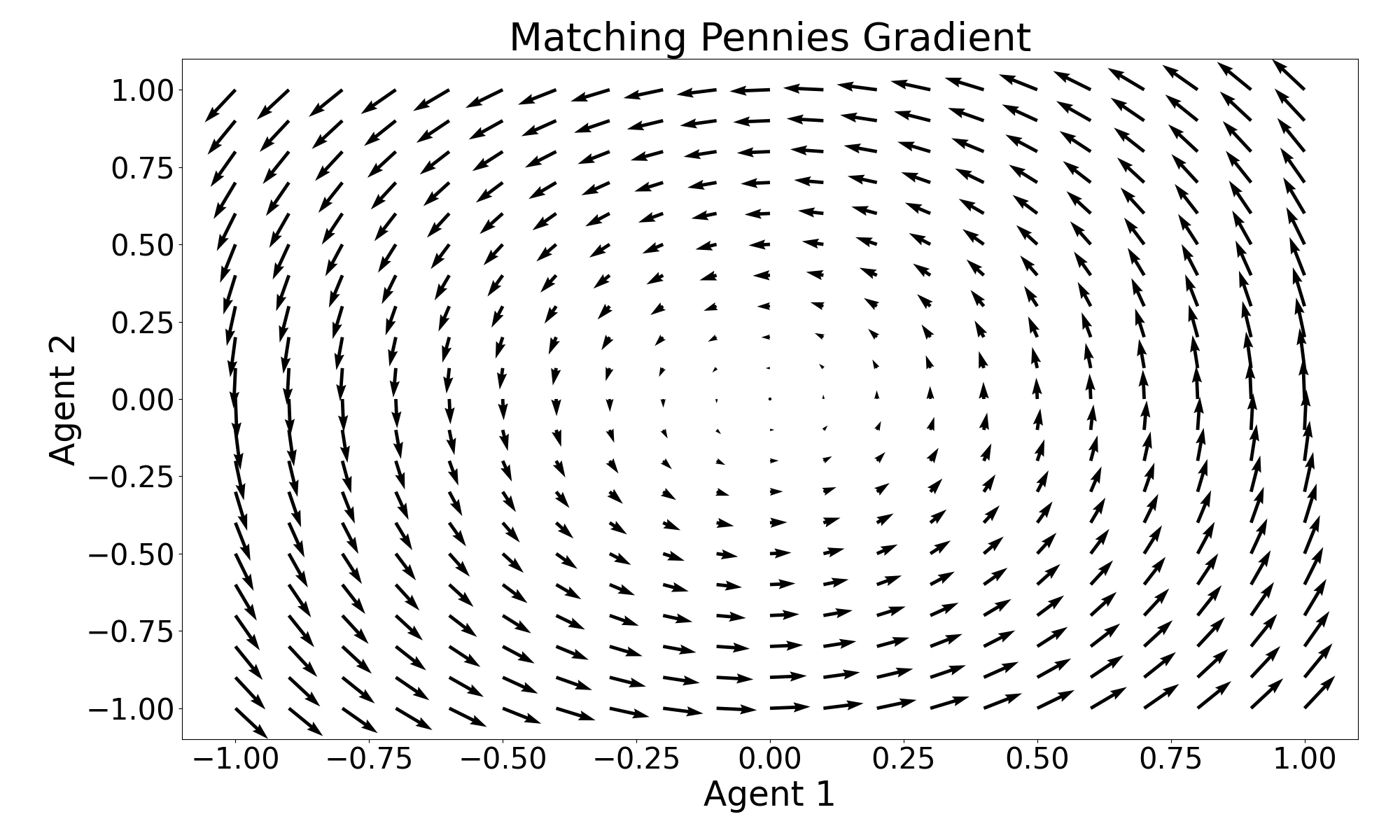}
        \caption{Matching Pennies gradient}
        \label{Matching Pennies-g}
    \end{subfigure}

    \caption{The two-dimensional visualization of strategies for agents engaged in a game using the PBOS algorithm is presented, with the left panel displaying the direction plot and the right panel showing the field plot. Continued on the following page.}
\end{figure}
\begin{figure}[p]
    \ContinuedFloat
    \centering
    \begin{subfigure}{0.4950\linewidth}
        \centering
        \includegraphics[width=1.0\linewidth]{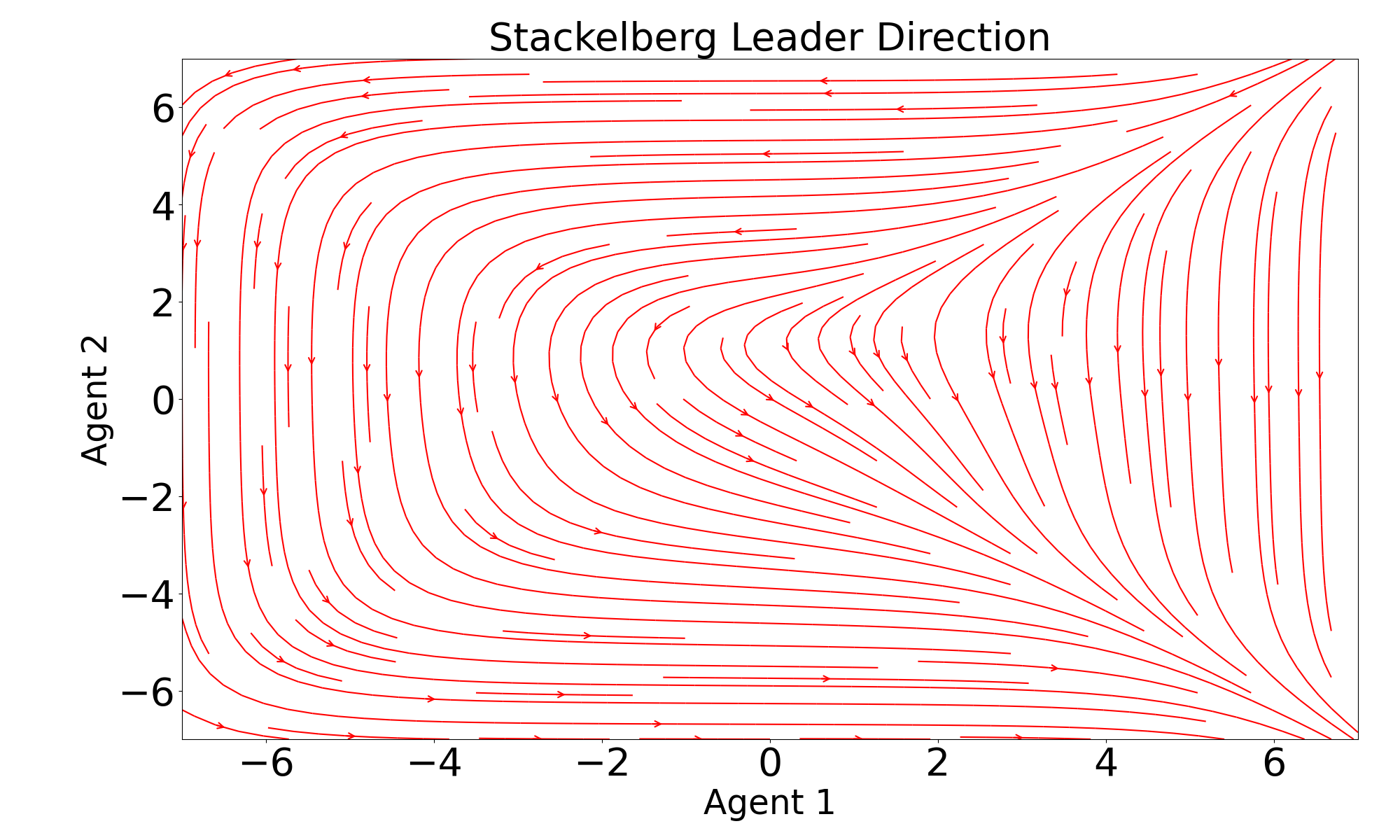}
        \caption{Stackelberg Leader direction}
        \label{Stackelberg Leader-d}
    \end{subfigure}
    \hfill
    \begin{subfigure}{0.4950\linewidth}
        \centering
        \includegraphics[width=1.0\linewidth]{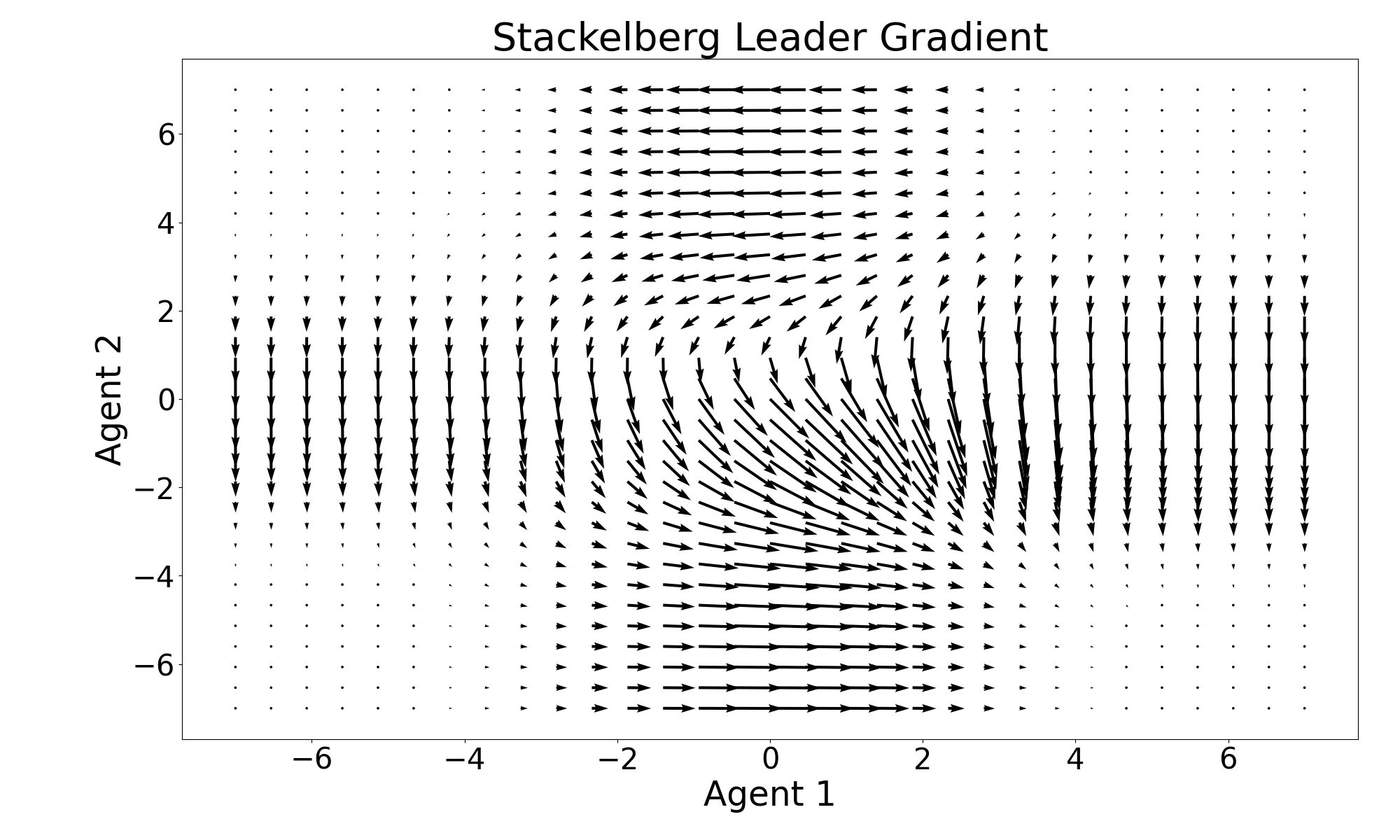}
        \caption{Stackelberg Leader gradient}
        \label{Stackelberg Leader-g}
    \end{subfigure}

    \begin{subfigure}{0.4950\linewidth}
        \centering
        \includegraphics[width=1.0\linewidth]{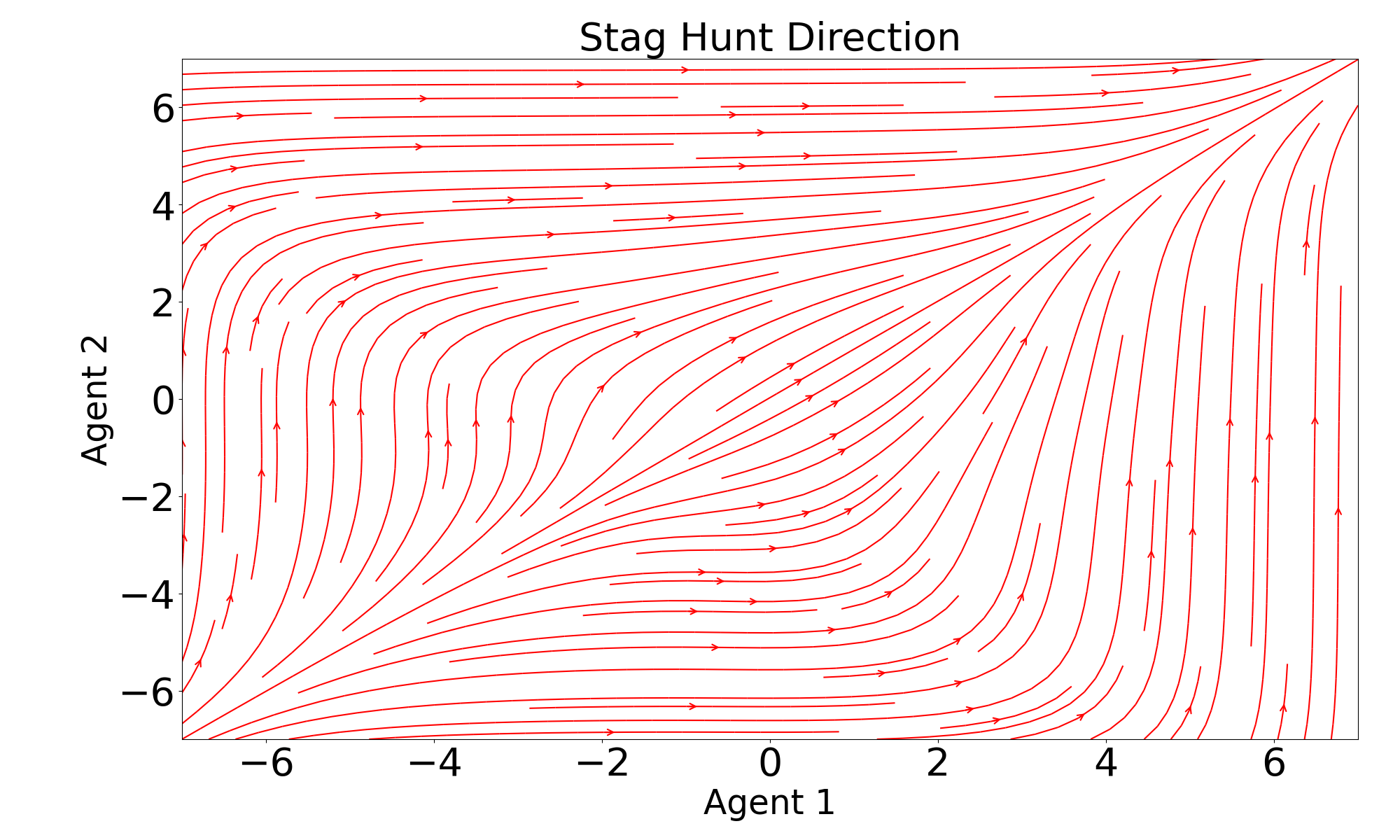}
        \caption{Stag Hunt direction}
        \label{Stag Hunt-d}
    \end{subfigure}
    \hfill
    \begin{subfigure}{0.4950\linewidth}
        \centering
        \includegraphics[width=1.0\linewidth]{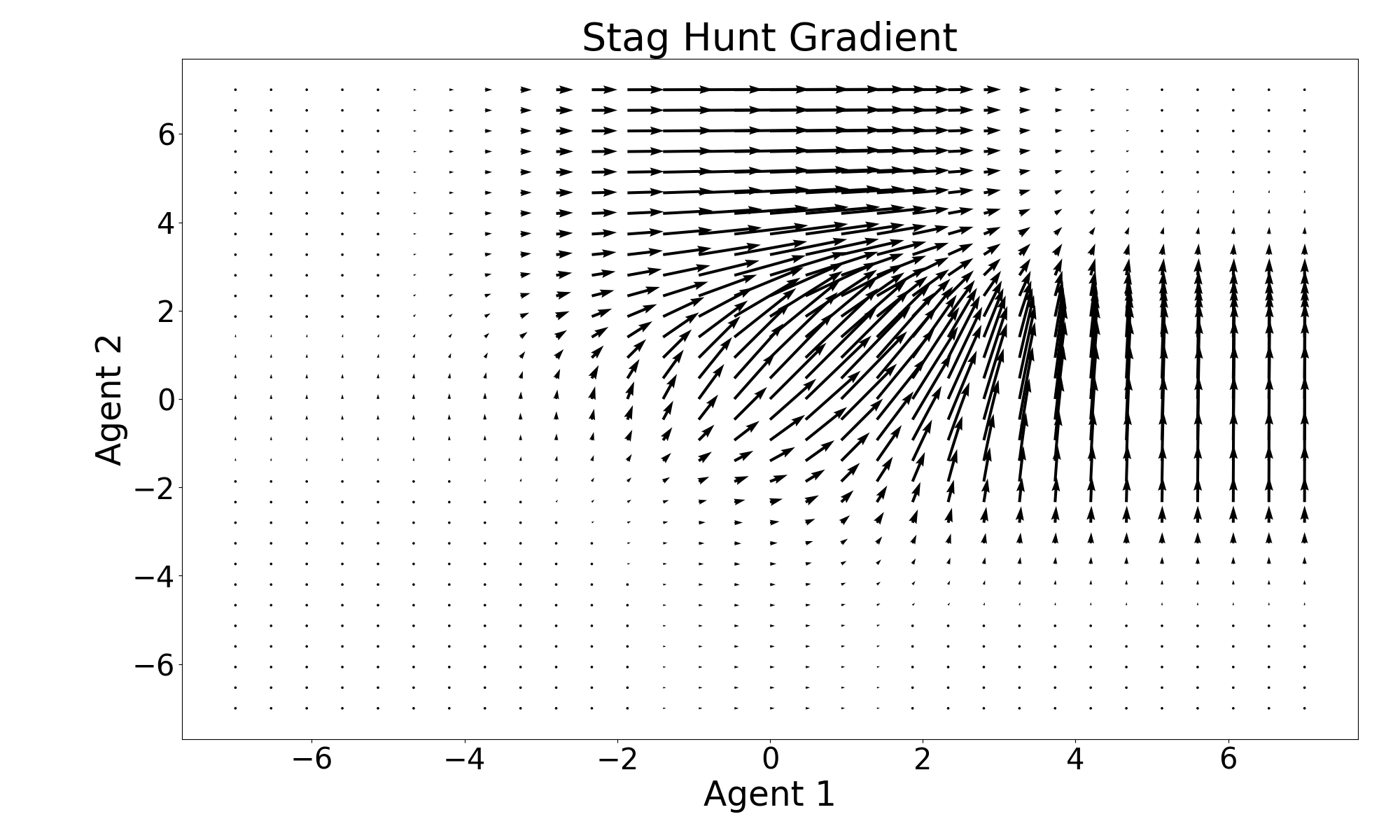}
        \caption{Stag Hunt gradient}
        \label{Stag Hunt-g}
    \end{subfigure}

    \caption[]{Continued: direction plots (left) and field plots (right) for the remaining games.}
\end{figure}

\end{appendices}

\end{document}